\PassOptionsToPackage{pagebackref=true}{hyperref}
\documentclass[final]{alt2026} 

\usepackage{capt-of}
\usepackage{booktabs}
\usepackage{makecell}
\usepackage{cancel}
\usepackage{pifont}
\usepackage{enumitem}
\usepackage{wrapfig}
\usepackage{xspace}
\usepackage{float}
\usepackage{algorithm}
\usepackage{algorithmic}
\usepackage{comment}

\usepackage{aliascnt}
\usepackage{cancel}
\usepackage[capitalize,noabbrev]{cleveref}

\usepackage{graphicx}
\usepackage{subcaption}
\usepackage[most]{tcolorbox}

\renewcommand*\backref[1]{\ifx#1\relax \else (cited on {p.~#1}) \fi}

\newcommand{\cmark}{\textcolor{green!40!black}{\ding{51}}} 
\newcommand{\xmark}{\textcolor{red!60!black}{\ding{55}}}   

\newcommand{\complexityhigh}[1]{\textcolor{red!60!black}{#1}}
\newcommand{\complexitylow}[1]{\textcolor{green!40!black}{#1}}

\newtcolorbox{takehome}[1][]{
  enhanced,
  colback=gray!5,        
  colframe=gray!80,      
  width=\linewidth,      
  sharp corners,         
  boxrule=0pt,           
  leftrule=3pt,          
  fonttitle=\bfseries,
  coltitle=black,
  title={Take-Home Message: #1},
  attach title to upper,
  after title={\smallskip\par},
  left=3pt,      
  right=3pt,     
}

\newaliascnt{lemma}{theorem}
\newtheorem{lemma}[lemma]{Lemma}
\aliascntresetthe{lemma}

\newaliascnt{corollary}{theorem}
\newtheorem{corollary}[corollary]{Corollary}
\aliascntresetthe{corollary}

\newaliascnt{definition}{theorem}
\newtheorem{definition}[definition]{Definition}
\aliascntresetthe{definition}

\newaliascnt{proposition}{theorem}
\newtheorem{proposition}[proposition]{Proposition}
\aliascntresetthe{proposition}

\newtheorem{assumption}{Assumption}

\crefname{algorithm}{Scheme}{Schemes}
\crefname{assumption}{Assumption}{Assumptions}
\crefname{reduction}{Reduction}{Reductions}
\crefname{equation}{Eq.}{Eqs.}
\crefname{appendix}{App.}{App.}
\Crefname{appendix}{App.}{App.}

\newcommand{\vect}[1]{\mathbf{#1}}

\newcommand{\0}{\mat{0}}
\newcommand{\teacher}{\vw^{\star}}
\newcommand{\mat}[1]{\mathbf{#1}}
\newcommand{\vw}{\vect{w}}
\newcommand{\z}{\vect{z}}
\newcommand{\X}{\mat{X}}

\newcommand{\y}{\vect{y}}

\newcommand{\I}{\mat{I}}

\newcommand{\tr}{\mathrm{Tr}}
\newcommand{\trace}{\tr}
\def\mSigma{{\mat{\Sigma}}}

\def\mP{{\mat{P}}}

\newcommand{\tparagraph}[1]{\vspace{-5pt}\paragraph{#1}}

\newcommand{\dd}{\mathop{}\!\mathrm{d}}

\def\reals{\mathbb{R}}

\floatname{algorithm}{Scheme}

\newcommand{\hfrac}[2]{{#1}/{#2}}

\newcommand{\cnt}[1]{\left[{#1}\right]}
\newcommand{\explain}[1]{\left[\substack{#1}\right]}

\newcommand{\prn}[1]{\left({#1}\right)}

\newcommand{\bigprn}[1]{\big({#1}\big)}
\newcommand{\Bigprn}[1]{\Big({#1}\Big)}
\newcommand{\biggprn}[1]{\bigg({#1}\bigg)}

\newcommand{\norm}[1]{\left\Vert{#1}\right\Vert}

\newcommand{\biggnorm}[1]{\bigg\Vert{#1}\bigg\Vert}

\newcommand{\featnoise}{v_x} 
\newcommand{\labelnoise}{v_z} 
\newcommand{\ratio}{\alpha} 
\newcommand{\mnist}{\texttt{MNIST}\xspace}

\DeclareMathOperator*{\argmin}{argmin}

\def\remref#1{Remark~\ref{#1}}

\def\appref#1{App.~\ref{#1}}

\makeatletter
\newenvironment{recall}[1][\proofname]{\par
\normalfont \topsep6\p@\@plus6\p@\relax
\trivlist
\item\relax
{\bfseries
Recall #1}%
{\bfseries\@addpunct{.}}\hspace\labelsep\ignorespaces
}
\makeatother

\makeatletter
\renewenvironment{proof}[1][\relax]{\par
  \normalfont \topsep6\p@\@plus6\p@\relax
  \trivlist
  \item[\hskip\labelsep\bfseries
    \ifx#1\relax \proofname\else\proofname{} #1\fi\@addpunct{.}]\ignorespaces
}
  {\jmlrQED}
\makeatother

\usepackage{titletoc} 
\newcommand{\appendixtableofcontents}{%
  \printcontents[app]{l}{1}{\section*{Appendix Contents}}
}

\title[Optimal L2 Regularization in High-dimensional Continual Linear Regression]{Optimal L2 Regularization\\
in High-dimensional Continual Linear Regression}

\usepackage{times}

\altauthor{%
 \Name{Gilad Karpel}
 \Email{karpelgilad@gmail.com}\\
 \addr Department of Data and Decisions Sciences, Technion
 \AND
 \Name{Edward Moroshko}
 \Email{edward.moroshko@gmail.com}\\
 \addr School of Engineering, University of Edinburgh
 \AND
 \Name{Ran Levinstein}
 \Email{ranlevinstein@gmail.com}\\
 \addr Department of Computer Science, Technion
 \AND
 \Name{Ron Meir} \Email{rmeir@ee.technion.ac.il}\\
 \addr Department of Electrical and Computing Engineering, Technion
 \AND
 \Name{Daniel Soudry} \Email{daniel.soudry@gmail.com}\\
 \addr Department of Electrical and Computing Engineering, Technion
 \AND
 \Name{Itay Evron}
 \Email{itay@evron.me}\\
 \addr Meta
}

\begin{document}

\maketitle

\begin{abstract}
We study generalization in an overparameterized continual linear regression setting, where a model is trained with L2 (isotropic) regularization across a sequence of tasks. 
We derive a closed-form expression for the expected generalization loss in the high-dimensional regime that holds for arbitrary linear teachers. 
We demonstrate that isotropic regularization mitigates label noise under both single-teacher and multiple i.i.d.~teacher settings, whereas prior work accommodating multiple teachers either did not employ regularization or used memory-demanding methods.
Furthermore, we prove that the optimal fixed regularization strength scales nearly linearly with the number of tasks $T$, specifically as $T/\ln T$.
To our knowledge, this is the first such result in theoretical continual learning.
Finally, we validate our theoretical findings through experiments on linear regression and neural networks, illustrating how this scaling law affects generalization and offering a practical recipe for the design of continual learning systems.
\end{abstract}


\section{Introduction}
\label{sec:intro}

In the era of large foundation models (e.g., LLMs), the ability to learn tasks \emph{sequentially} without catastrophic forgetting \citep{mccloskey1989catastrophic} or retraining from scratch has become increasingly critical \citep[see surveys in][]{wang2024comprehensive,van2024continual}. 
Continual learning not only saves substantial computational resources but is also essential in settings where previous data cannot be retained, e.g., under privacy or storage constraints.

We focus on continual linear regression, a fundamental analytical setting that has proven useful for illuminating diverse aspects of continual learning systems, such as the effects of task similarity \citep{lin2023theory,hiratani2024disentangling,tsipory2025greedy},
task recurrence \citep{evron2022catastrophic,swartworth2023nearly}, 
overparameterization \citep{goldfarb2024theJointEffect},
algorithms \citep{doan2021NTKoverlap,peng2023ideal},
regularization \citep{li2023fixed,levinstein2025optimal},
and trade-offs such as memory versus generalization \citep{li2025memory}.
Similar linear settings are also used to study continual \emph{classification} \citep[e.g.,][]{evron2023classification,jung2025convergence}.

\pagebreak


We analyze a popular regularization scheme that penalizes changes in the parameter space to prevent forgetting previous expertise.
Specifically, we focus on L2 regularization, which penalizes changes isotropically, i.e., $\vw_t = \argmin_{\vw}\big\{\mathcal{L}_{t}(\vw) +\lambda d\|\vw-\vw_{t-1}\|^2\big\}$.
Despite not maintaining an additional regularization weight matrix, L2 regularization has been found to work well in practice \citep{Hsu2018Baselines,lubana2021regularization,smith2022closer}.

The regularization strength $\lambda$ plays a central role in methods like ours, governing the balance between stability and plasticity \citep{mermillod2013stability,kirkpatrick2017ewc}. 
Theoretically guided scaling laws for the regularization strength can greatly reduce computational cost, alleviating the need for extensive hyperparameter tuning. 
Only a few theoretical works have focused on the regularization strength.
For example, \citet{levinstein2025optimal} shows that near-optimal convergence rates can be achieved using a fixed strength. However, their analysis pertains only to the training error in the worst case, leaving open the question of whether the chosen strength is truly optimal.

While some prior theoretical work analyzes the \emph{training} loss and forgetting \citep[e.g.,][]{evron2022catastrophic,evron2025better,swartworth2023nearly,goldfarb2024theJointEffect,levinstein2025optimal}, we adopt a complementary and widely studied  \emph{statistical} viewpoint, focusing on the \emph{generalization} loss of our continual scheme.
Following prior work in a similar spirit \citep[e.g.,][]{lin2023theory,li2023fixed,zhao2024statistical}, we study a random design in which labels are generated by an underlying (noisy) linear `teacher' (i.e., task ground truth).
Importantly, our analysis accommodates \emph{multiple teachers}, in contrast to several prior studies \citep[e.g.,][]{ding2024understanding,zhao2024statistical,levinstein2025optimal}, which assume a single global teacher for all tasks. 
Allowing multiple teachers provides a more realistic model for practical scenarios where the underlying ground truth (e.g., user preferences or fraud behaviors) naturally evolve over time. 
Under an i.i.d.~feature assumption common in statistical studies \citep[e.g.,][]{lin2023theory,goldfarb2023analysis,zhao2025highdimensional}, a single teacher yields generalization affected solely by label noise, while multiple teachers introduce task-level variation, again, making the setting more realistic and challenging.

A closely related paper by \citet{lin2023theory},  which also considered multiple teachers, studied an \emph{unregularized} scheme where past tasks influence optimization only \emph{implicitly}, through the bias of gradient algorithms (see also \citealt{evron2022catastrophic}).
While their analysis enables strong theoretical results, our findings highlight that incorporating a regularizer is crucial, particularly for reducing the noise floor induced by label noise.

Our results provide general insights for \emph{any} setup with multiple teachers, whereas our more concrete findings pertain to teachers drawn i.i.d.~from some distribution---allowing simple yet illustrative analysis that captures effects of teacher variance. 
We place particular emphasis on the regularization strength, proving that its optimal value scales nearly linearly with the number of tasks. 
This practical rule of thumb alleviates the need for costly hyperparameter tuning with many tasks, as we can now tune the regularization strength using only a few tasks, and then scale it accordingly to longer task horizons.
We empirically verify these findings on an \mnist-based dataset using both a linear regression model and a ReLU-activated neural network.

\paragraph{Summary of our contributions.}

\begin{enumerate}[leftmargin=0.5cm, itemindent=0.05cm, labelsep=0.2cm, itemsep=2pt,topsep=-1pt]

\item We analyze generalization in continual linear regression in a high-dimensional regime, accommodating multiple teachers, label noise, and an arbitrary number of tasks---a setup broader than those in prior work (see \cref{tab:comparison}).

\item For both a single teacher and i.i.d.~teachers, we identify the optimal regularization strength $\lambda^{\star}$. 

\item We demonstrate that L2 regularization, which does not require maintaining an additional weighting matrix, can close a gap in prior theoretical literature by staying informative in noisy regimes.

\item We validate our theoretical results in a standard regularized continual linear model, using both synthetic random regression data and an \mnist-based classification problem.

\item Finally, we demonstrate similar findings with simple neural networks, validating our theory-guided rule of thumb for choosing the regularization strength.

\end{enumerate}

\section{Setting: Regularized Continual Linear Regression}
\label{sec:setting}

As explained in the introduction, we study the illustrative continual linear regression setting using an L2 regularization scheme, described formally below.

\paragraph{Learning tasks in a sequence.}
In our setting, the learner is exposed to a sequence of \(T\) tasks \(\left((\X_i, \y_i)\right)_{i=1}^T\), where \(\X_i \in \mathbb{R}^{n \times d}\), \(\y_i \in \mathbb{R}^{n}\).\footnote{
To simplify exposition, we assume all tasks have $n$ samples, though our analysis can extend to varying sample sizes.
} 
At each iteration \(t\), the learner updates the predictor \(\vw_t\) by fitting only the current task, \emph{without access to samples from previous tasks}.

To be able to learn continually and mitigate \emph{catastrophic forgetting}, it is common to add a regularization term towards the previous iterate $\vw_{t-1}$.
While the regularizer is often based on the Fisher information of previous tasks \citep{kirkpatrick2017ewc,benzing2021unifying_regularization}, recent works have observed empirically that L2 (isotropic) regularization achieves similar performance \citep{Hsu2018Baselines, lubana2021regularization,smith2022closer}. 
Thus, we follow recent theoretical work \citep[e.g.,][]{li2023fixed,levinstein2025optimal} and focus on the following L2 regularization scheme.

{\centering
\begin{algorithm}[H]
  \caption{Continual Linear Regression with L2 Regularization
   \strut
  \label{schm:regularized}}
\begin{algorithmic}
\vspace{-0.1em}
   \STATE {
   \textbf{Input:} 
   $\{(\X_i,\y_i)\}_{i=1}^T$, 
   initial iterate $\vw_{0}$,
   regularization strength $\lambda$
   } 
   \STATE { 
   \textbf{For each} iteration $t=1,\dots,T$: 
   }
   \STATE {
   $\quad \displaystyle
        \mathbf w_t \;\gets\;
        {\argmin}_{\mathbf w}\Bigl\{\|\mathbf X_{t}\mathbf w-\mathbf y_{t}\|^2
        +\lambda d\|\mathbf w-\mathbf w_{t-1}\|^2\Bigr\}$
    }
   \STATE {
   \textbf{Output:} $\vw_{T}$
   } 
\end{algorithmic}
\end{algorithm}
\vspace{-0.4em}
}

Note that for $\lambda\to0$, the algorithm reduces to optimizing the current task alone (as analyzed in \citet{evron2022catastrophic}), whereas for large values of $\lambda$, the current task has negligible influence.
\linebreak
We normalize the regularization strength by the data dimensionality, following a convention commonly used in Bayesian settings, where the prior is often scaled in this way~\citep{lee2018deep}.

For ease of readability, we keep implicit any dependence of $\vw_t$ on other quantities (e.g., $\lambda,d$).
Furthermore, we collect all notation used throughout the paper in \cref{tab:notation} (\appref{app:notations}).

\subsection{Statistical Setup: High-Dimensional Random Design with Linear Teachers}
\label{sec:statistical_setting}

To study generalization, we adopt a statistical setup with random data, multiple (linear) teachers, and noisy labels. 
Variants of this setup are commonly used in the theoretical literature on continual learning \citep[e.g.,][]{goldfarb2023analysis,lin2023theory,zhao2024statistical, Friedman25} and in broader machine learning research \citep[e.g.,][]{hsu2012random,belkin2020two,bartlett2021benign}.
As shown in \cref{tab:comparison}, our setup is comparatively permissive, allowing both label noise and multiple teachers (e.g., i.i.d.~teachers).

\pagebreak

\paragraph{Data model.}
We consider a random design,
with independent random feature matrices \linebreak
$\X_1,\dots, \X_T\in \mathbb{R}^{n\times d}$,
where the entries of each matrix $\X_{t}$ are i.i.d.~random variables with mean~$0$, variance $\featnoise>0$, and a finite   $\prn{4+\epsilon}$\textsuperscript{th} moment for some $\epsilon>0$.\footnote{
For example, the entries of $\X_{1}$ can be $\mathcal{N}(0,1)$ and the entries of $\X_{2}$ can be $\text{Unif}{[-\sqrt{3}, \sqrt{3}]}$.
} 
Moreover, we focus on the high-dimensional regime, where 
$$n,d\to \infty,\quad \frac{n}{d}\to \ratio\le 1\,.
$$
For each task $t\in [T]\triangleq\left\{1,\dots,T\right\}$, labels are generated according to a corresponding ground-truth `teacher' vector $\vw_t^{\star}\in\mathbb{R}^{d}$ via a noisy linear model. That is,
$$
\y_{t}=\X_{t}\vw_{t}^{\star}+\z_{t}
\,,
$$
where $\z_t$ is a label noise vector, satisfying
\(\mathbb{E}[\z_t]=\0\) and \(\mathrm{Cov}(\z_t)=\labelnoise \I\) for some noise variance $\labelnoise \geq 0$, sampled independently of other tasks, i.e., of $\X_{t'},\z_{t'},\forall t'\neq t$.

Finally, for technical reasons, we assume finite asymptotic teacher correlations. That is,
$$
(\teacher_i - \teacher_j)^{\top}(\teacher_k - \teacher_l)
\xrightarrow[d\to\infty]{}
r_{i,j,k,l}
\,,
\qquad
\forall 
i, j, k, l\in\left\{0,1,\dots,T\right\}
\,.$$

\paragraph{Implications of the data model.}
High dimensionality allows us to derive explicit expressions using the Marchenko--Pastur theorem \citep{marchenko1967pastur}. 
Importantly, accommodating multiple teachers enables us to analyze i.i.d.~teachers, thereby characterizing continual scenarios that are more practical than those considered in most prior work (see \cref{sec:noisy_teacher}).

\paragraph{Metrics of interest.}
Some prior theoretical work in continual learning \citep[e.g.,][]{goldfarb2024theJointEffect,evron2022catastrophic,evron2025better} has analyzed the (training) squared error $\norm{\X_{i} \vw-\y_{i}}^2$.
However, under our statistical model, it is readily seen that averaging over the random test data and label noise yields
\begin{equation}
\label{eq:expected_loss}
\begin{aligned}
\mathbb{E}
\,\big\|\X_{i}\vw - \y_{i}\big\|^2
&= 
\mathbb{E}\,\big\|\X_{i}(\vw-\vw_{i}^\star) - \z_{i}\big\|^2 
= 
\featnoise n 
\,\norm{\vw - \vw_{i}^{\star}}^{2}
+
\labelnoise n
\,,
\end{aligned}
\end{equation}
making $\norm{\vw-\vw_{i}^\star}$ the only quantity influenced by the actual solution $\vw$.
Thus, we follow other prior work \citep{lin2023theory} and measure \emph{generalization} by the distance to the teacher(s) generating the labels.
Our formal definitions are provided below.

\begin{definition}[Individual Task Loss]
\label{def:train_loss}
For any $\vw\in\reals^{d}$, 
the loss of task $i\in\cnt{T}$ is defined as the distance of $\vw$ from the task's associated teacher.
That is,
$$\mathcal{L}_i(\vw)=\left\Vert \vw-\teacher_i\right\Vert ^{2}\,.$$
\end{definition}

\begin{definition}[Average Generalization Loss]\label{def:generlization:error}
The expected generalization loss at iteration $t$ is defined as the average of the individual losses of \emph{all} tasks encountered so far, evaluated at the iterate $\vw_{t}$.
That is,
$$
\mathbb{E}
\left[G_t\right]
\triangleq
\frac{1}{t}
\sum_{i=1}^{t} 
\mathbb{E}
\left[
\mathcal{L}_i(\vw_t)
\right]
=
\frac{1}{t}\sum_{i=1}^{t}
\mathbb{E}
\left\Vert\vw_t-\teacher_i\right\Vert^{2}
\,.$$
Note that $\vw_t$ is a random variable, as it depends on $\X_i,\z_i$ and the possibly random $\vw_i^{\star}, \forall i\in \cnt{t}$.
\end{definition}

\section{Generalization Analysis and Optimal Regularization Scaling Laws}
\label{sec:results}

In this section, we derive a closed-form expression for the expected generalization loss (of \cref{def:generlization:error}). 
This general expression accommodates diverse teacher setups and serves as the foundation for investigating an important special case in later sections.

\begin{theorem}[Generalization under General Teachers]
\label{eq:main_result}
Recall the high-dimensional statistical setting of \cref{sec:statistical_setting}, where the number of samples and the dimension $n,d\to \infty$ and $\frac{n}{d}\to \ratio\le 1$.
Moreover, recall that $\featnoise,\labelnoise$ are the variances of features and labels, respectively.
Then, the expected generalization loss (\cref{def:generlization:error}) 
after learning $T$ tasks with \cref{schm:regularized} is,
\[
\begin{aligned}
\mathbb{E}[G_T] 
  &
  = \underbrace{\labelnoise c\,\frac{1-a^{T}}{1-a}}_{\text{Label noise term}}
     \;+\underbrace{\;\frac{1}{T}\sum_{i=1}^{T}
     \left\Vert \vw_{T}^{\star}-\vw_{i}^{\star}\right\Vert^{2}}_{\text{Teacher variability term}}
  + \underbrace{2\big(\vw_{T}^{\star}-\bar{\vw}^{\star}\big)^{\!\top}
     \sum_{i=1}^{T} b^{\,T-i+1}
     \big(\vw_{i-1}^{\star}-\vw_{i}^{\star}\big)}_{\text{Interaction term}}\\[4pt]
  &
  \quad + \underbrace{\sum_{i=1}^{T}\sum_{i'=1}^{T}
     a^{\,T-\max(i,i')+1}\,b^{|i-i'|}
     \big(\vw_{i-1}^{\star}-\vw_{i}^{\star}\big)^{\!\top}
     \big(\vw_{i'-1}^{\star}-\vw_{i'}^{\star}\big)
     }_{\text{Accumulated temporal (teacher) correlations}}
     \,,
\end{aligned}
\]
where we define $\teacher_0 \triangleq \vw_0$ (by abuse of notation), the average teacher $\bar{\vw}^{\star} \triangleq \frac{1}{T}\sum_{i=1}^T \teacher_i$, and
\[
\begin{aligned}
  a &\triangleq 
  \tfrac{1}{2}
  \biggprn
  {
          1-\ratio + \frac{
          \lambda(1+\ratio)+\featnoise(1-\ratio)^{2}}
          {\sqrt{\lambda^2 +2\lambda\featnoise\prn{1+\alpha}+\featnoise^2 \prn{1-\ratio}^2}}
    }\,,
\\[4pt]
  b &\triangleq 
  \tfrac{1}{2}
  \biggprn
  {
          1-\ratio - \tfrac{\lambda}{\featnoise}
          +
          \tfrac{1}{\featnoise}
          \sqrt{\lambda^{2}
                 + 
                 2\lambda\featnoise(1+\ratio)
                 + 
                 \featnoise^{2} (1-\ratio)^{2}}
    },
    \\[4pt]
    c &\triangleq
    \tfrac{1}{2\featnoise}
  \biggprn
  {
        \frac{\featnoise(1+\ratio)+\lambda}
          {\sqrt{\lambda^2
                 + 2\lambda\featnoise(1+\ratio)
                 + \featnoise^{2}(1-\ratio)^{2}}}
          - 1
  }
  \,.\\[8pt]
\end{aligned}
\]
\end{theorem}

To aid interpretation, we first decompose the generalization loss into its constituent terms. 
\linebreak
We then provide a proof sketch outlining the key steps of the full proof presented in \appref{app:main_proofs}.

\paragraph{Decomposing the generalization loss.}
The expected generalization in \cref{eq:main_result} can be decomposed into four terms, interpreted as follows:
\begin{itemize}[leftmargin=0.5cm, itemindent=0.05cm, labelsep=0.2cm, itemsep=1pt,topsep=-1pt]
  \item \textbf{Label noise term:} arises from label noise of the underlying model, scaling with label variance.
  
  \item \textbf{Teacher variability:} captures the intrinsic mismatch between the current teacher and past teachers due to task heterogeneity.
  
  \item \textbf{Interaction term:} reflects how the deviation of the final teacher from the average teacher interacts with the sequence of teacher updates, discounting older changes.
  
  \item \textbf{Temporal correlations:} accounts for correlations between successive teacher updates, describing how past changes compound and align (or misalign) over time.
  
\end{itemize}

\paragraph{Proof sketch.}
The iterates of \cref{schm:regularized} follow the recursive form:
\[
\vw_{t}=\left(\X_{t}^{\top}\X_{t}+\lambda d\I\right)^{-1}\left(\X_{t}^{\top}\mathbf{y}_{t}+\lambda d\vw_{t-1}\right)\,. 
\]
Denote $\mP_{t}=\lambda d\left(\mathbf{X}_{t}^{\top}\mathbf{X}_{t}+\lambda d\I\right)^{-1}\!,~\mathbf S_{t:k}\triangleq\mP_t\mP_{t-1}\cdots\mP_k$
and unroll the recursive form, to obtain
\[
\vw_t-\vw_i^\star
=
\mathbf S_{t:1}(\vw_0-\vw_1^\star)
+\frac{1}{\lambda d}\sum_{k=1}^t \mathbf S_{t:k}\X_k^\top\z_k
+\sum_{k=2}^t \mathbf S_{t:k}\big(\vw_{k-1}^\star-\vw_k^\star\big)
+(\vw_t^\star-\vw_i^\star)\,.
\]
To analyze $\mathbb{E}[G_T]=\frac{1}{T}\sum_{i}\mathbb{E}\!\left\|\vw_{T}-\vw_{i}^{\star}\right\|^{2}$, we invoke a Marchenko--Pastur result to approximate the resolvent $\left(\tfrac{1}{\featnoise n}\mathbf X^\top \mathbf X - z \I \right)^{-1}$ with its deterministic equivalent \citep{liao2023rmtml}:
$$
\Bigg\Vert\,\mathbb E \left[\left(\tfrac{1}{\featnoise n}\mathbf X^\top \mathbf X - z \I \right)^{-1}\right]-\frac{ -\left(1-\tfrac{1}{\ratio}-z\right) +\sqrt{\bigl(1-\tfrac{1}{\ratio}-z\bigr)^{2} - \tfrac{4}{\ratio}z} }{ -2\ratio/z } \,\I\, \Bigg\Vert
\xrightarrow[n,d\to\infty,\,\, d/n\to 1/\alpha]{\text{a.s.}}
 0\,.
$$
We use the resolvent and standard identities to obtain closed-form expressions for
$\mathbb E[\mP_t]$, $\mathbb E[\mP_t^2]$, and $\mathbb E[\left(\frac{1}{\lambda d}\right)^2\mP_t\,\X^\top_t\X_t\,\mP_t]$
as scalar multiples of $\I$, depending only on $(\lambda,\featnoise,\ratio)$.
These are then substituted into the unrolled form to yield the explicit expression.
\jmlrQED

\paragraph{Sanity check.}
\citet{lin2023theory} studied our \cref{schm:regularized} in the limit $\lambda\to 0$, i.e., an equivalent unregularized scheme where each task is solved to convergence via SGD initialized at the previous iterate.
Below, we verify that \cref{eq:main_result} recovers their Theorem~4.1 as a special case.
\begin{example}[No Regularization]
\label{example:no_reg}
Indeed, in the high-dimensional regime, \cref{eq:main_result} generalizes Theorem~4.1 in \citet{lin2023theory}.
Specifically, after adjusting notation, their result becomes,
$$
\mathbb{E}[G_T]
    =
    \tfrac{\left(1-\ratio\right)^{T}}{T}
    \sum_{i=1}^{T}\left\Vert \vw_{i}^{\star}\right\Vert ^{2}
    +
    \tfrac{1}{T}\sum_{i=1}^{T}\left(1-\ratio\right)^{T-i}\ratio
    \sum_{k=1}^{T}\left\Vert \vw_{k}^{\star}-\vw_{i}^{\star}\right\Vert ^{2}
    +
    \tfrac{d\labelnoise}{d-n-1}
    \left(1-\left(1-\ratio\right)^{T}\right)
    ,
$$
aligning with \cref{eq:main_result} in the high-dimensional, unregularized case ($d,n \to \infty$, $\lambda\to 0$, $\featnoise=1$).
\linebreak
Proof is given in \appref{pr:sp}.
\end{example}

\subsection{Special Case: Multiple i.i.d.~Teachers}
\label{sec:noisy_teacher}

Thus far, \cref{eq:main_result} has accommodated multiple ``label-generating'' teachers.
To obtain more meaningful results, we now assume that teachers are sampled {i.i.d.}\footnote{
More generally, our analysis accommodates teachers drawn independently from several distributions with a common mean and covariance. For ease of exposition, however, we present results only for the i.i.d. setting.
}~from a common distribution.

\begin{assumption}[i.i.d.~Teachers]
\label{assump:iid_teachers}
Considering the statistical setting of \cref{sec:statistical_setting}, the teachers
\linebreak
$\vw_{1}^{\star},\dots,\vw_{T}^{\star}$ are drawn i.i.d.~from a distribution with mean \(\vw^{\star}\) and covariance \(\mathbf{\Sigma}\) where\linebreak \(\|\mathbf{\Sigma}\|_*\leq C<\infty\), independently of all other random variables (features and labels).
\end{assumption}

This assumption is more permissive than most theoretical work in continual learning (\cref{tab:comparison}), which typically assumes (i) a single teacher with label noise or (ii) joint realizability across \emph{training} sets (as in highly overparameterized regimes).
By contrast, in applications like autonomous systems and e-commerce, underlying dynamics often shift.
For instance, varying agents and simulators inevitably introduce teacher variance due to environmental gaps \citep{zhao2020sim}.
Our setup more accurately captures these practical challenges by explicitly allowing for multiple teachers.

\pagebreak

We begin with  \cref{cor:closed}, which specializes \cref{eq:main_result} to the i.i.d.~teacher setting.
We then use it to study how the regularization strength $\lambda$ impacts generalization.
In this section, all expectations are taken over features, label noise, \emph{and teachers}. For ease of presentation, we assume a zero initialization $\vw_0=0$.
Full proofs are given in \appref{allproofs}.

\begin{corollary}[Generalization under i.i.d.~Teachers]
\label{cor:closed}
Under i.i.d.~teachers  (\cref{assump:iid_teachers}), and for all $T>1$, the expected generalization loss of \cref{schm:regularized} becomes,
\begin{align*}
\mathbb E[G_{T}]=
&
\underbrace{
\labelnoise{c}\frac{1-a^{T}}{1-a}
}_{\text{Label noise}}
+
\underbrace{
a^{T}\left\Vert \vw^{\star}\right\Vert ^{2}
}_{\text{Teacher scale}}
+
\underbrace{
2\trace(\boldsymbol{\Sigma})
}_{\substack{\text{Teacher}\\\text{variance}}}
\prn{
\frac{1-b}{1-a}
+
\frac{2b^{T}-1}{T}
+
\frac{a^{T}\prn{-a +2b -1}
}{2(1-a)}
}
\,.
\end{align*}
\end{corollary}

The following lemma contrasts regularization calibrated to the task horizon $T$ against any \emph{fixed} strength. 
Specifically, the generalization loss with optimal $\lambda^{\star}(T)$ is monotonically decreasing in $T$ (under a single teacher), but non-monotonic for any fixed strength (under i.i.d.~teachers).

\begin{lemma}[More Tasks Can Be Useful or Harmful---Depending on Teacher Setup]
\label{thm:more_tasks}
Consider\linebreak the i.i.d.~teacher setting (\cref{assump:iid_teachers}) with a label noise $\labelnoise>0$.
\begin{itemize}[leftmargin=0.5cm, itemindent=0.05cm, labelsep=0.2cm, itemsep=1pt,topsep=-1pt]
    \item \textbf{Single teacher:}
    When $\vw^{\star}_1=\dots=\vw^{\star}_{T}$ (i.e., $\trace\left(\mathbf{\Sigma}\right)=0$), the expected generalization loss with an \emph{optimal} horizon-dependent strength $\lambda^\star(T)$ is monotonically \emph{decreasing} in the task horizon $T$. 
    \item \textbf{Multiple i.i.d.~teachers:} When teachers vary across tasks (i.e., $\trace\!\left(\mathbf{\Sigma}\right)>0$), then under any \emph{fixed} strength $\lambda>0$, the expected generalization loss is monotonically \emph{increasing} in the task horizon $T$ for all $T \geq T'$ for a sufficiently large $T'$.
\end{itemize}
\end{lemma}

\subsubsection{Optimal Scaling of the Regularization Strength:
$\displaystyle\lambda^{\star}\asymp\frac{\featnoise \alpha T}{\log(\text{signal}/\text{noise})}$}
Now, we characterize the optimal regularization strength $\lambda^{\star}$ that minimizes the generalization loss.
A theory-guided rule for $\lambda$ is essential when hyperparameter tuning is prohibitive, as in continual learning.
In our setting, we prove a near linear optimal scaling with $\alpha, \featnoise, T$ leading to mean-square convergence to the ``global'' teacher, i.e., $\vw_{T}\xrightarrow[]{T\to\infty}\teacher$.

\begin{theorem}[Optimal Regularization Scales Nearly Linearly with $T$]
\label{thm:optimal}
Under i.i.d.\ teachers (\cref{assump:iid_teachers}) with non-zero\footnote{Since $\vw_{0} = \0$, 
additionally having $\mathbb{E}[\teacher_{t}] = \0$ would directly imply an optimal regularization of $\lambda^{\star}\to\infty$.} mean teacher $\teacher$, the optimal fixed regularization strength that minimizes the expected generalization 
loss after $T$ iterations satisfies 
$\lambda^\star
=
\lambda^\star(T)\asymp\frac{T}{\ln T}$.
More precisely, for any $\epsilon>0$, there exists
$T_0 = O\!\left(\hfrac{1}{\epsilon^2}\right)$ such that for all
$T \ge T_0$,
\[(1-\epsilon)
\frac{2\featnoise \alpha T}{
\ln
\Bigprn{
    \frac{4\alpha T\left\Vert \vw^{\star}\right\Vert ^{2}\featnoise
    }
    {{\labelnoise+
    \featnoise\mathrm{Tr}(\boldsymbol{\Sigma})
    \left(1+\ratio\right)}
    }
}
}
<
\lambda^\star
<
(1+\epsilon)
\frac{2\featnoise \alpha T}{
\ln
\Bigprn{
    \frac{4\alpha T\left\Vert \vw^{\star}\right\Vert ^{2}\featnoise
    }
    {\labelnoise+\featnoise\mathrm{Tr}(\boldsymbol{\Sigma})
    \left(1+\ratio\right)}
    }
}
\,.
\]
Furthermore, in the degenerate noiseless case of $\mathrm{Tr}(\boldsymbol\Sigma)=\labelnoise=0$, it holds that $\lambda^\star\to 0$.
\end{theorem}

\paragraph{Interpretation.}
The denominator can be interpreted as a $\log(\text{signal}/\text{noise})$ factor.
The \emph{signal} consists of $\norm{\teacher}^2 \featnoise$, capturing the strength of shared structure across tasks,
which is reinforced over longer task horizons $T$,
while $\frac{n}{d}\to\ratio$ measures how informative each task is.
The internal denominator scales with \emph{noise} sources: label noise $\labelnoise$,
feature noise $\featnoise$,
and teacher variance $\trace(\mSigma)$.

In \cref{sec:related}, we contrast our optimal strength with results from prior work.

\pagebreak

\cref{thm:more_tasks} indicates that for fixed $\lambda$, performance deteriorates as $T$ exceeds a certain threshold (resembling overfitting). Consequently, given an unlimited ``budget'' of tasks, the optimal strategy is to increase regularization as much as possible, as formalized below. 

\begin{theorem}[Asymptotic Generalization under Strong Regularization]
\label{thm:asymptotic_mse}
Under i.i.d.~teachers \linebreak(\cref{assump:iid_teachers}), as the horizon $T \to \infty$, expected generalization loss decreases monotonically with strength $\lambda$.
Furthermore, the iterates converge in mean square to the ``global'' teacher $\teacher$:\footnote{
Limits are taken sequentially (first $T\to\infty$, then $\lambda\to\infty$) rather than simultaneously.
}
$$
\lim_{\lambda\to\infty} 
\Bigprn{\,\lim_{T \to \infty}  
\mathbb{E}
\left\Vert \vw_{T} - \vw^{\star}\right\Vert^{2}
\,}= 0\,.
$$
\end{theorem}

\section{Experiments on Synthetic Regression Data}
\label{sec:synthetic}

Next, we empirically demonstrate how regularization shapes generalization and contrast it with the unregularized case, validating our theoretical analysis.

\paragraph{Noisy teacher setup.}
For each task $i\in [T]$, we sample a teacher $\teacher_i \sim\mathcal{N}(\teacher,\,\tfrac{1}{d}\I)$, matching \cref{assump:iid_teachers}. 
We use $d=50$ dimensions, $n = 30$ examples per task, and initialization $\vw_0 = \0$.

\paragraph{Longer task sequences with stronger regularization overcome label noise.}

\cref{fig:Gen_T} shows that regularization reduces noise effects, but only when calibrated to the task horizon. 
We observe that models with weaker regularization benefit from early tasks but soon plateau\footnote{
\cref{thm:more_tasks} predicts that the expected generalization monotonically increases beyond a certain task. While this trend is obscured here by noise under the chosen hyperparameters, it becomes clearer in the figures of \cref{sec:mnist}.
}
at a noise floor, reflecting high plasticity but low stability. 
Conversely, excessive regularization impairs plasticity, effectively preventing learning. 
\cref{fig:tasks} confirms that the optimal regularization strength scales nearly linearly with the horizon---balancing the stability–plasticity tradeoff to lower the noise floor. 

\begin{figure}[h!]
\centering
\subfigure[\small{Regularization can decrease the label noise floor.}]{
  \centering
  \parbox{.49\columnwidth}{\centering
    \includegraphics[width=1\linewidth]{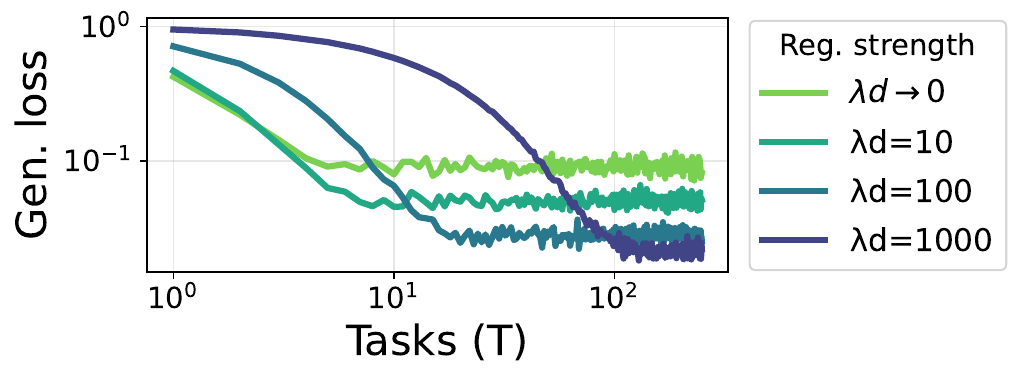}%
  }%
  \label{fig:Gen_T}
}
\hfill
%
\subfigure[\small Optimal strength grows nearly linearly with the number of tasks, as predicted by \cref{thm:optimal}.\linebreak
We compute $\lambda^{\star}$ with an exhaustive search.]{
  \centering
  \parbox{.47\columnwidth}{\centering
    \includegraphics[width=.86\linewidth]{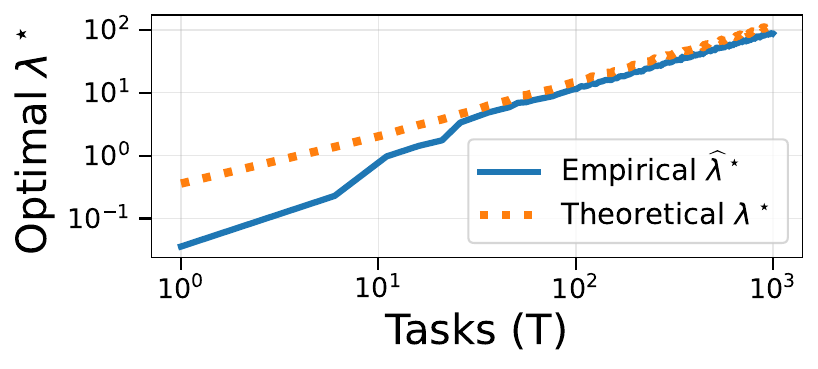}%
  }%
  \label{fig:tasks}
}
\caption{
\textbf{Regularization effects on synthetic data.}
Interactions between number of tasks ($T$), regularization strength ($\lambda$), and resulting generalization loss (\emph{normalized} by a trivial regressor $\vw=\0$).\linebreak
Here, the label noise and feature variances are $\labelnoise = \featnoise = 1$; Each curve averages over $5$ runs.
\label{fig:longer_seqeuences}
}
\vspace{-1em}
\end{figure}

\paragraph{While the unregularized scheme collapses under label noise, optimal regularization endures.}

The unregularized scheme solved via gradient descent to convergence corresponds to our \cref{schm:regularized} as $\lambda\to 0$.
While \citet{evron2022catastrophic,evron2025better} analyze this setup's training loss in jointly realizable---or noiseless---settings, \citet{lin2023theory} accommodates label noise and studies the generalization loss.
However, as our \cref{example:no_reg} shows, that scheme fails under high label noise due to a noise floor $\frac{d\labelnoise}{d-n-1}$ that \emph{persists} as $T\to \infty$ for any number of samples $n$ and dimension $d$ (where $\alpha<1$).
In~contrast, we demonstrate next that a properly calibrated regularization in  \cref{schm:regularized} overcomes even immense label noise $\labelnoise$, provided a sufficient task horizon $T$.

\vspace{1em}

\begin{wrapfigure}[10]{b!}{0.55\textwidth}
\vspace{-1.3em}
\centering
\includegraphics[width=\linewidth]{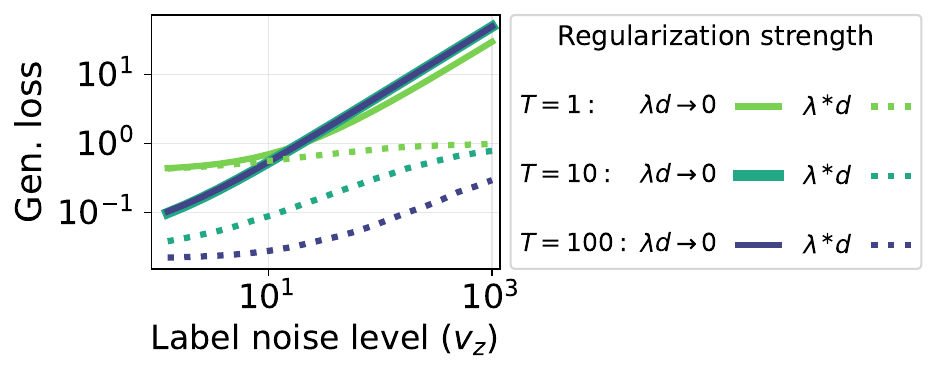} 
\vspace{-2em}
\captionof{figure}{
{Optimal regularization substantially improves generalization. 
Here, feature variance is ${\featnoise=1}$; each curve averages over 20 runs.
\label{fig:gensig}
}
}
\end{wrapfigure}

\cref{fig:gensig} compares the generalization loss of the aforementioned unregularized scheme ($\lambda\to 0$) with that of the optimal $\lambda^\star$.
Indeed, in the presence of high label noise:
(i) the unregularized scheme fails completely due to the high floor noise, \emph{regardless} of the task horizon;
and \linebreak
(ii) a properly calibrated $\lambda^\star$ can substantially mitigate label noise after seeing enough tasks.

\section{Experiments on \mnist-Based Dataset}
\label{sec:mnist}

While our analysis assumes i.i.d.~features and a linear model, we now validate our results in more challenging scenarios where the assumptions of \cref{sec:setting} do not hold.
We use finite-dimensional \mnist data rather than high-dimensional i.i.d.~features, solve classification instead of regression,
inject bit-flip rather than continuous label noise,
and evaluate both linear models and ReLU networks.
In all cases, the
optimal regularization strength scales as $\hfrac {T}{\ln T}$, consistent with \cref{thm:optimal}.

\subsection{Constructing a Continual Binary Classification Dataset with Noisy (Flipped) Labels}
\label{sec:MNIST_setting}
\paragraph{Samples.}
We convert the \mnist (multiclass) dataset into a binary classification problem by grouping digits 0--4 into class 0 and digits 5--9 into class 1.
To facilitate our teacher models (below), we narrow the data to a random subset of 500 training images and 200 test images.

\paragraph{Teacher models.}
We consider two types of teachers:
\begin{enumerate}[leftmargin=0.8cm, labelsep=0.3cm, itemsep=3pt,label=(\alph*)]

\item \textbf{A single noiseless teacher} that outputs the true labels (tasks still suffer from \emph{label} noise).

\item \textbf{Multiple i.i.d.~teachers drawn from a global teacher.} 
We first train a global teacher with the same architecture as the evaluated model---linear in \cref{sec:experiments_linear} and neural in \cref{sec:experiments_nns}---on the \emph{entire} subsample (train and test).
Then, each task draws a teacher by adding Gaussian noise to the global teacher's \emph{weights}.
Teachers output binary labels (thresholded at 0.5).

\end{enumerate}

\paragraph{Noisy labels.}
Each task consists of the \emph{same} sampled images, but differs in labels due to two noise sources.
Specifically, each training label is first set by the evaluated teacher (noiseless or noisy), and then corrupted by bit-flip \emph{label} noise---i.e., each label is flipped (i.i.d.) with probability $p=0.2$.

\subsection{Linear Regression Experiments: Validation Beyond i.i.d.~Features}
\label{sec:experiments_linear}

We start from a regularized continual linear regression model (\cref{schm:regularized}), trained with MSE to predict (noisy) \emph{binary} labels generated by either a single noiseless teacher (\cref{fig:GT_single_teacher_ridge}) or multiple teachers drawn from a global teacher (\cref{fig:GT_many_teachers_ridge}).
Predictions are thresholded at $0.5$ to yield binary outputs. 
Empirically, the optimal regularization scales roughly as $\hfrac {T}{\ln T}$, consistent with \cref{thm:optimal} even though the i.i.d.~feature assumption does not hold.

\pagebreak

\begin{figure}[t!]
\centering
\subfigure[\textbf{Linear:} Single teacher.]{
  \centering
  \includegraphics[width=.54\columnwidth]{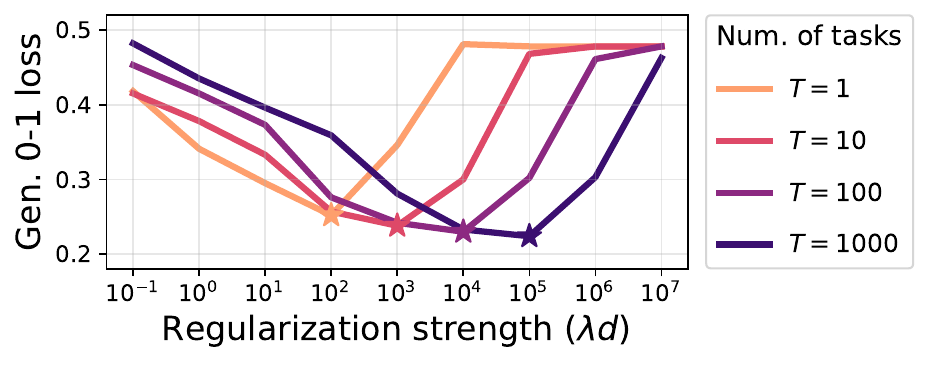}
  \label{fig:GT_single_teacher_ridge}
}
\hfill
%
\subfigure[\textbf{Linear:} Multiple i.i.d.~teachers.]{
  \centering
  \includegraphics[width=.41\columnwidth]{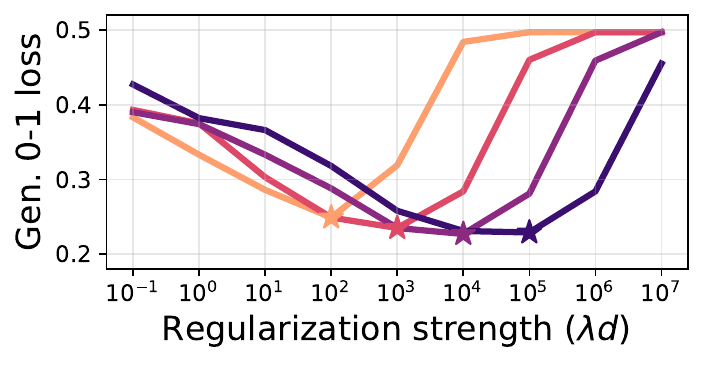}
  \label{fig:GT_many_teachers_ridge}
}
\\
\vspace{2pt}
\subfigure[\textbf{NN:} Single teacher.]{
  \centering
  \includegraphics[width=.54\columnwidth]{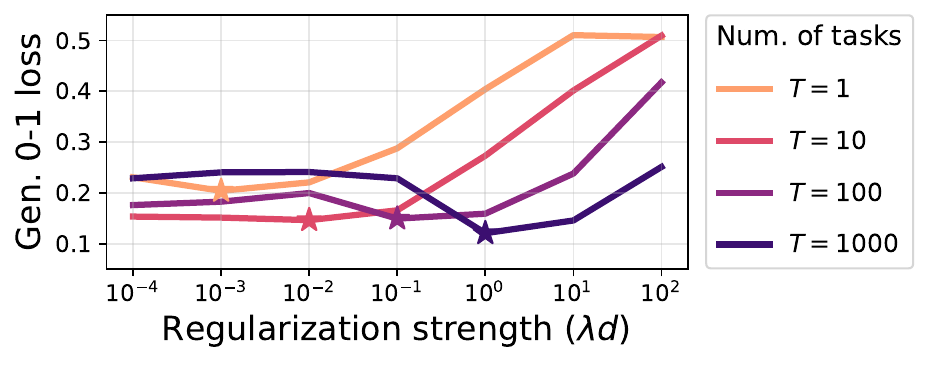}
  \label{fig:GT_single_teacher}
}
\hfill
%
\subfigure[\textbf{NN:} Multiple i.i.d.~teachers.]{
  \centering
  \includegraphics[width=.41\columnwidth]{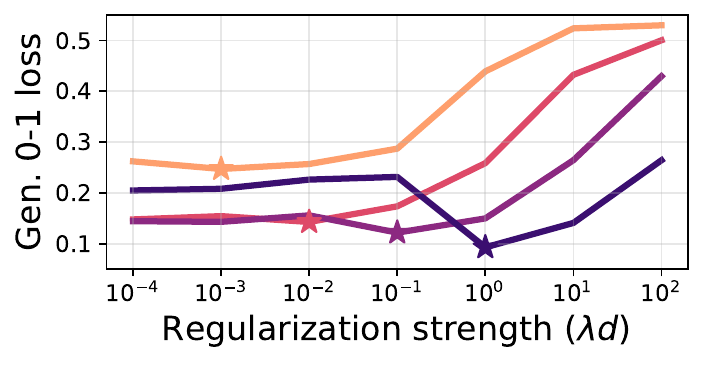}
  \label{fig:GT_many_teachers}
}
\vspace{-1pt}
\caption{
\textbf{Regularization in linear models and neural networks (NN).}
We plot the held-out classification error after learning $T$ noisy tasks. 
Curves are averaged over 5 seeds; stars mark minima.
\label{fig:all_plots_ridge}
\vspace{-0.7em}
}
\end{figure}

\subsection{Neural Network Experiments: Validation Beyond Linear Models}
\label{sec:experiments_nns}

We also consider a two-layer ReLU neural network (784→500→1) with a sigmoid output, trained continually with L2 regularization.
Again, noisy binary labels are generated by either a single noiseless teacher or multiple teachers sampled from a global (linear) teacher (\cref{fig:GT_single_teacher,fig:GT_many_teachers}).
We optimize MSE for consistency with \cref{schm:regularized}, thresholding predictions at $0.5$.
Even in this neural \linebreak setting, optimal regularization follows the $\hfrac {T}{\ln T}$ scaling, matching the linear theory in \cref{thm:optimal}.


\begin{takehome}[Practical Recipe for Regularization Strength $\lambda$]
    We observe that the optimal regularization scaling 
    $\lambda^{\star}(T) \asymp T$ persists across architectures, potentially up to log factors.\footnote{
    While \cref{thm:optimal} suggests a $T/\ln T$, other unknown logarithmic terms---e.g., $\norm{\textbf{w}^*}^2$---may dominate the denominator in practice. 
    Consequently, we recommend focusing on the more significant linear scaling in $T$.
    } 
    To determine the leading constants, first tune $\lambda$ on a short task horizon and then scale it for longer horizons.
    
    For example, for $T\!=\!1000$, tune on $T\!=\!10$ and set 
    $\lambda(1000)=
    \frac{1000}{10}
    \,\widehat{\lambda^{\star}}(10)
    =100 \,\widehat{\lambda^{\star}}(10)$.
\end{takehome}

\vspace{-0.7em}


\section{Related Work and Discussion}
\label{sec:related}

\begingroup
\makeatletter
\let\Hy@backout\@gobble
\begin{table}[b!]
\renewcommand{\arraystretch}{1.3}
\setlength{\tabcolsep}{4pt} 
\centering
\captionsetup{font=small}
\vspace{-0.5em}
\caption{
\textbf{Comparison to related work on continual linear regression.}
We highlight papers most closely related to our work—primarily ones that also adopt a statistical perspective.
The second column shows the studied regularization method, and the number of regularization parameters it requires. 
\label{tab:comparison}
}
\vspace{-0.5em}
\small
\begin{tabular}{p{3.4cm}p{3.3cm} cc p{4.2cm}}
\Xhline{1.2pt}
\textbf{Paper} 
& 
    \parbox[c]{3.3cm}{
    \textbf{Regularization}\newline%
    ({\# reg.~parameters})
    }
& 
\parbox[c]{1.4cm}{\vspace{0.1cm} \centering \textbf{Multiple{\newline}teachers?}\vspace{0.1cm}}
& 
\parbox[c]{1.4cm}{\vspace{0.1cm}\centering \textbf{Label noise?}\vspace{0.1cm}}
 & \textbf{Features} 
 \\
\Xhline{1.2pt}
\citet{lin2023theory}
& None \hfill (N/A)
& \cmark & \cmark & Gaussian i.i.d.~entries 
\\
\hline
\citet{li2023fixed} &
L2 \hfill
(\complexitylow{scalar}) & \xmark &
\cmark
& Fixed matrices with{\newline}commutable covariance \newline (2 tasks only)
\\
\hline
\citet{zhao2024statistical} & Aligned with {\newline}covariance \hfill 
(\complexityhigh{matrix})
& \xmark & \cmark & 
Fixed matrices with{\newline}commutable covariance
\\
\hline
\citet{ding2024understanding} & 
Implicit from{\newline}step budget \texttt{(*)}  \hfill (N/A) & \xmark & \cmark & 
Random matrices with{\newline} commutable covariance \texttt{(**)}\!
\\
\hline
\citet{zhu2025global} & Based on all \newline directions \hfill
(\complexityhigh{matrix}) & \cmark & \cmark &  
Any
\\
\hline
\citet{levinstein2025optimal} & L2 \hfill
(\complexitylow{scalar})
 & \xmark & \xmark & 
Any fixed matrices \newline(in random task ordering)
\\
\hline
\citet{zhao2025highdimensional}
{\newline}{(Concurrent)}
& L2 \hfill
(\complexitylow{scalar})
& \xmark & \cmark &  
Random matrices with{\newline}commutable covariance
\\
\Xhline{1.2pt}
\textbf{Ours} & L2 \hfill
(\complexitylow{scalar})
& \cmark & \cmark & Zero-mean i.i.d.~entries
\\
\Xhline{1.2pt}
\end{tabular}
\vspace{0.5em}
\caption*{
\small
\raggedright 
\,\texttt{(*)}\, 
A finite number of steps induces an implicit bias towards previous solutions \citep{jung2025convergence}.
\\
\texttt{(**)} Their main proof implicitly assumes commutable covariance matrices in Eq.~(13).
}
\vspace{-1.em}
\setlength{\tabcolsep}{6pt} 
\end{table}
\makeatother
\endgroup

We extend the literature on generalization in statistical continual learning settings \citep[e.g.,][]{lin2023theory,li2023fixed,ding2024understanding,zhao2025highdimensional} by incorporating both \emph{multiple} teachers and L2 regularization. 
While optimization-focused studies typically provide worst-case upper bounds \citep[e.g.,][]{evron2022catastrophic,evron2025better,levinstein2025optimal}, our statistical approach enables an \emph{exact} average-case generalization analysis. 
This allows us to establish, to the best of our knowledge, the first characterization of optimal regularization strength, spanning both classical single-teacher and general i.i.d.~multi-teacher frameworks.
A comparative overview is provided in \cref{tab:comparison}.

\pagebreak

\paragraph{Optimal regularization.}
Much of the research on continual learning has focused on designing regularization matrices that perform well empirically \citep{kirkpatrick2017ewc,zenke2017continual,aljundi2018memory}. Such work typically emphasizes the regularization \emph{directions} or relative magnitudes of coefficients, while leaving the overall strength $\lambda$ to hyperparameter tuning. In contrast, L2 regularization, which also exhibits strong empirical performance \citep{Hsu2018Baselines,lubana2021regularization,smith2022closer}, offers a convenient analytical proxy for studying regularization strength, as it depends only on a single tunable parameter, $\lambda$.

\citet{levinstein2025optimal} studied L2 regularization under random task orderings with arbitrary fixed features and a single noiseless teacher. 
They showed that a fixed\footnote{As in our work, their strength is static during training and determined a priori by the horizon $T$. 
In contrast, they also employ a dynamic strength---increasing at every iteration---to achieve an optimal worst-case rate.
}
$\lambda\asymp\ln T$ yields near-optimal worst-case convergence guarantees, though they do not rule out the possibility that other fixed strengths could perform better (see their Appendix~A).
Interestingly, in a similar single-teacher setting with no label noise, our analysis dictates that $\lambda^\star \rightarrow 0$ is the optimal strength for generalization (\cref{thm:optimal}). 
We find that $\lambda^\star > 0$ emerges only in the presence of non-zero label noise or teachers that vary between tasks.
Indeed, in our average-case setting, while regularization slows down convergence, it decreases the asymptotic error caused by label noise or teacher variation. 
Finally, under a \emph{cyclic} task ordering with $N$ cycles over a sequence of $T$ tasks, \citet{cai2025lastIterate} suggest a fixed $\lambda\asymp T\sqrt{\ln N}$. 
While their cyclic ordering and our arbitrary one are not directly comparable, their linear dependence on $T$ appears mostly consistent with our findings.
\pagebreak

\paragraph{Regularization weighting matrices and covariance structure.} 
\citet{zhao2024statistical,zhu2025global} build a weighting matrix directly from the empirical covariance, leading to a worst-case requirement of $d^2$ regularization parameters. However, this might be considered as ``too much'' memory for realizable linear regression, as with so many parameters we can avoid forgetting entirely (\citealt{evron2023classification}, Proposition~5.5). In contrast, we use a practical single regularization parameter.

It is also important to distinguish between the observed empirical covariance $\tfrac{1}{n}\X^\top \X$ and the population covariance that characterizes the sample-generating distribution. 
For example, we explore L2 regularization, i.e., a weighting matrix of $\lambda\I$ assuming a feature \textit{population} covariance of $\featnoise\I$.
Alternatively, some prior work assumes a design with \emph{fixed} features, but require that the \emph{empirical} covariance matrices commute \citep[e.g.,][]{li2023fixed,li2025memory,zhao2024statistical}.

\paragraph{Beyond isotropic covariance.}
We assumed a $\boldsymbol{\Sigma} = v_x \mathbf{I}$ feature covariance.
However, our analysis can extend to a known, spectrally bounded covariance matrix
$c\I \preceq \boldsymbol{\Sigma} \preceq C\I$ with Gaussian features, by modifying the algorithm
to include a whitening step.
That is, apply the same algorithm (\ref{schm:regularized}), 
$\vw_t
= \argmin
\Bigl\{
\bigl\lVert \tilde{\mathbf{X}}_{t}\vw - \mathbf{y}_{t} \bigr\rVert^{2}
+ \lambda d \bigl\lVert \vw - \vw_{t-1} \bigr\rVert^{2}
\Bigr\}$, with whitened matrices $\tilde{\mathbf{X}}_{t} = \mathbf{X}_{t} \boldsymbol{\Sigma}^{-1/2}$.

\paragraph{Connections to prior-based methods.}
Our work analyzes L2 regularization in continual learning, where an additional $\lambda \norm{\vw-\vw_{t-1}}^2$ term acts as a dynamic \emph{prior} to stabilize updates.
This formulation is fundamentally a proximal method \citep{cai2025lastIterate} and shares heritage with online and meta-learning frameworks that utilize static or dynamic priors, such as FOBOS \citep{duchi09fobos}, FTRL and FTRL-Proximal \citep{mcmahan2011ftrl}, and ARUBA \citep{khodak2019aruba}.
Although those methods may provide stronger \emph{regret} guarantees than our proximal method, they depart from standard regularization methods in continual learning \citep[e.g.,][]{kirkpatrick2017ewc}, especially ones involving isotropic regularization \citep[e.g.,][]{smith2022closer}.
While exploring the connections between meta-learning and non-isotropic weight matrices remains a promising direction for future work, our current setting isolates the effects of regularization \emph{strength}, providing direct practical implications for common training recipes.

\paragraph{Task typicality and ordering.} 
While we do not focus on aspects of task ordering, 
our \cref{eq:main_result} nevertheless provides insights into how the teacher ordering influences generalization. 
For example, the theorem decomposes the loss and reveals a \emph{teacher variability} term that is minimized when the teacher of the \emph{final} task is the one closest to the mean of teachers $\bar{\vw}^\star$. 
In other words, learning should end with the most \emph{typical} task, matching recent findings by \citet{li2025optimal}. 

\paragraph{Label noise vs.~teacher variance.}
The teacher variance and label noise are naturally related. 
In our setting (\cref{sec:statistical_setting}), the labels of task $t\in[T]$ are generated as
$\y_{t}=\X_{t}\vw_{t}^{\star}+\z_{t}$
for a teacher $\teacher_{t}$ and random label noise $\z_{t}$. 
If the teacher itself exhibits noise as in \cref{sec:noisy_teacher}, labels can be rewritten as
$
\y_{t}=\X_{t}\bigl(\teacher+\boldsymbol{\xi}_{t}\bigr)+\z_{t}
=\X_{t}\vw^{\star}+\tilde{\z}_{t},
$
where $\tilde{\z}_{t}=\X_{t}\boldsymbol{\xi}_{t}+\z_{t}$.
However, these two sources are not equivalent: the label noise $\mathbf z_t$ is independent of the design matrix $\X_t$, 
while the effective noise $\tilde{\mathbf z}_{t}$ inherently \emph{depends} on it, leading to non-trivial theoretical and practical implications.

\paragraph{Theoretical vs.~empirical scaling} 
We demonstrate that the qualitative prediction of \cref{thm:optimal}---namely, that $\lambda^{\star}(T)\asymp\hfrac {T}{\ln T}$---is consistent across both our linear regression and neural network experiments (\cref{fig:all_plots_ridge}). 
However, the \emph{quantitative} prediction---i.e., the exact value accounting for problem-dependent factors---does not align perfectly with the empirical optimum, as illustrated in \cref{fig:C} of \cref{app:not_aligned}.
To isolate the source of this mismatch, we show that second-order feature correlations are largely responsible: by introducing a preprocessing whitening step (see earlier discussion), the observed optimal regularization aligns with the theoretical prediction  (see \cref{fig:A} in \cref{sec:whitening}). 
Finally, we validate a nearly linear dependence of the empirically optimal regularization on the aspect ratio $\alpha$ (\cref{fig:B} in \cref{sec:aspect}), in agreement with the scaling suggested by \cref{thm:optimal}.

\paragraph{Future directions.} In this work, we analyzed continual learning in a simplified setting with a linear model and isotropic regularization, assuming i.i.d.~features and i.i.d.~teachers. 
Despite these simplifications, we extracted a meaningful scaling law that appears to hold in practical neural network settings: the optimal regularization strength increases nearly linearly with the number of tasks $T$.
Future research could relax these assumptions by considering non-linear models, more expressive non-isotropic regularization, or non-i.i.d.~environments (e.g., Markovian task sequences). 
We hope our analysis provides a foundation for more robust algorithms and a deeper theoretical understanding of continual learning.

\acks{The research of DS was funded by the European Union 
{(ERC, A-B-C-Deep, 101039436)}. Views and opinions expressed are however those of the author only and do not necessarily reflect those of the European Union or the European Research Council Executive Agency (ERCEA). Neither the European Union nor the granting authority can be held responsible for them.

The research of RM was partially supported by the Israel Science Foundation (grant 1693/22) and by the Skillman chair.
}

\bibliography{99_biblio}

\newpage

\appendix

\renewcommand{\thetheorem}{\Alph{section}.\arabic{theorem}}
\renewcommand{\thelemma}{\Alph{section}.\arabic{lemma}}

\makeatletter
\@addtoreset{theorem}{section}
\makeatother

\startcontents[app]

\begingroup
\renewcommand{\baselinestretch}{1.5}\selectfont
\appendixtableofcontents
\endgroup
\newpage

\section{Table of Notations}
\label{app:notations}

\vspace{-0.7em}
\begin{table}[H]
\centering
\caption{Notation table.}
\label{tab:notation}
\setlength{\tabcolsep}{4pt} 
\setlength{\extrarowheight}{6pt}
\begin{tabular}{l p{0.53\linewidth}}
\toprule
\textbf{Notation / Definition} & \textbf{Short description} \\
\midrule
$d$ & Dimension (number of input features) \\
$n$ & Number of data items (samples) per task \\
$\ratio\triangleq \dfrac{n}{d}$ & Aspect ratio (sample-to-dimension ratio) \\
$T$ & Number of iterations (tasks) \\
$\labelnoise$ & Label noise variance \\
$\featnoise$ & Feature noise variance \\
$\lambda$ & Regularization strength \\
$\mathbf{X}$ & Data set (feature matrix) \\
$\mathbf{y}$ & Labels \\
$\mathbf{z}$ & Label noise vector \\
$\mathbf{w}$ & Predictor (model parameters) \\
$\teacher$ & Teacher (ground-truth parameter) \\
$\bar{\vw}^{\star} \triangleq \frac{1}{T}\sum_{i=1}^T \teacher_i$ & Average teacher \\
$G_t$ & Generalization error at iteration $t$ \\
$\boldsymbol{\Sigma}$ & Teacher covariance matrix \\
$\displaystyle \mathbf P \triangleq \lambda d\!\left(\mathbf{X}^{\top}\mathbf{X}+\lambda d\,\mathbf{I}\right)^{-1}$
& Regularized inverse covariance matrix 
\\
$\displaystyle \tilde{\lambda} \triangleq \dfrac{\lambda}{v_x}$ & Scaled regularization strength \\
$\boldsymbol{\xi}$ & Random variable modeling teacher noise \\
$\displaystyle D \triangleq \sqrt{\tilde{\lambda}^2 + 2\tilde{\lambda}(1+\ratio) + (1-\ratio)^2}$
& Discriminant term appearing in closed-form expressions \\
$\displaystyle a \triangleq \tfrac{1}{2}\!\left(
1-\ratio+
\frac{\tilde{\lambda}(1+\ratio)+(1-\ratio)}{D}
\right)$
& Memory contraction coefficient \\
$\displaystyle b \triangleq \tfrac{1}{2}\!\left(
1-\ratio-\tilde{\lambda}+D
\right)$
& Weight assigned to task-to-task drift \\
$\displaystyle c \triangleq \tfrac{1}{2}\!\left(
(1+\ratio)+\tilde{\lambda}-D
\right)$
& Noise amplification coefficient \\
$\displaystyle f \triangleq v_z\,c$ & Effective label-noise contribution \\
\bottomrule
\end{tabular}
\end{table}

\newpage

\section{Additional Continual Linear Regression Experiments}
\label{sec:exp1}
In this appendix, we experiment with the regularized continual linear regression model
(\cref{schm:regularized}). 
Compared to the previous setting of \cref{sec:MNIST_setting}, we introduce additive Gaussian label noise with variance $0.2$, and construct each task from a \emph{different} subsample of images. We report the generalization error corresponding to the mean squared error (MSE). Furthermore, the teacher is evaluated on a random subset of 500 sampled data points. 

\subsection{Theory vs.~Experiments: Qualitative, but not Quantitative, Alignment}
\label{app:not_aligned}

We train the model on the \emph{normalized}\footnote{By normalization we mean per-feature centering and scaling to unit variance, i.e., $x_j \leftarrow (x_j - \hat {\mathbb{E}}[x_j]) / \sqrt{\widehat{\mathrm{Var}}(x_j)}$; unlike whitening, this preprocessing does not remove cross-feature correlations.} \mnist{} dataset. 
We visualize the empirical generalization error\footnote{The empirical generalization error is normalized by the maximum empirical generalization error across all regularization strengths for the specific task horizon $T$. 
That is, each column in the heatmap is scaled between $0$~and~$1$.} as a heat map over the regularization strength and number of tasks.
For comparison, we also plot a curve of the optimal regularization strength predicted by \cref{thm:optimal}. 

The experiment demonstrates that in the absence of whitening (see next section), there is a mismatch between the empirically optimal regularization and the theoretical prediction.

\begin{figure}[H]
\centering
  \includegraphics[width=.9\columnwidth]{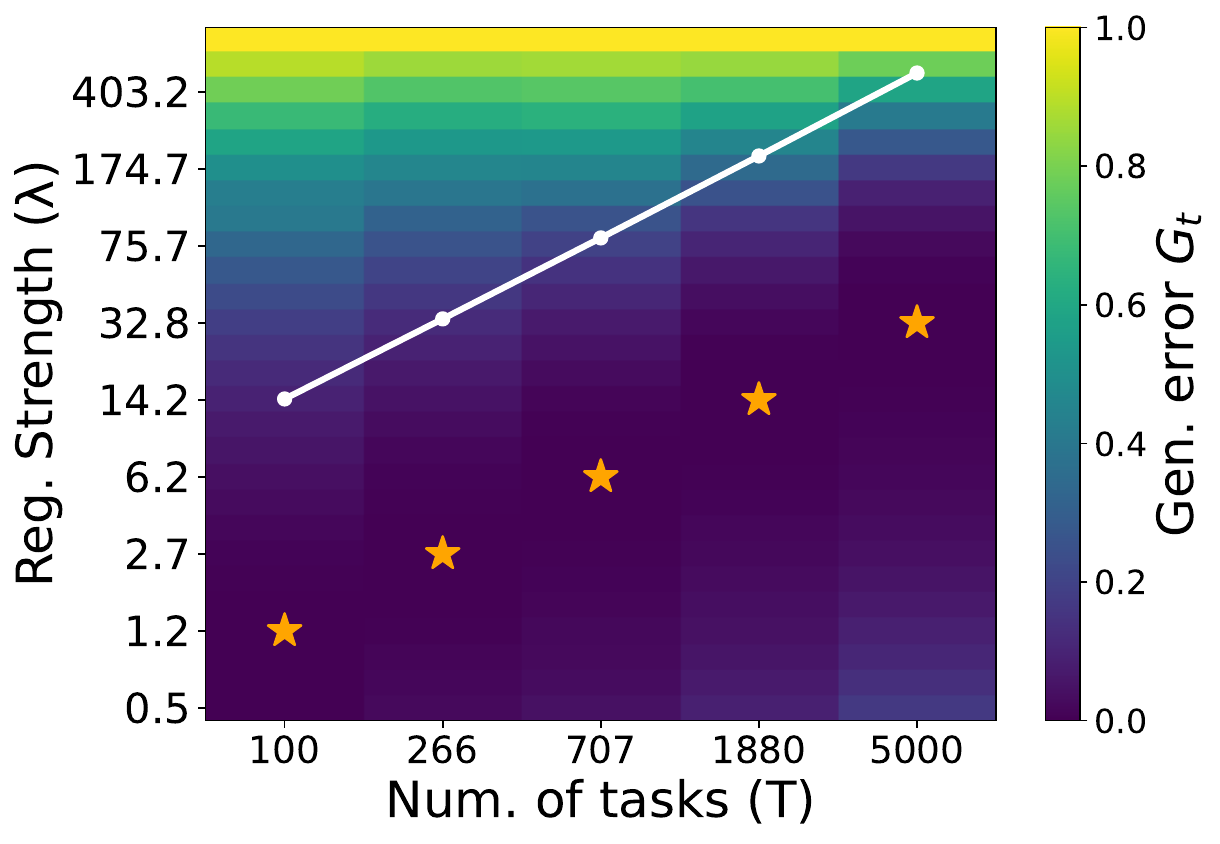}
\caption{
\textbf{Regularization strength vs. Task horizon.}
Stars mark the empirical optimal regularization strength obtained after training on a noisy sequence of $T$ tasks and averaging over 5 random seeds.
We also plot the analytical optima, predicted by our \cref{thm:optimal}; shown as a white curve.
We observe a seemingly-multiplicative mismatch between the empirical and analytical optima.
}
\label{fig:C}
\end{figure}

\subsection{A Whitening Step Reconciles Experiments with Theory}
\label{sec:whitening}
Here, prior to training, the input
images are whitened using the empirical feature covariance computed over the entire dataset of \mnist. 
The empirical optimal strength is obtained by searching over a (one dimensional) grid centered at the theoretically predicted optimum, and compared to the theoretical optimal regularization predicted by \cref{thm:optimal}.  

We observe that the optimal regularization strength matches the prediction of \cref{thm:optimal} almost perfectly.
This implies that the mismatch in the previous section stems largely from second-order feature correlations (while our theory assumes a isotropic covariance matrix).

\begin{figure}[H]
\centering
\subfigure[Single teacher.]{
  \centering
  \includegraphics[width=.58\columnwidth]{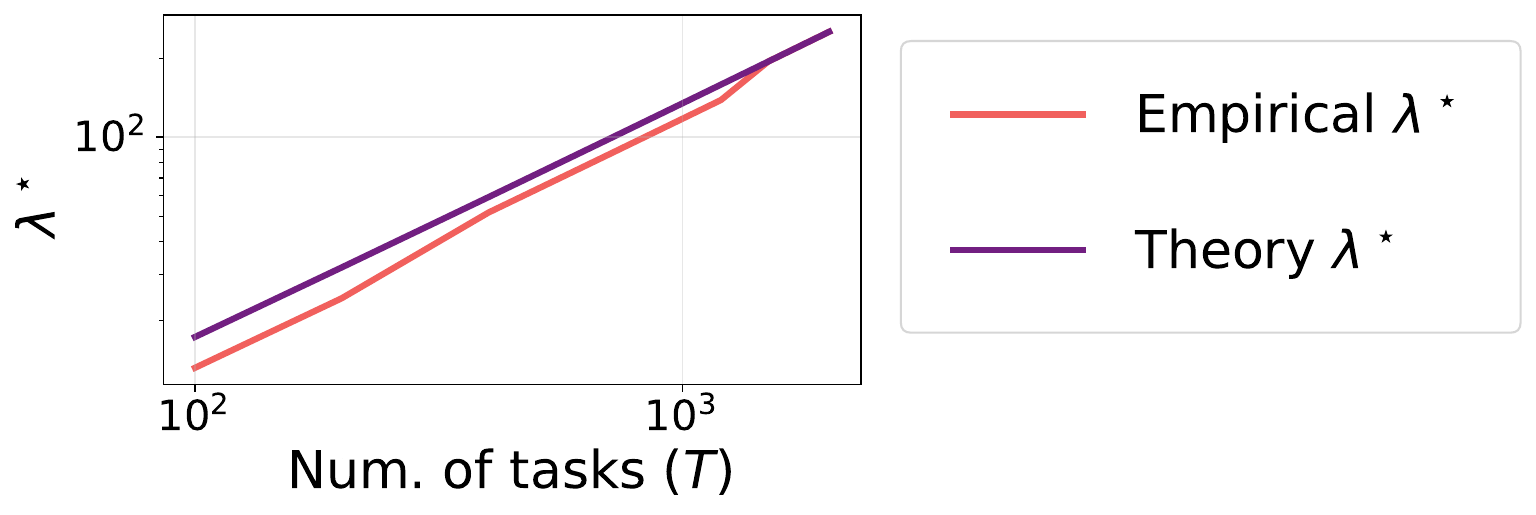}
}
\hfill
%
\subfigure[Multiple i.i.d.~teachers.]{
  \centering
  \includegraphics[width=.34\columnwidth]{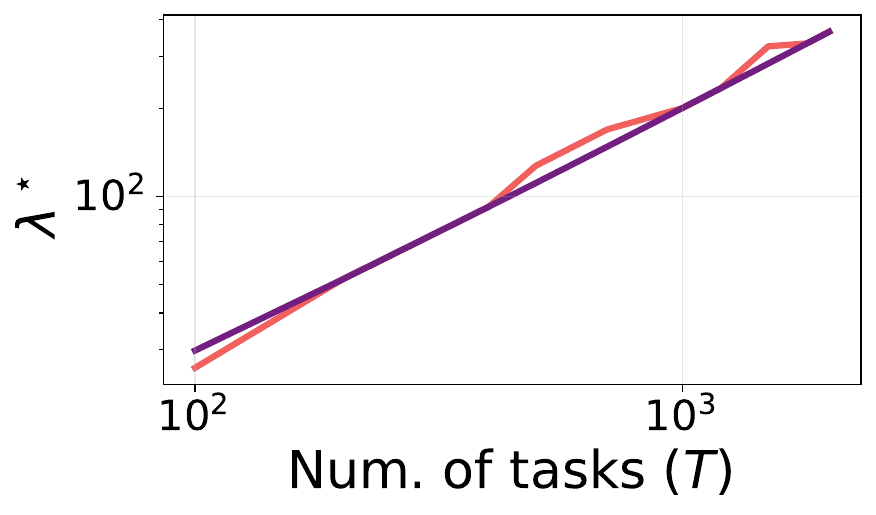}}
\caption{
\textbf{Empirical optimal regularization strength vs.~Task horizon.}
We plot the empirical optimal strength after learning a noisy sequence with $T$ tasks---averaged over 5 seeds---and compare it to the optimum predicted by our analysis.
}
\label{fig:A}
\end{figure}

\subsection{Aspect-Ratio Scaling Predicted by Theory Extends Beyond i.i.d.~Features}
\label{sec:aspect}
We evaluate a linear model trained on the original \emph{non-whitened} \mnist-based data, but vary the aspect ratio $\tfrac{n}{d}\to\alpha$.
Here, we focus on the single noiseless teacher. 
Still, the optimal regularization exhibits a near linear dependence on $\alpha$, in agreement with the scaling predicted by \cref{thm:optimal}.
\begin{figure}[H]
\vspace{0.5em}
\centering
  \includegraphics[width=0.7\columnwidth]{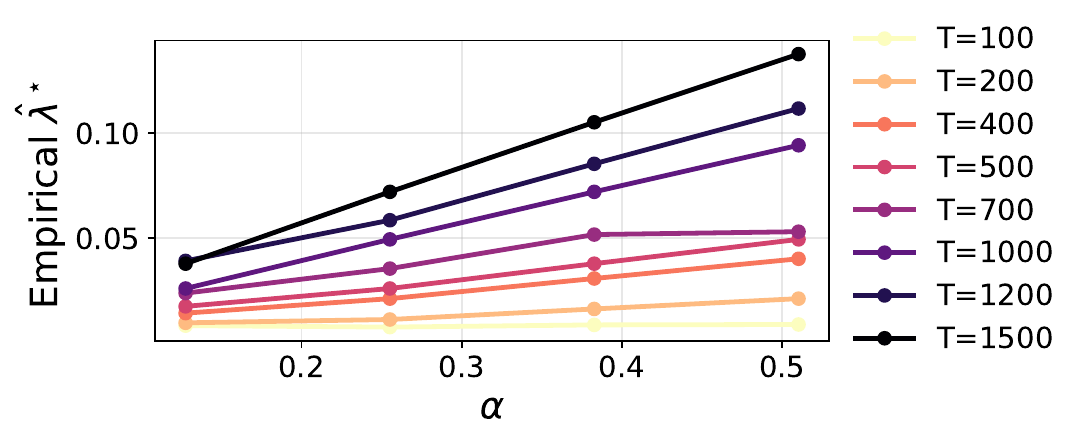}
\vspace{-0.5em}
\caption{
\textbf{Effect of aspect ratio on optimal regularization strength.}
We plot the empirical optimal regularizer after learning a noisy sequence with $T$ tasks.
Each curve reports the average empirical and theoretical values over 5 random seeds.
}
\label{fig:B}
\end{figure}

\newpage

\section{Proving the Main Result (\cref{eq:main_result})}

\label{app:main_proofs}

\begin{recall}[\cref{eq:main_result}]
Recall the high-dimensional statistical setting of \cref{sec:statistical_setting}, where the number of samples and the dimension $n,d\to \infty$ and $\frac{n}{d}\to \ratio\le 1$.
Moreover, recall that $\featnoise,\labelnoise$ are the variances of features and labels, respectively.
Then, the expected generalization loss (\cref{def:generlization:error}) 
after learning $T$ tasks with \cref{schm:regularized} is,
\[
\begin{aligned}
\mathbb{E}[G_T] 
  &
  = \underbrace{\labelnoise c\,\frac{1-a^{T}}{1-a}}_{\text{Label noise term}}
     \;+\underbrace{\;\frac{1}{T}\sum_{i=1}^{T}
     \left\Vert \vw_{T}^{\star}-\vw_{i}^{\star}\right\Vert^{2}}_{\text{Teacher variability term}}
  + \underbrace{2\big(\vw_{T}^{\star}-\bar{\vw}^{\star}\big)^{\!\top}
     \sum_{i=1}^{T} b^{\,T-i+1}
     \big(\vw_{i-1}^{\star}-\vw_{i}^{\star}\big)}_{\text{Interaction term}}\\[4pt]
  &
  \quad + \underbrace{\sum_{i=1}^{T}\sum_{i'=1}^{T}
     a^{\,T-\max(i,i')+1}\,b^{|i-i'|}
     \big(\vw_{i-1}^{\star}-\vw_{i}^{\star}\big)^{\!\top}
     \big(\vw_{i'-1}^{\star}-\vw_{i'}^{\star}\big)
     }_{\text{Accumulated temporal (teacher) correlations}}
     \,,
\end{aligned}
\]
where we define $\teacher_0 \triangleq \vw_0$ (by abuse of notation), the average teacher $\bar{\vw}^{\star} \triangleq \frac{1}{T}\sum_{i=1}^T \teacher_i$, and
\[
\begin{aligned}
  a &\triangleq 
  \tfrac{1}{2}
  \biggprn
  {
  1-\ratio + \frac{
  \lambda(1+\ratio)+\featnoise(1-\ratio)^{2}}
  {\sqrt{\lambda^2 +2\lambda\featnoise\prn{1+\ratio}+\featnoise^2 \prn{1-\ratio}^2}}
    }\,,
\\[4pt]
  b &\triangleq 
  \tfrac{1}{2}
  \biggprn
  {
          1-\ratio - \tfrac{\lambda}{\featnoise}
          +
          \tfrac{1}{\featnoise}
          \sqrt{\lambda^{2}
                 + 
                 2\lambda\featnoise(1+\ratio)
                 + 
                 \featnoise^{2} (1-\ratio)^{2}}
    },
    \\[4pt]
    c &\triangleq
    \tfrac{1}{2\featnoise}
  \biggprn
  {
        \frac{\featnoise(1+\ratio)+\lambda}
          {\sqrt{\lambda^2
                 + 2\lambda\featnoise(1+\ratio)
                 + \featnoise^{2}(1-\ratio)^{2}}}
          - 1
  }
  \,.\\[8pt]
\end{aligned}
\]
\end{recall}
We begin by proving the following Lemma, which provides a closed-form expression for the predictor update at each iteration $t$.

\begin{lemma}
\label{lem:solution}
At each iteration $t$, the iterative update under~\cref{schm:regularized} is given by 
$$\vw_{t}=\left(\X_{t}^{\top}\X_{t}+\lambda d\I\right)^{-1}\left(\X_{t}^{\top}\mathbf{y}_{t}+\lambda d\vw_{t-1}\right)~.$$
\end{lemma}
\begin{proof}
    
Define the convex objective
$$
J_t(\vw)\;=\left\Vert\X_t\vw-\y_t\right\Vert^2+\lambda d\left\Vert\vw-\vw_{t-1}\right\Vert^2
.
$$
Its gradient is
\begin{align*}
\nabla_\vw J_t(\vw)
&= 2\,\X_t^\top\left(\X_t\vw-\y_t\right)+2\lambda d\left(\vw-\vw_{t-1}\right)
\\&= 2\left[\left(\X_t^\top\X_t+\lambda d\I\right)\vw-\left(\X_t^\top\y_t+\lambda d\vw_{t-1}\right)\right].
\end{align*}

At an optimizer $\vw_t$ we must have $\nabla_\vw J_t(\vw_t)=\mathbf 0$, hence the normal equations
$$
(\X_t^\top\X_t+\lambda d\mathbf I)\,\vw_t=\X_t^\top\y_t+\lambda d\vw_{t-1}.
$$
If $\lambda>0$, then $\X_t^\top\X_t+\lambda d\I$ is symmetric positive definite and invertible, so
$$
\vw_t
= \left(\X_t^\top\X_t+\lambda d\mathbf I\right)^{-1}\left(\X_t^\top\y_t+\lambda d\vw_{t-1}\right).
$$
\end{proof}

\begin{definition}[Matrix Equivalent {\citep[][Notation 1]{liao2023rmtml}}]
\label{def:1}
For $\mathbf X, \mathbf Y\in \mathbb R^{n\times n}$ (deterministic or random matrices), we say that $\mathbf X\leftrightarrow \mathbf Y$ if, for all $\mathbf A\in \mathbb R^{n\times n}$ and $\mathbf a,\mathbf b \in \mathbb R^n$ of unit norms, we have 
$$
\frac 1 n\trace \left(\mathbf A\left(\mathbf X-\mathbf Y\right)\right)\to 0,
\quad
\mathbf a^\top \left(\mathbf X-\mathbf Y\right)\mathbf b\to 0, 
\quad
\left\Vert \mathbb E\left[\mathbf X-\mathbf Y\right]\right\Vert \to 0
$$
where, for random variables, the convergence is almost surely.
\end{definition}

\begin{proof}[for \cref{eq:main_result}]
We now turn to prove the theorem.
First, we aim to derive a formula independent of the predictors at each iteration.

According to~\cref{schm:regularized}, at each iteration $t$ the predictor is updated as $$\mathbf w_t \;\gets\;
        \argmin_{\vw}\Bigl\{\|\mathbf X_{t}\mathbf w-\mathbf y_{t}\|^2
        +\lambda d\|\mathbf w-\mathbf w_{t-1}\|^2\Bigr\}$$
From~\cref{lem:solution}, the solution to this problem is given by:
$$\vw_{t}=\left(\X_{t}^{\top}\X_{t}+\lambda d\I\right)^{-1}\left(\X_{t}^{\top}\mathbf{y}_{t}+\lambda d\vw_{t-1}\right)$$
For some $1\leq k, i \leq T$ we have
\begin{align}
\vw_{k}-\vw_{i}^{\star}&=\left(\X_{k}^{\top}\X_{k}+\lambda d\I\right)^{-1}\left(\X_{k}^{\top}\left(\X_{k}\vw_{k}^{\star}+\z_{k}\right)+\lambda d\vw_{k-1}\right)-\vw_{i}^{\star}\notag\\&=\left(\X_{k}^{\top}\X_{k}+\lambda d\I\right)^{-1}\left(\X_{k}^{\top}\left(\X_{k}\vw_{k}^{\star}+\z_{k}\right)+\lambda d\vw_{k-1}\right)\notag\\&\qquad\qquad-\left(\X_{k}^{\top}\X_{k}+\lambda d\I\right)^{-1}\left(\X_{k}^{\top}\X_{k}+\lambda d\I\right)\vw_{i}^{\star}\notag\\&=\left(\X_{k}^{\top}\X_{k}+\lambda d\I\right)^{-1}\left[\X_{k}^{\top}\left(\X_{k}\vw_{k}^{\star}+\z_{k}\right)+\lambda d\vw_{k-1}-\left(\X_{k}^{\top}\X_{k}+\lambda d\I\right)\vw_{i}^{\star}\right]\notag\\&=\left(\X_{k}^{\top}\X_{k}+\lambda d\I\right)^{-1}\left[\X_{k}^{\top}\z_{k}+\lambda d\vw_{k-1}-\lambda d\vw_{i}^{\star}+\X_{k}^{\top}\X_{k}\left(\vw_{k}^{\star}-\vw_{i}^{\star}\right)\right]\notag\\&=\left(\X_{k}^{\top}\X_{k}+\lambda d\I\right)^{-1}\left[\X_{k}^{\top}\z_{k}+\lambda d\left(\vw_{k-1}-\vw_{i}^{\star}\right)+\X_{k}^{\top}\X_{k}\left(\vw_{k}^{\star}-\vw_{i}^{\star}\right)\right].
\label{eq:start}
\end{align}

Denote $$\mP_{t}=\lambda d\left(\mathbf{X}_{t}^{\top}\mathbf{X}_{t}+\lambda d\I\right)^{-1}.$$

Next, we derive the deterministic equivalent for $\mathbf P_t$; see \cref{def:1}.
Let $\mathbf Z \in \mathbb{R}^{n \times d}$ with entries $Z_{ij}$ that are i.i.d.~with mean 0 and unit variance, and having a finite moment of order $4+\epsilon$ for some $\epsilon>0$. 
Define the resolvent
\[
Q(z) = \Bigl(\tfrac{1}{n}\mathbf Z^\top \mathbf Z - z \mathbf I_d \Bigr)^{-1}. 
\]
Then, by Theorem~2.6 from~\cite{liao2023rmtml}: as $n,d \to \infty$ with $d/n \to \frac 1\ratio \in (0,\infty)$, $Q(z)$ admits a deterministic equivalent
\begin{align}
Q(z) \leftrightarrow
\bar{Q}(z) \;=\; m_{\tfrac{1}{\ratio}}(z)\,\mathbf I_d,
\label{eq:March}
\end{align}
where $m_{\frac 1 \ratio}(z)$
satisfies the following closed-form expression:
\begin{equation}
\label{eq:mp_stieltjes}
m_{\frac 1 \ratio}(z)
= \frac{ -\left(1-\tfrac{1}{\ratio}-z\right) 
       + \sqrt{\bigl(1-\tfrac{1}{\ratio}-z\bigr)^{2} - \tfrac{4}{\ratio}z} }
       { -\tfrac{2}{\ratio}z }.
\end{equation}
Since $X_{i,j}$ has mean $0$ and variance $\featnoise$, we have
$
\mathbf X \;\overset{d}{=}\; \sqrt \featnoise\,\mathbf Z
$.
We aim to derive closed-form expressions for $\mathbb E[\mathbf P]$, $\mathbb E[\mathbf P^2]$ and $\mathbb{E}\left[\frac{1}{(\lambda d)^{2}}\mathbf{P}\mathbf{X}^{\top}\mathbf{X}\mathbf{P}\right]$, which will be used in our subsequent derivations. In our setting (\cref{sec:statistical_setting}), $\mathbf P$ denotes the random resolvent
$\mathbf P \triangleq \lambda d(\mathbf X^\top\mathbf X+\lambda d\,\mathbf I)^{-1}$.
Expectations are taken with respect to the limiting spectral distribution of
$\frac{1}{n}\mathbf X^\top\mathbf X$ (Marchenko--Pastur law, \cref{eq:March}),
under which $\mathbb E[\mathbf P]$ and $\mathbb E[\mathbf P^2]$ exist and are scalar multiples of the identity, as we show below.

Using the explicit form of $m_{\frac 1\ratio}\left(z\right)$ in \cref{eq:mp_stieltjes} we obtain, 
\begin{align}
\mathbb{E\left[\mathbf{P}\right]}&=\mathbb{E}\left[\lambda d\left(\featnoise\mathbf{Z}^{\top}\mathbf{Z}-\left(-\lambda d\right)\I\right)^{-1}\right]=\frac{\lambda d}{n\featnoise}\mathbb{E}\left[\left(\frac{1}{n}\mathbf{Z}^{\top}\mathbf{Z}-\left(-\frac{\lambda d}{n\featnoise}\right)\I\right)^{-1}\right]\notag\\&=\frac{\lambda d}{n\featnoise}m_{\frac{1}{\ratio}}\left(-\frac{\lambda d}{n\featnoise}\right)\I\notag\\&=\frac{-\left(1-\frac{1}{\ratio}+\frac{\lambda d}{n\featnoise}\right)+\sqrt{\left(1-\frac{1}{\ratio}+\frac{\lambda d}{n\featnoise}\right)^{2}+\frac{4}{\ratio}\frac{\lambda d}{n\featnoise}}}{\frac{2}{\ratio}}\I\notag\\&=\frac{1}{2}\left(1-\ratio-\frac{\lambda}{\featnoise}+\sqrt{\left(\frac{\lambda}{\featnoise}\right)^{2}+\frac{2\lambda}{\featnoise}\left(\ratio+1\right)+\left(1-\ratio\right)^{2}}\right)\I
\label{eq:E[P]}
\end{align}
and,
\begin{align}
\mathbb{E}\left[\mathbf{P}^{2}\right]&=\mathbb{E}\left[\left(\lambda d\right)^{2}\left(\lambda d\I+\featnoise\mathbf{Z}^{\top}\mathbf{Z}\right)^{-2}\right]=\mathbb{E}\left[-\frac{\left(\lambda d\right)^{2}}{d}\frac{\dd}{\dd\lambda}\left(-\left(-\lambda d\right)\I+\featnoise\mathbf{Z}^{\top}\mathbf{Z}\right)^{-1}\right]\notag&\\&=-\frac{\lambda^{2}d}{n\featnoise}\mathbb{E}\left[\frac{\dd}{\dd\lambda}\left(-\left(-\frac{\lambda d}{n\featnoise}\right)\I+\frac{1}{n}\mathbf{Z}^{\top}\mathbf{Z}\right)^{-1}\right]
=\left(\frac{\lambda d}{n\featnoise}\right)^{2}m'_{\frac{1}{\ratio}}\left(-\frac{\lambda d}{n\featnoise}\right)\I\notag\\&=\frac{\ratio}{2}\,\frac{\left(1+\frac{1}{\ratio}\right)\frac{\lambda d}{n\featnoise}+\left(1-\frac{1}{\ratio}\right)\!\left(\left(1-\frac{1}{\ratio}\right)-\sqrt{\left(\frac{\lambda d}{n\featnoise}\right)^{2}+2\!\left(1+\frac{1}{\ratio}\right)\!\frac{\lambda d}{n\featnoise}+\left(1-\frac{1}{\ratio}\right)^{2}}\right)}{\sqrt{\left(\frac{\lambda d}{n\featnoise}\right)^{2}+2\!\left(1+\frac{1}{\ratio}\right)\!\frac{\lambda d}{n\featnoise}+\left(1-\frac{1}{\ratio}\right)^{2}}}\,\I\notag\\&=\frac{1}{2}\left(1-\ratio+\frac{\left(\ratio+1\right)\frac{\lambda}{\featnoise}+\left(\ratio-1\right)^{2}}{\sqrt{\left(\frac{\lambda}{\featnoise}\right)^{2}+\frac{2\lambda}{\featnoise}\left(\ratio+1\right)+\left(\ratio-1\right)^{2}}}\right)\I\,.
\label{eq:E[P^2]}
\end{align}

The term $\mathbb{E}\left[\frac{1}{(\lambda d)^{2}}\mathbf{P}\mathbf{X}^{\top}\mathbf{X}\mathbf{P}\right]$
is a linear combination of $\mathbb E[\mathbf P]$ and $\mathbb E[\mathbf P^2]$,
\begin{align*}
\mathbb{E}\left[\frac{1}{\left(\lambda d\right)^{2}}\mathbf{P}\mathbf{X}^{\top}\mathbf{X}\mathbf{P}\right]
&=\mathbb{E}\left[\left(\lambda d\I+\mathbf{X}^{\top}\mathbf{X}\right)^{-1}\mathbf{\mathbf{X}^{\top}X}\left(\lambda d\I+\mathbf{X}^{\top}\mathbf{X}\right)^{-1}\right]
\\
&=\mathbb{E}\left[\left(\lambda d\I+\mathbf{X}^{\top}\mathbf{X}\right)^{-1}\left(\left(\mathbf{\lambda d\I+\mathbf{X}^{\top}X}\right)-\lambda d\I\right)\left(\lambda d\I+\mathbf{X}^{\top}\mathbf{X}\right)^{-1}\right]
\\
&=\mathbb{E}\left[\left(\lambda d\I+\mathbf{X}^{\top}\mathbf{X}\right)^{-1}-\lambda d\left(\lambda d\I+\mathbf{X}^{\top}\mathbf{X}\right)^{-2}\right]
\\
&=\frac{1}{\lambda d}\left[\mathbb{E}\left[\mathbf{P}\right]-\mathbb{E}\left[\mathbf{P}^{2}\right]\right]\I\\&=\frac{1}{2d\featnoise}\left(\frac{\featnoise\left(\ratio+1\right)+\lambda}{\sqrt{\lambda^{2}+2\lambda\featnoise\left(\ratio+1\right)+\featnoise^{2}\left(1-\ratio\right)^{2}}}-1\right)\I.
\end{align*}

We continue with our main result. From~\cref{eq:start} we obtain,
\begin{align}
&\vw_{k}-\vw_{i}^{\star}\notag
\\
&=\lambda d\left(\mathbf{X}_{t}^{\top}\mathbf{X}_{t}+\lambda d\I\right)^{-1}\!\left(\vw_{k-1}-\vw_{i}^{\star}\right)+
\left(\mathbf{X}_{t}^{\top}\mathbf{X}_{t}+\lambda d\I\right)^{-1}\!\left(\X_{k}^{\top}\z_{k}+\X_{k}^{\top}\X_{k}\left(\vw_{k}^{\star}-\vw_{i}^{\star}\right)\right)\notag\\&=\mP_{k}\left(\vw_{k-1}-\vw_{i}^{\star}\right)+\frac{1}{\lambda d}\mP_{k}\left(\X_{k}^{\top}\z_{k}+\X_{k}^{\top}\X_{k}\left(\vw_{k}^{\star}-\vw_{i}^{\star}\right)\right).
\label{eq:12}
\end{align}
Setting $i=k$ in~\cref{eq:12} we then obtain,
\begin{equation}
\vw_{k}-\vw_{k}^{\star}=\mP_{k}\left(\vw_{k-1}-\vw_{k}^{\star}\right)+\frac{1}{\lambda d}\mP_{k}\X_{k}^{\top}\z_{k}.
\label{eq:13}
\end{equation}

For all $1\leq i,t\leq T$ using~\cref{eq:13} we have, \begin{align}\vw_{t}-\vw_{i}^{\star}&=\mP_{t}\left(\vw_{t-1}-\vw_{t}^{\star}\right)+\frac{1}{\lambda d}\mP_{t}\X_{t}^{\top}\z_{t}+\left(\vw_{t}^{\star}-\vw_{i}^{\star}\right).
\label{eq:14}
\end{align}

From~\cref{eq:14},
\begin{align}
\vw_{t}-\vw_{i}^{\star}&=\mP_{t}\left(\mP_{t-1}\left(\vw_{t-2}-\vw_{t-1}^{\star}\right)+\frac{1}{\lambda d}\mP_{t-1}\X_{t-1}^{\top}\z_{t-1}+\left(\vw_{t-1}^{\star}-\vw_{t}^{\star}\right)\right)\notag\\&+\frac{1}{\lambda d}\mP_{t}\X_{t}^{\top}\z_{t}+\left(\vw_{t}^{\star}-\vw_{i}^{\star}\right)\notag\\&\ldots\notag\\\vw_{t}-\vw_{i}^{\star}&=\left[\prod_{m=t}^{1}\mP_{m}\right]\left(\vw_{0}-\vw_{1}^{\star}\right)+\frac{1}{\lambda d}\sum_{k=1}^{t}\left[\prod_{m=t}^{k}\mP_{m}\right]\X_{k}^{\top}\z_{k}+\notag\\&\sum_{k=2}^{t}\left[\prod_{m=t}^{k}\mP_{m}\right]\left(\vw_{k-1}^{\star}-\vw_{k}^{\star}\right)+\left(\vw_{t}^{\star}-\vw_{i}^{\star}\right).
\label{eq:seq}
\end{align}

\pagebreak

We proceed to the calculation of the main result $\mathbb{E}\left[G_{T}\right]=\frac{1}{T}\sum_{i=1}^{T}\mathbb{E}\left[\left\Vert \vw_{t}-\vw_{i}^{\star}\right\Vert ^{2}\right]$.
For convenience, we define $\mathbf{S}_{i:j} \triangleq \prod_{m=i}^{j} \mP_m=\mP_i\ldots \mP_j$ for $i\ge j$ and $\mathbf{S}_{i:j}=\I$ for $i<j$. We use the standard notation: for $a,b\in\mathbb{R}$ we write
$
\max(a,b) \;=\; a \vee b$ and $\min(a,b) \;=\; a \wedge b$.

\begin{align}
&\mathbb{E}\left[\left\Vert \vw_{t}-\vw_{i}^{\star}\right\Vert ^{2}\right]
\notag
\\
&
=\mathbb{E}\,\biggnorm{\mathbf S_{t:1}\left(\vw_{0}-\vw_{1}^{\star}\right)+\frac{1}{\lambda d}\sum_{k=1}^{t}\mathbf S_{t:k}\X_{k}^{\top}\z_{k}+\sum_{k=2}^{t}\mathbf S_{t:k}\left(\vw_{k-1}^{\star}-\vw_{k}^{\star}\right)+\left(\vw_{t}^{\star}-\vw_{i}^{\star}\right)
}^2
\notag
\\
&=
\underbrace{
\mathbb{E}\,
\biggnorm{\mathbf S_{t:1}\left(\vw_{0}-\vw_{1}^{\star}\right)+\frac{1}{\lambda d}\sum_{k=1}^{t}\mathbf S_{t:k}\X_{k}^{\top}\z_{k}
}^2
}_{\text{term 1}}
\notag
\\
&\quad+\underbrace{2\mathbb{E}\left(\sum_{k=2}^{t}\mathbf S_{t:k}\left(\vw_{k-1}^{\star}-\vw_{k}^{\star}\right)+\left(\vw_{t}^{\star}-\vw_{i}^{\star}\right)\right)^{\top}\left(\mathbf S_{t:1}\left(\vw_{0}-\vw_{1}^{\star}\right)+\frac{1}{\lambda d}\sum_{k=1}^{t}\mathbf S_{t:k}\X_{k}^{\top}\z_{k}\right)}_{\text{term 2}}
\notag
\\
&\quad+\underbrace{
\mathbb{E}\,
\biggnorm{\sum_{k=2}^{t}\mathbf S_{t:k}\left(\vw_{k-1}^{\star}-\vw_{k}^{\star}\right)+\left(\vw_{t}^{\star}-\vw_{i}^{\star}\right)
}^2
}_{\text{term 3}}\,.
\label{eq:difference_3_terms}
\end{align}

\newpage
In the following pages, we derive the terms above, employing the following properties.
\begin{enumerate}[label={[\arabic*]}]
    \item  For any \(1\leq k\leq t \leq T\),
\begin{align*}
&\mathbb{E}\left[\prod_{m=k}^{t}\mP_{m}\prod_{m=t}^{k}\mP_{m}\right]=\mathbb{E}\left[\mP_{k}\cdots\mP_{t-1}\mP_{t}\mP_{t}\mP_{t-1}\cdots\mP_{k}\right]&
\\
&=\mathbb{E}_{\mP_{k},\dots,\mP_{t-1}}\left[\mathbb{E}_{\mP_{t}}\left[\mP_{k}\cdots\mP_{t-1}\mP_{t}\mP_{t}\mP_{t-1}\cdots\mP_{k}\Bigl|\mP_{k},\mP_{k+1},\ldots,\mP_{t-1}\right]\right]
\\
\explain{3}&=\mathbb{E}\left[\mP_{k}\cdots\mP_{t-1}\mathbb{E}\left[\mP_{t}^{2}\right]\mP_{t-1}\cdots\mP_{k}\right]
=\mathbb{E}\left[\mP_{k}\cdots\mP_{t-1}\mathbb{E}\left[\mP^{2}\right]\mP_{t-1}\cdots\mP_{k}\right]
\\
\explain{4}&=\mathbb{E}\left[\mP^{2}\right]\mathbb{E}\left[\mP_{k}\cdots\mP_{t-1}\mP_{t-1}\cdots\mP_{k}\right]=\ldots=\left(\mathbb{E}\left[\mP^{2}\right]\right)^{t-k+1}.
\end{align*}

\item For any \(1\leq k'<k\leq t \leq T\),
\begin{align*}
&\mathbb{E}\left[\prod_{m=k}^{t}\mP_{m}\prod_{m=t}^{k}\mP_{m}\prod_{m=k-1}^{k'}\mP_{m}\right]
=
\mathbb{E}\left[\mP_{k}\cdots\mP_{t}\mP_{t}\cdots\mP_{k}\mP_{k-1}\cdots\mP_{k'}\right]&
\\
&=\mathbb{E}_{\mP_{k'},\ldots,\mP_{t-1}}\left[\mathbb{E}_{\mP_{t}}\left[\mP_{k}\cdots\mP_{t}\mP_{t}\cdots\mP_{k}\mP_{k-1}\cdots\mP_{k'}\Bigl|\mP_{k'},\ldots,\mP_{t-1}\right]\right]
\\
\explain{3}&=\mathbb{E}_{\mP_{k'},\ldots,\mP_{t-1}}\left[\mathbb{E}_{\mP_{t}}\left[\mP_{k'}\cdots\mP_{t-1}\mathbb{E}\left[\mP_{t}^{2}\right]\mP_{t-1}\cdots\mP_{k}\cdots\mP_{k'}\Bigl|\mP_{k'},\ldots,\mP_{t-1}\right]\right]
\\
\explain{4}
&=\mathbb{E}\left[\mP^{2}\right]\mathbb{E}\left[\mP_{k'}\cdots\mP_{t-1}\mP_{t-1}\cdots\mP_{k}\ldots\mP_{k'}\right]
=
\dots
=
\mathbb{E}\left[\mP_{k'}\cdots\mP_{k-1}\right]\mathbb{E}\left[\mP^{2}\right]^{t-k+1}
\\
&=\mathbb{E}_{\mP_{k'},\dots,\mP_{k-2}}\left[\mathbb{E}_{\mP_{k-1}}\left[\mP_{k'}\cdots\mP_{k-1}\Bigl|\mP_{k'},\dots,\mP_{k-2}\right]\right]\mathbb{E}\left[\mP^{2}\right]^{t-k+1}
\\
\explain{3}&=\mathbb{E}\left[\mathbb{E}\left[\mP_{k'}\cdots\mP_{k-2}\mathbb{E}\left[\mP_{k-1}\right]\right]\right]\mathbb{E}\left[\mP^{2}\right]^{t-k+1}
\\
\explain{4}&=\mathbb{E}\left[\mP\right]\mathbb{E}\left[\mP_{k'}\cdots\mP_{k-2}\right]\mathbb{E}\left[\mP^{2}\right]^{t-k+1}
=\ldots=\mathbb{E}\left[\mP\right]^{k-k'}\mathbb{E}\left[\mP^{2}\right]^{t-k+1}.
\end{align*}
Note that when $k=k'$, it reduces to case [1].

\item The data matrices are assumed to be independent of the past (see \cref{sec:statistical_setting}).

\item 
\textbf{$\mathbb{E}[\mathbf P], \mathbb{E}[\mathbf P^2]$ are both multiples of $\I$, and thus commutative.}
By~\cref{eq:E[P],eq:E[P^2]}, which give closed-form expressions for $\mathbb{E}[\mathbf P]$ and $\mathbb{E}[\mathbf P^2]$, i.e., both are scalar multiples of the identity matrix, then commutativity is automatically implied.
\end{enumerate} 

\newpage

Term 1:
\begin{align*}
&\mathbb{E}\,
\biggnorm{\mathbf S_{t:1}\left(\vw_{0}-\vw_{1}^{\star}\right)+\frac{1}{\lambda d}\sum_{k=1}^{t}\mathbf S_{t:k}\X_{k}^{\top}\z_{k}
}^2
\\
&=\mathbb{E}\left((\vw_{0}-\vw_{1}^{\star})^{\top}\mathbf{S}_{t:1}^{\top}+\frac{1}{\lambda d}\sum_{k=1}^{t}\z_{k}^{\top}\X_{k}\mathbf{S}_{t:k}^{\top}\right)\left(\mathbf{S}_{t:1}(\vw_{0}-\vw_{1}^{\star})+\frac{1}{\lambda d}\sum_{k=1}^{t}\mathbf{S}_{t:k}\X_{k}^{\top}\z_{k}\right)
\\
\explain{\text{*}}&=\mathbb{E}\left(\left(\vw_{0}-\vw_{1}^{\star}\right)^{\top}\mathbf{S}_{t:1}^{\top}\mathbf{S}_{t:1}\left(\vw_{0}-\vw_{1}^{\star}\right)+\frac{1}{(\lambda d)^{2}}\sum_{k=1}^{t}\left(\z_{k}\right)^{\top}\X_{k}\mathbf{S}_{t:k}^{\top}\sum_{k'=1}^{t}\mathbf{S}_{t:k'}\X_{k'}^{\top}\z_{k'}\right)
\\
\explain{\text{**}}
&
=
\left(\mathbb{E}\left[\mP^{2}\right]\right)^{t}\left\Vert \vw_{0}-\vw_{1}^{\star}\right\Vert ^{2}+\labelnoise d\mathbb{E}\left[\frac{1}{(\lambda d)^{2}}\mP_{k}\left(\X_{k}\right)^{\top}\X_{k}\mP_{k}\right]\frac{1-\left(\mathbb{E}\left[\mP^{2}\right]\right)^{t}}{1-\mathbb{E}\left[\mP^{2}\right]}
\,,
\end{align*}
where [*] follows by \cref{sec:statistical_setting} as the noise variables are sampled independently across tasks with $\mathbb E\left[\mathbf z_t\right] =\0$,
and [**] follows as the left inner term (in the third line) is,
\begin{align*}
\left(\vw_{0}-\vw_{1}^{\star}\right)^{\top}\mathbb{E}\left[\mathbf{S}_{t:1}^{\top}\mathbf{S}_{t:1}\right]\left(\vw_{0}-\vw_{1}^{\star}\right)
&
\stackrel{\explain{1,4}}{=}
\left(\mathbb{E}\left[\mP^{2}\right]\right)^{t}\left\Vert \vw_{0}-\vw_{1}^{\star}\right\Vert ^{2}\,,
\end{align*}
and the right inner term is,
\begin{align*}
&
\frac{1}{\left(\lambda d\right)^{2}}\mathbb{E}\left[\sum_{k=1}^{t}\sum_{k'=1}^{t}\z_{k}^{\top}\X_{k}\mathbf{S}_{t:k}^{\top}\mathbf{S}_{t:k}\X_{k'}^{\top}\z_{k'}\right]=\frac{1}{\left(\lambda d\right)^{2}}\mathbb{E}\left[\sum_{k=1}^{t}\z_{k}^{\top}\X_{k}\mathbf{S}_{t:k}^{\top}\mathbf{S}_{t:k}\X_{k}^{\top}\z_{k}\right]
\\
&=\frac{1}{\left(\lambda d\right)^{2}}\mathbb{E}\left[\sum_{k=1}^{t}\trace\left[\z_{k}^{\top}\X_{k}\mathbf{S}_{t:k}^{\top}\mathbf{S}_{t:k}\X_{k}^{\top}\z_{k}\right]\right]
=\frac{\labelnoise}{\left(\lambda d\right)^{2}}\sum_{k=1}^{t}\mathbb{E}\trace\left[\X_{k}^{\top}\X_{k}\mathbf{S}_{t:k}^{\top}\mathbf{S}_{t:k}\right]
\\&=\frac{\labelnoise}{\left(\lambda d\right)^{2}}\sum_{k=1}^{t}\mathbb{E}\trace\left[\X_{k}^{\top}\X_{k}\left(\mathbf{S}_{t:k+1}\mP_{k}\right)^{\top}\left(\mathbf{S}_{t:k+1}\mP_{k}\right)\right]
\\
&=\frac{\labelnoise}{\left(\lambda d\right)^{2}}
\sum_{k=1}^{t}\mathbb{E}\trace\left[\mP_{k}\X_{k}^{\top}\X_{k}\mathbf{P}_k^{\top}\mathbf{S}_{t:k+1}^{\top}\mathbf{S}_{t:k+1}\right]
\\&=\labelnoise
\sum_{k=1}^{t}\trace\left[\mathbb{E}\left[\frac{1}{\left(\lambda d\right)^{2}}\mP_{k}\X_{k}^{\top}\X_{k}\mP_{k}\right]\mathbb{E}\left[\mathbf{S}_{t:k+1}^{\top}\mathbf{S}_{t:k+1}\right]\right]
\\\explain{1}
&=\labelnoise
\sum_{k=1}^{t}\trace\left[\mathbb{E}\left[\frac{1}{\left(\lambda d\right)^{2}}\mP_{k}\X_{k}^{\top}\X_{k}\mP_{k}\right]\left(\mathbb{E}\left[\mP^{2}\right]\right)^{t-k}\right]
\\
\explain{4}&=\labelnoise
\sum_{k=1}^{t}d\mathbb{E}\left[\frac{1}{\left(\lambda d\right)^{2}}\mP_{k}\X_{k}^{\top}\X_{k}\mP_{k}\right]\left(\mathbb{E}\left[\mP^{2}\right]\right)^{t-k}
\\&=\labelnoise d\mathbb{E}\left[\frac{1}{\left(\lambda d\right)^{2}}\mP_k\X^{\top}_k\X_k\mP_k\right]\sum_{k=1}^{t}\left(\mathbb{E}\left[\mP^{2}\right]\right)^{t-k}%
\\&=\labelnoise d\mathbb{E}\left[\frac{1}{\left(\lambda d\right)^{2}}\mP_k\X_k^{\top}\X_k\mP_k\right]\frac{1-\left(\mathbb{E}\left[\mP^{2}\right]\right)^{t}}{1-\mathbb{E}\left[\mP^{2}\right]}\,,
\end{align*}
where [1], [4] are explained above.

Term 2:
\begin{align*}
&2\mathbb{E}\left(\sum_{k=2}^{t}\mathbf{S}_{t:k}\left(\vw_{k-1}^{\star}-\vw_{k}^{\star}\right)+\left(\vw_{t}^{\star}-\vw_{i}^{\star}\right)\right)^{\top}\left(\mathbf{S}_{t:1}\left(\vw_{0}-\vw_{1}^{\star}\right)+\frac{1}{\lambda}\sum_{k=1}^{t}\mathbf{S}_{t:k}\X_{k}^{\top}\z_{k}\right)
\\
&=2\mathbb{E}\Bigg[\left(\sum_{k=2}^{t}\left(\vw_{k-1}^{\star}-\vw_{k}^{\star}\right)^{\top}\mathbf{S}_{t:k}^{\top}\right)\left(\mathbf{S}_{t:1}\left(\vw_{0}-\vw_{1}^{\star}\right)+\frac{1}{\lambda}\sum_{k=1}^{t}\mathbf{S}_{t:k}\X_{k}^{\top}\z_{k}\right)
\\
&\hspace{3em}+\left(\vw_{t}^{\star}-\vw_{i}^{\star}\right)^{\top}\left(\mathbf{S}_{t:1}\left(\vw_{0}-\vw_{1}^{\star}\right)+\frac{1}{\lambda}\sum_{k=1}^{t}\mathbf{S}_{t:k}\X_{k}^{\top}\z_{k}\right)\Bigg]
\\
&=2\mathbb{E}\biggl[\sum_{k=2}^{t}\left(\vw_{k-1}^{\star}-\vw_{k}^{\star}\right)^{\top}\mathbf{S}_{t:k}^{\top}\mathbf{S}_{t:k}\mathbf{S}_{k-1:1}\left(\vw_{0}-\vw_{1}^{\star}\right)
+\left(\vw_{t}^{\star}-\vw_{i}^{\star}\right)^{\top}\mathbf{S}_{t:1}\left(\vw_{0}-\vw_{1}^{\star}\right)\Bigg]
\\
\explain{2}&=2\bigg(\sum_{k=2}^{t}\left(\vw_{k-1}^{\star}-\vw_{k}^{\star}\right)^{\top}\mathbb{E}\left[\mP^{2}\right]^{t-k+1}\mathbb{E}\left[\mP\right]^{k-1}\left(\vw_{0}-\vw_{1}^{\star}\right)
\\&\hspace{4em}+\left(\vw_{t}^{\star}-\vw_{i}^{\star}\right)^{\top}\mathbb{E}\left[\mP\right]^{t}\left(\vw_{0}-\vw_{1}^{\star}\right)\bigg)\,.
\end{align*}
Term 3:
\begin{align*}
&
\mathbb{E}\,
\biggnorm{\sum_{k=2}^{t}\mathbf S_{t:k}\left(\vw_{k-1}^{\star}-\vw_{k}^{\star}\right)+\left(\vw_{t}^{\star}-\vw_{i}^{\star}\right)
}^2
\\
&=\mathbb{E}\Bigg[\left(\sum_{k=2}^{t}\left(\vw_{k-1}^{\star}-\vw_{k}^{\star}\right)^{\top}\mathbf{S}_{t:k}^{\top}\right)\left(\sum_{k=2}^{t}\mathbf{S}_{t:k}\left(\vw_{k-1}^{\star}-\vw_{k}^{\star}\right)\right)
\\
&\qquad+2\left(\vw_{t}^{\star}-\vw_{i}^{\star}\right)^{\top}\left(\sum_{k=2}^{t}\mathbb{E}\left[\mathbf{S}_{t:k}\left(\vw_{k-1}^{\star}-\vw_{k}^{\star}\right)\right]\right)+\left\Vert \vw_{t}^{\star}-\vw_{i}^{\star}\right\Vert ^{2}\Bigg]
\\
&=\sum_{k=2}^{t}\sum_{k'=2}^{t}\left(\vw_{k-1}^{\star}-\vw_{k}^{\star}\right)^{\top}\mathbb{E}\left[\mathbf{S}_{t:k}^{\top}\mathbf{S}_{t:k'}\right]\left(\vw_{k'-1}^{\star}-\vw_{k'}^{\star}\right)
\\
&\qquad+2\left(\vw_{t}^{\star}-\vw_{i}^{\star}\right)^{\top}\sum_{k=2}^{t}\mathbb{E}\left[\mathbf{S}_{t:k}\right]\left(\vw_{k-1}^{\star}-\vw_{k}^{\star}\right)+\left\Vert \vw_{t}^{\star}-\vw_{i}^{\star}\right\Vert ^{2}
\\
&=\sum_{k=2}^{t}\sum_{k'=2}^{t}\left(\vw_{k-1}^{\star}-\vw_{k}^{\star}\right)^{\top}\mathbb{E}\left[\mathbf{S}_{t:k\vee k'}^{\top}\mathbf{S}_{t:k\vee k'}{\mathbf{S}}_{\left(\left(k\vee k'\right)-1\right):k\wedge k'}\right]\left(\vw_{k'-1}^{\star}-\vw_{k'}^{\star}\right)
\\
&\qquad+2\left(\vw_{t}^{\star}-\vw_{i}^{\star}\right)^{\top}\sum_{k=2}^{t}\mathbb{E}\left[\mathbf{S}_{t:k}\right]\left(\vw_{k-1}^{\star}-\vw_{k}^{\star}\right)+\left\Vert \vw_{t}^{\star}-\vw_{i}^{\star}\right\Vert ^{2}
\\\explain{2}
&=\,\sum_{k=2}^{t}\sum_{k'=2}^{t}\left(\vw_{k-1}^{\star}-\vw_{k}^{\star}\right)^{\top}\mathbb{E}\left[\mP^{2}\right]^{t-\left(k\vee k'\right)+1}\mathbb{E}\left[\mP\right]^{|k-k'|}\left(\vw_{k'-1}^{\star}-\vw_{k'}^{\star}\right)
\end{align*}
\begin{align*}
&  \qquad+2\left(\vw_{t}^{\star}-\vw_{i}^{\star}\right)^{\top}\sum_{k=2}^{t}\mathbb{E}\left[\mP\right]^{t-k+1}\left(\vw_{k-1}^{\star}-\vw_{k}^{\star}\right)+\left\Vert \vw_{t}^{\star}-\vw_{i}^{\star}\right\Vert ^{2},
\end{align*}

where [2]\footnote{In general, for $k\neq k'$ one obtains a two-case identity:
if $k>k'$ then $\mathbb{E}[\mathbf S_{t:k}^\top \mathbf S_{t:k'}]
=\mathbb{E}[\mP^2]^{\,t-k+1}\,\mathbb{E}[\mP]^{\,k-k'}$, while if $k<k'$ then
$\mathbb{E}[\mathbf S_{t:k}^\top \mathbf S_{t:k'}]
=\mathbb{E}[\mP]^{\,k'-k}\,\mathbb{E}[\mP^2]^{\,t-k'+1}$, where the factor
$\mathbb{E}[\mP]^{|k-k'|}$ appears on the side of the longer product.
In our setting $\mathbb{E}[\mP]=b\mathbf I$ and $\mathbb{E}[\mP^2]=a\mathbf I$
(hence they commute), so both cases collapse to
$\mathbb{E}[\mP^2]^{\,t-(k\vee k')+1}\mathbb{E}[\mP]^{\,|k-k'|}$.} is explained above.

To finalize the result, we sum all the terms we obtained.
First, we define $$a=\left\Vert \mathbb{E}\left[\mP^{2}\right]\right\Vert_2 ,~b=\left\Vert \mathbb{E}\left[\mP\right]\right\Vert_2,~c=\frac{d}{(\lambda d)^{2}}\left\Vert \mathbb{E}\left[\mP\X^{\top}\X\mP\right]\right\Vert_2. $$
Then, we plug all the terms into \cref{eq:difference_3_terms}, and obtain,
\begin{align}
\begin{split}
&\mathbb{E}\left[\left\Vert \vw_{t}-\vw_{i}^{\star}\right\Vert ^{2}\right]=a^{t}\left\Vert \vw_{0}-\vw_{1}^{\star}\right\Vert ^{2}+\labelnoise c\frac{1-a^{t}}{1-a}\\&\quad+2\sum_{k=2}^{t}a^{t-k+1}b^{k-1}\left(\vw_{k-1}^{\star}-\vw_{k}^{\star}\right)^{\top}\left(\vw_{0}-\vw_{1}^{\star}\right)+2b^{t}\left(\vw_{t}^{\star}-\vw_{i}^{\star}\right)^{\top}\left(\vw_{0}-\vw_{1}^{\star}\right)\\&\quad+\sum_{k=2}^{t}\sum_{k'=2}^{t}a^{t-\max\left(k,k'\right)+1}b^{\left|k-k'\right|}\left(\vw_{k-1}^{\star}-\vw_{k}^{\star}\right)^{\top}\left(\vw_{k'-1}^{\star}-\vw_{k'}^{\star}\right)
\\&\quad+2\sum_{k=2}^{t}b^{t-k+1}\left(\vw_{t}^{\star}-\vw_{i}^{\star}\right)^{\top}\left(\vw_{k-1}^{\star}-\vw_{k}^{\star}\right)+\left\Vert \vw_{t}^{\star}-\vw_{i}^{\star}\right\Vert ^{2}.
\label{eq:mainresult1}
\end{split}
\end{align}
Finally, the theorem follows by substituting the above expression into
$$\mathbb{E}\left[G_T\right]=\frac{1}{T}\sum_{i=1}^{T}\mathbb{E}\left[\left\Vert\mathbf w_T-\mathbf w_i^\star\right\Vert^2\right]\,.$$

\end{proof}

\newpage

\subsection{Proof of~\cref{example:no_reg}}
\label{pr:sp}
\begin{recall}[\cref{example:no_reg}]
Indeed, in the high-dimensional regime, \cref{eq:main_result} generalizes Theorem~4.1 in \citet{lin2023theory}.
Specifically, after adjusting notation, their result becomes,
$$
\mathbb{E}[G_T]
    =
    \tfrac{\left(1-\ratio\right)^{T}}{T}
    \sum_{i=1}^{T}\left\Vert \vw_{i}^{\star}\right\Vert ^{2}
    +
    \tfrac{1}{T}\sum_{i=1}^{T}\left(1-\ratio\right)^{T-i}\ratio
    \sum_{k=1}^{T}\left\Vert \vw_{k}^{\star}-\vw_{i}^{\star}\right\Vert ^{2}
    +
    \tfrac{d\labelnoise}{d-n-1}
    \left(1-\left(1-\ratio\right)^{T}\right)
    ,
$$
aligning with \cref{eq:main_result} in the high-dimensional, unregularized case ($d,n \to \infty$, $\lambda\to 0$, $\featnoise=1$).
\end{recall} 

\medskip
\begin{proof}
Let $\X\in \mathbb R^{n\times d}$. From~\citet[][Theorem 4.3]{barata2012moore}, we know that\linebreak$B(\lambda)
\triangleq
(\mathbf X^\top \mathbf X+\lambda \mathbf I)^{-1}\mathbf X^\top
\to \X^{+}$ as $\lambda \to 0$, where $\X^{+}$ is the Moore–Penrose pseudoinverse.
\begin{align*}
&\lim_{\lambda\to0^+}a=\lim_{\lambda\to0^+}\frac{1}{2}\left(1-\ratio+\frac{\frac{\lambda}{\featnoise}\left(1+\ratio\right)+\left(1-\ratio\right)^{2}}{\sqrt{\left(\frac{\lambda}{\featnoise}\right)^{2}+2\frac{\lambda}{\featnoise}\left(1+\ratio\right)+\left(1-\ratio\right)^{2}}}\right)=1-\ratio
\\&
\lim_{\lambda\to0^+}b=\lim_{\lambda\to0^+}\frac{1}{2}\left(1-\ratio-\frac{\lambda}{\featnoise}+\sqrt{\left(\frac{\lambda}{\featnoise}\right)^{2}+2\frac{\lambda}{\featnoise}\left(1+\ratio\right)+\left(1-\ratio\right)^{2}}\right)=1-\ratio
\end{align*}
Substituting the values to our result in~\cref{eq:mainresult1}, with $\vw_0=0$,
\begin{align*}
&\mathbb{E}\left[\|\vw_{t}-\vw_{i}^{\star}\|^{2}\right]=\underbrace{\labelnoise\frac{1}{2\featnoise}\left(\frac{\featnoise\left(1+\ratio\right)+\lambda}{\sqrt{\lambda^{2}+2\lambda\featnoise\left(1+\ratio\right)+\featnoise^{2}\left(1-\ratio\right)^{2}}}-1\right)\frac{1-\left(1-\ratio\right)^{t}}{1-\left(1-\ratio\right)}}_{\text{term 1}}
\\
&\quad+\underbrace{\left(1-\ratio\right)^{t}\left\Vert \vw_{1}^{\star}\right\Vert ^{2}+2\left(1-\ratio\right)^{t}\sum_{k=2}^{t}\left(\vw_{k}^{\star}-\vw_{k-1}^{\star}\right)^{\top}\vw_{1}^{\star}+2\left(1-\ratio\right)^{t}\left(\vw_{i}^{\star}-\vw_{t}^{\star}\right)^{\top}\vw_{1}^{\star}}_{\text{term 2}}
\\
&\quad+\underbrace{\sum_{k=2}^{t}\sum_{k'=2}^{t}\left(1-\ratio\right)^{t-\min(k,k')+1}\left(\vw_{k-1}^{\star}-\vw_{k}^{\star}\right)^{\top}\left(\vw_{k'-1}^{\star}-\vw_{k'}^{\star}\right)}_{\text{term 3}}\\&\quad\underbrace{+2\sum_{k=2}^{t}\left(1-\ratio\right)^{t-k+1}\left(\vw_{t}^{\star}-\vw_{i}^{\star}\right)^{\top}\left(\vw_{k-1}^{\star}-\vw_{k}^{\star}\right)+\left\Vert \vw_{t}^{\star}-\vw_{i}^{\star}\right\Vert ^{2}}_{\text{term 3}}.
\end{align*}

First term is:
\begin{align*}
&\lim_{\lambda \to 0^+}\labelnoise\frac{1}{2\featnoise}\left(\frac{\featnoise\left(1+\ratio\right)+\lambda}{\sqrt{\lambda^{2}+2\lambda\featnoise\left(1+\ratio\right)+\featnoise^{2}\left(1-\ratio\right)^{2}}}-1\right)\frac{1-\left(1-\ratio\right)^{t}}{1-\left(1-\ratio\right)}
\\
&\stackrel{\explain{1}}{=}\labelnoise\frac{1}{2}\left(\frac{1+\ratio}{1-\ratio}-1\right)\frac{1-\left(1-\ratio\right)^{t}}{1-\left(1-\ratio\right)}
=\labelnoise\frac{\ratio}{1-\ratio}\frac{1-\left(1-\ratio\right)^{t}}{\ratio}
=\labelnoise\frac{1-\left(1-\ratio\right)^{t}}{1-\ratio}
\,,
\end{align*}
where step [1] follows from setting $\featnoise=1, \lambda\to 0$ according to the notation in \cite{lin2023theory}.
\pagebreak

Second term is:
\begin{align*}
&\left(1-\ratio\right)^{t}\left\Vert \vw_{1}^{\star}\right\Vert ^{2}+2\left(1-\ratio\right)^{t}\sum_{k=2}^{t}\left(\vw_{k}^{\star}-\vw_{k-1}^{\star}\right)^{\top}\vw_{1}^{\star}+2\left(1-\ratio\right)^{t}\left(\vw_{i}^{\star}-\vw_{t}^{\star}\right)^{\top}\vw_{1}^{\star}&
\\
&=\left(1-\ratio\right)^{t}\left\Vert \vw_{1}^{\star}\right\Vert ^{2}+2\left(1-\ratio\right)^{t}\Biggl[\left(\cancel{\vw_{t}^{\star}}-\vw_{1}^{\star}\right)^{\top}\vw_{1}^{\star}+\left(\left(\vw_{i}^{\star}\right)^{\top}\vw_{1}^{\star}-\cancel{\left(\vw_{t}^{\star}\right)^{\top}\vw_{1}^{\star}}\right)\Biggr]\\&=\left(1-\ratio\right)^{t}\left\Vert \vw_{1}^{\star}\right\Vert^{2}+2\left(1-\ratio\right)^{t}\left(-\left\Vert \vw_{1}^{\star}\right\Vert^{2}+\left(\vw_{i}^{\star}\right)^{\top}\vw_{1}^{\star}\right)
\\
&=-\left(1-\ratio\right)^{t}\left\Vert \vw_{1}^{\star}\right\Vert ^{2}+2\left(1-\ratio\right)^{t}\left(\vw_{i}^{\star}\right)^{\top}\vw_{1}^{\star}-\left(1-\ratio\right)^{t}\left\Vert \vw_{i}^{\star}\right\Vert ^{2}+\left(1-\ratio\right)^{t}\left\Vert \vw_{i}^{\star}\right\Vert ^{2}
\\
&=-\left(1-\ratio\right)^{t}\left\Vert \vw_{i}^{\star}-\vw_{1}^{\star}\right\Vert ^{2}+\left(1-\ratio\right)^{t}\left\Vert \vw_{i}^{\star}\right\Vert ^{2}
\,.
\end{align*}

Finally, the third term is:
\begin{align*}
&\sum_{k=2}^{t}\sum_{k'=2}^{t}\left(1-\ratio\right)^{t-\min\left(k,k'\right)+1}\left(\vw_{k-1}^{\star}-\vw_{k}^{\star}\right)^{\top}\left(\vw_{k'-1}^{\star}-\vw_{k'}^{\star}\right)&
\\
&\quad+2\sum_{k=2}^{t}\left(1-\ratio\right)^{t-k+1}\left(\vw_{t}^{\star}-\vw_{i}^{\star}\right)^{\top}\left(\vw_{k-1}^{\star}-\vw_{k}^{\star}\right)+\left\Vert \vw_{t}^{\star}-\vw_{i}^{\star}\right\Vert ^{2}
\\&=\sum_{s=2}^{t}\left(1-\ratio\right)^{t-s+1}\left[2\left(\vw_{s-1}^{\star}-\vw_{s}^{\star}\right)^{\top}\sum_{r=s}^{t}\left(\vw_{r-1}^{\star}-\vw_{r}^{\star}\right)-\left\Vert \vw_{s-1}^{\star}-\vw_{s}^{\star}\right\Vert ^{2}\right]
\\
&\quad+2\sum_{k=2}^{t}\left(1-\ratio\right)^{t-k+1}\left(\vw_{t}^{\star}-\vw_{i}^{\star}\right)^{\top}\left(\left(\vw_{k-1}^{\star}-\vw_{t}^{\star}\right)-\left(\vw_{k}^{\star}-\vw_{t}^{\star}\right)\right)+\left\Vert \vw_{t}^{\star}-\vw_{i}^{\star}\right\Vert ^{2}
\\
&=2\sum_{s=2}^{t}\left(1-\ratio\right)^{t-s+1}\Biggl[\left\Vert \vw_{s-1}^{\star}-\vw_{t}^{\star}\right\Vert ^{2}-\left\Vert \vw_{s-1}^{\star}-\vw_{t}^{\star}\right\Vert ^{2}+2\left(\vw_{s-1}^{\star}-\vw_{s}^{\star}\right)^{\top}\left(\vw_{s-1}^{\star}-\vw_{t}^{\star}\right)\\&\hspace{14em}-\left\Vert \vw_{s-1}^{\star}-\vw_{s}^{\star}\right\Vert ^{2}\Biggr]+\left\Vert \vw_{t}^{\star}-\vw_{i}^{\star}\right\Vert ^{2}
\\
&\quad+2\left(1-\ratio\right)^{t-1}\left(\vw_{t}^{\star}-\vw_{i}^{\star}\right)^{\top}\left(\vw_{1}^{\star}-\vw_{t}^{\star}\right)+
2\ratio\sum_{k=3}^{t}\left(1-\ratio\right)^{t-k+1}\left(\vw_{t}^{\star}-\vw_{i}^{\star}\right)^{\top}\left(\vw_{k-1}^{\star}-\vw_{t}^{\star}\right)
\\
&=\left(\sum_{s=2}^{t}\left(1-\ratio\right)^{t-s+1}\left[\left\Vert \vw_{s-1}^{\star}-\vw_{t}^{\star}\right\Vert ^{2}-\left\Vert \vw_{t}^{\star}-\vw_{s}^{\star}\right\Vert ^{2}\right]\right)+\left\Vert \vw_{t}^{\star}-\vw_{i}^{\star}\right\Vert ^{2}
\\
&\quad+2\left(1-\ratio\right)^{t-1}\left(\vw_{t}^{\star}-\vw_{i}^{\star}\right)^{\top}\left(\vw_{1}^{\star}-\vw_{t}^{\star}\right)+
2\ratio\sum_{k=3}^{t}\left(1-\ratio\right)^{t-k+1}\left(\vw_{t}^{\star}-\vw_{i}^{\star}\right)^{\top}\left(\vw_{k-1}^{\star}-\vw_{t}^{\star}\right)
\\
&=\left(1-\ratio\right)^{t-1}\left\Vert \vw_{1}^{\star}-\vw_{t}^{\star}\right\Vert ^{2}+\ratio\sum_{k=3}^{t}\left(1-\ratio\right)^{t-k+1}\left\Vert \vw_{k-1}^{\star}-\vw_{t}^{\star}\right\Vert ^{2}+\left\Vert \vw_{t}^{\star}-\vw_{i}^{\star}\right\Vert ^{2}
\\
&\quad+2\left(1-\ratio\right)^{t-1}\left(\vw_{t}^{\star}-\vw_{i}^{\star}\right)^{\top}\left(\vw_{1}^{\star}-\vw_{t}^{\star}\right)+
2\ratio\sum_{k=3}^{t}\left(1-\ratio\right)^{t-k+1}\left(\vw_{t}^{\star}-\vw_{i}^{\star}\right)^{\top}\left(\vw_{k-1}^{\star}-\vw_{t}^{\star}\right)
\end{align*}

\begin{align*}
&=\left(1-\ratio\right)^{t-1}\left(\left\Vert \vw_{1}^{\star}-\vw_{t}^{\star}\right\Vert ^{2}+2\left(\vw_{t}^{\star}-\vw_{i}^{\star}\right)^{\top}\left(\vw_{1}^{\star}-\vw_{t}^{\star}\right)\right)+\left\Vert \vw_{t}^{\star}-\vw_{i}^{\star}\right\Vert ^{2}
\\
&\quad+\ratio\sum_{k=3}^{t}\left(1-\ratio\right)^{t-k+1}\left\Vert \vw_{k-1}^{\star}-\vw_{i}^{\star}\right\Vert ^{2}-\ratio\sum_{k=3}^{t}\left(1-\ratio\right)^{t-k+1}\left\Vert \vw_{t}^{\star}-\vw_{i}^{\star}\right\Vert^{2}
\\&=\left(1-\ratio\right)^{t-1}\left(\left\Vert \vw_{1}^{\star}-\vw_{t}^{\star}\right\Vert ^{2}+2\left(\vw_{t}^{\star}-\vw_{i}^{\star}\right)^{\top}\left(\vw_{1}^{\star}-\vw_{t}^{\star}\right)\right)+\left\Vert \vw_{t}^{\star}-\vw_{i}^{\star}\right\Vert ^{2}
\\
&\quad+\ratio\sum_{k=2}^{t-1}\left(1-\ratio\right)^{t-k}\left\Vert \vw_{k}^{\star}-\vw_{i}^{\star}\right\Vert ^{2}-\left(1-\ratio\right)\left(1-\left(1-\ratio\right)^{t-2}\right)\left\Vert \vw_{t}^{\star}-\vw_{i}^{\star}\right\Vert ^{2}
\\
&=\left(1-\ratio\right)^{t-1}\left(\left\Vert \vw_{1}^{\star}-\vw_{t}^{\star}\right\Vert ^{2}+2\left(\vw_{t}^{\star}-\vw_{i}^{\star}\right)^{\top}\left(\vw_{1}^{\star}-\vw_{t}^{\star}\right)+\left\Vert \vw_{t}^{\star}-\vw_{i}^{\star}\right\Vert ^{2}\right)
\\
&\quad+\ratio\sum_{k=2}^{t}\left(1-\ratio\right)^{t-k}\left\Vert \vw_{k}^{\star}-\vw_{i}^{\star}\right\Vert ^{2}\\&=\left(1-\ratio\right)^{t}\left\Vert \vw_{i}^{\star}-\vw_{1}^{\star}\right\Vert ^{2}+\ratio\sum_{k=1}^{t}\left(1-\ratio\right)^{t-k}\left\Vert \vw_{k}^{\star}-\vw_{i}^{\star}\right\Vert ^{2}.
\end{align*}

Summing all the terms:
\begin{align*}
&\mathbb{E}\left[\left\Vert \vw_{t}-\vw_{i}^{\star}\right\Vert ^{2}\right]=\sum_{k=1}^{t}\left(1-\ratio\right)^{t-k} \ratio\left\Vert \vw_{k}^{\star}-\vw_{i}^{\star}\right\Vert ^{2}+\cancel{\left(1-\ratio\right)^{t}\left\Vert \vw_{i}^{\star}-\vw_{1}^{\star}\right\Vert ^{2}}
\\
&\hspace{8em}\cancel{-\left(1-\ratio\right)^{t}\left\Vert \vw_{i}^{\star}-\vw_{1}^{\star}\right\Vert ^{2}}+\left(1-\ratio\right)^{t}\left\Vert \vw_{i}^{\star}\right\Vert +\frac{\labelnoise}{1-\ratio}\left(1-\left(1-\ratio\right)^{t}\right)
\\
&=\sum_{k=1}^{t}\left(1-\ratio\right)^{t-k} \ratio\left\Vert \vw_{k}^{\star}-\vw_{i}^{\star}\right\Vert ^{2}
+\left(1-\ratio\right)^{t}\left\Vert \vw_{i}^{\star}\right\Vert ^{2}+\frac{\labelnoise}{1-\ratio}\left(1-\left(1-\ratio\right)^{t}\right)\,.
\end{align*}
Overall,
\begin{align*}
&\mathbb{E}\left[G_{T}\right]=
\frac{1}{T}\sum_{i=1}^{T}\mathbb{E}\left[\left\Vert \vw_{t}-\vw_{i}^{\star}\right\Vert ^{2}\right]
\\
&=\frac{1}{T}\sum_{i=1}^{T}\left(1-\ratio\right)^{T-i}\ratio\sum_{k=1}^{T}\left\Vert \vw_{k}^{\star}-\vw_{i}^{\star}\right\Vert ^{2}+\frac{1}{T}\left(1-\ratio\right)^{T}\sum_{i=1}^{T}\left\Vert \vw_{i}^{\star}\right\Vert ^{2}+\frac{\labelnoise}{1-\ratio}\left(1-\left(1-\ratio\right)^{T}\right)
\,\\&\overset{\explain{1}}{=}
\tfrac{\left(1-\ratio\right)^{T}}{T}
\sum_{i=1}^{T}\left\Vert \vw_{i}^{\star}\right\Vert ^{2}
+
\tfrac{1}{T}\sum_{i=1}^{T}\left(1-\ratio\right)^{T-i}\ratio
\sum_{k=1}^{T}\left\Vert \vw_{k}^{\star}-\vw_{i}^{\star}\right\Vert ^{2}
+
\tfrac{d\labelnoise}{d-n-1}
\left(1-\left(1-\ratio\right)^{T}\right)
\end{align*}
Where [1] follows by our Setting~\ref{sec:statistical_setting} where $n,d\to \infty$.
\end{proof}

\newpage

\section{Auxiliary Lemmas}
We begin with a few auxiliary Lemmas that will be used in the proofs.

First, we introduce a few auxiliary variables.
\begin{align*}
\tilde{\lambda}&\triangleq\frac{\lambda}{\featnoise},
\quad 
N\triangleq\left(1+\ratio\right)\tilde{\lambda}+\left(1-\ratio\right)^{2},
\\
D&\triangleq\sqrt{\tilde{\lambda}^{2}+2\tilde{\lambda}\left(1+\ratio\right)+\left(1-\ratio\right)^{2}}=\sqrt{\left(1+\ratio+\tilde{\lambda}\right)^{2}-4\ratio}.
\end{align*}

\begin{lemma}. For all $\lambda > 0$ we have $b>a$.
\label{lem:b>a}
\end{lemma}

\begin{proof}
Recalling the notations from \cref{tab:notation}, we have that,
$$
\begin{aligned}
b-a&=\frac{1}{2}\left(1-\ratio-\tilde{\lambda}+D-\left(1-\ratio+\frac{\tilde{\lambda}\left(1+\ratio\right)+\left(1-\ratio\right)^{2}}{D}\right)\right)
\\
&=\frac{-\tilde{\lambda}D+D^{2}-\tilde{\lambda}\left(1+\ratio\right)-\left(1-\ratio\right)^{2}}{2D}=\frac{-\tilde{\lambda}D+\tilde{\lambda}^{2}+\tilde{\lambda}\left(1+\ratio\right)}{2D}
\\&=\frac{\tilde{\lambda}\left(\left(\tilde{\lambda}+\ratio+1\right)-D\right)}{2D}=\frac{2\ratio\tilde \lambda}{D\left(\tilde{\lambda}+\ratio+1+D\right)}>0.
\end{aligned}
$$
Note that $\tilde \lambda >0 \iff \lambda>0$, so $b>a$ for all $\lambda>0$.

\end{proof}

\begin{lemma}
\label{lem:0<ab<1}
For all $ \lambda > 0$ we have $0<a<1, 0< b<1$.
\end{lemma}

\begin{proof}
Recalling the notations from \cref{tab:notation}, we have that,
\begin{align}
a &= \frac{1}{2}\left(1-\ratio+\frac{\tilde{\lambda}(1+\ratio)+(1-\ratio)^{2}}{D}\right)
   \overset{1\ge \ratio}{>} 0,
\end{align}

and
\begin{align*}
b=\frac{1}{2}\left(1-\ratio-\tilde{\lambda}+D\right)&
=\frac{1}{2}\left(1-\ratio-\tilde{\lambda}+\sqrt{\left(1+\ratio+\tilde{\lambda}\right)^{2}-4\ratio}\right)\\
\explain{\ratio>0}
&<\frac{1}{2}\left(1-\ratio-\tilde{\lambda}+\sqrt{\left(1+\ratio+\tilde{\lambda}\right)^{2}}\right)=1.
\end{align*}
Note that $\tilde \lambda >0 \iff \lambda>0$,
using \cref{lem:b>a} we have $0<a<1, 0<b<1$ for all $\lambda>0$.
\end{proof}

\pagebreak

\begin{lemma}
\label{lem:da:dl}
The derivative of $a$ with respect to $\lambda$ takes the compact form $\frac{\dd}{\dd\lambda}a=\frac{2\tilde{\lambda}\ratio}{D^{3}\featnoise}$ and is positive for all $\lambda>0$.
\end{lemma}
\begin{proof}
Recalling the notation from \cref{tab:notation}, we have by definition
\begin{align*}
&\frac{\dd}{\dd\lambda}a=
\frac{1}{2\featnoise}\left(\frac{1+\ratio}{D}-\frac{N}{D^{3}}\left(\tilde{\lambda}+\ratio+1\right)\right)=\frac{\left(1+\ratio\right)D^{2}-N\left(\tilde{\lambda}+\ratio+1\right)}{2\featnoise D^{3}},
\end{align*}
we focus on the numerator:
\begin{equation}  
\begin{aligned}\label{eq:23}
    &\left(1+\ratio\right)D^{2}-N\left(\tilde{\lambda}+\ratio+1\right)
    \\
    &=\left(1+\ratio\right)\left(\tilde{\lambda}^{2}+2\tilde{\lambda}\left(1+\ratio\right)+\left(1-\ratio\right)^{2}\right)-\left(\left(1+\ratio\right)\tilde{\lambda}+\left(1-\ratio\right)^{2}\right)\left(\tilde{\lambda}+\left(1+\ratio\right)\right)
    \\
    &=\cancel{\left(1+\ratio\right)\tilde{\lambda}^{2}}+2\tilde{\lambda}\left(1+\ratio\right)^{2}+\cancel{\left(1+\ratio\right)\left(1-\ratio\right)^{2}}
    \\
    &\quad-\left[\cancel{\left(1+\ratio\right)\tilde{\lambda}^{2}}+\left(1+\ratio\right)^{2}\tilde{\lambda}+\left(1-\ratio\right)^{2}\tilde{\lambda}+\cancel{\left(1+\ratio\right)\left(1-\ratio\right)^{2}}\right]
    \\
    &=\left[\left(1+\ratio\right)^{2}-\left(1-\ratio\right)^{2}\right]\tilde{\lambda}=4\tilde{\lambda}\ratio.
\end{aligned}
\end{equation}

Substituting back, we obtain $\frac{\dd}{\dd\lambda}a=\frac{2\tilde{\lambda}\ratio}{D^{3}\featnoise}$, which is always positive as each factor in the product in positive.
\end{proof}

\begin{lemma}\label{lem:db:dl}
The derivative of $b$ with respect to $\lambda$ takes the compact form $\frac{\dd}{\dd\lambda}b=\frac{1}{2\featnoise}\Bigl(\frac{-D+\tilde{\lambda}+\ratio+1}{D}\Bigr)$ and is positive for all $\lambda>0$.
\end{lemma}

\begin{proof}
Recalling the notation from \cref{tab:notation}, we have by definition
$$
\frac{\dd}{\dd\lambda}b=\frac{1}{2\featnoise}\Bigl(-1+\frac{\tilde{\lambda}+\ratio+1}{D}\Bigr),$$
since   
\begin{equation}\label{eq:24}
\left(\tilde{\lambda}+\ratio+1\right)^{2}=\tilde{\lambda}^{2}+2\tilde{\lambda}\left(1+\ratio\right)+\left(1+\ratio\right)^{2}>\tilde{\lambda}^{2}+2\tilde{\lambda}\left(1+\ratio\right)+\left(1-\ratio\right)^{2}=D^{2},
\end{equation}
then,
$$
\frac{\dd}{\dd\lambda}b=\frac{1}{2\featnoise}\underbrace{\Bigl(-1+\frac{\tilde{\lambda}+\ratio+1}{D}\Bigr)}_{>0}>0.$$

\end{proof}

\begin{lemma}
\label{lem:5}
$\tilde{\lambda}+\ratio+1+D>2$ for all $\tilde \lambda >0$.
\end{lemma}

\begin{proof}
Recalling the notations from \cref{tab:notation}, we have that,
\begin{align*}
\tilde{\lambda}+\ratio+1+D
=
\tilde{\lambda}+\ratio+1+
\sqrt{
\bigprn{1-\ratio+\tilde{\lambda}}^{2}
\!+4\tilde{\lambda}\ratio}
\geq
\tilde{\lambda}+\ratio+1+1-\ratio+\tilde{\lambda}
=2\tilde{\lambda} + 2
>2.
\end{align*}
\end{proof} 

\begin{lemma}
\label{lem:1}
$1-b\geq\frac{\ratio}{1+\ratio+\tilde{\lambda}}$ for all $\tilde \lambda >0$.
\end{lemma}

\begin{proof}
Recalling the notations from \cref{tab:notation}, we have that,
\begin{align*}
1-b&=\frac{1+\ratio+\tilde{\lambda}-\sqrt{\left(1+\ratio+\tilde{\lambda}\right)^{2}-4\ratio}}{2}=\frac{\cancel{\left(1+\ratio+\tilde{\lambda}\right)^{2}}-\cancel{\left(1+\ratio+\tilde{\lambda}\right)^{2}}+4\ratio}{2\left(1+\ratio+\tilde{\lambda}+\sqrt{\left(1+\ratio+\tilde{\lambda}\right)^{2}-4\ratio}\right)}
\\
&=\frac{2\ratio}{\left(1+\ratio+\tilde{\lambda}+\sqrt{\left(1+\ratio+\tilde{\lambda}\right)^{2}-4\ratio}\right)}\geq\frac{\ratio}{1+\ratio+\tilde{\lambda}}
\,.
\end{align*}
\end{proof}

\begin{lemma}
\label{lem:3}
$a,b\leq1-\frac{\ratio}{1+\ratio+\tilde{\lambda}}$ for all $\tilde \lambda >0$. (a sharper bound on $a, b$)
\end{lemma}

\begin{proof}
\cref{lem:1,lem:b>a} imply that $1-a\geq1-b\geq\frac{\ratio}{1+\ratio+\tilde{\lambda}}$ so $a,b\leq1-\frac{\ratio}{1+\ratio+\tilde{\lambda}}$.
\end{proof}

\begin{lemma}
\label{lem:c=1bound}
If $\lim_{x\to0^+} h(x)=a$ for some $a\in\mathbb{R}$ and $\lim_{x\to\infty} h(x)=b$ for some $b\in\mathbb{R}$, and $h$ is continuous for all $x>0$, then $h$ is bounded on $(0,\infty)$.
\end{lemma}

\begin{proof}
Since $\lim_{x\to0^+} h(x)=a$, there exists $\delta>0$ such that
\[
0<x<\delta \;\Rightarrow\; |h(x)-a|<1.
\]
Hence for $0<x<\delta$ we have $|h(x)|\le |a|+1\triangleq M_1$.

Similarly, since $\lim_{x\to\infty} h(x)=b$, there exists $R>0$ such that
\[
x>R \;\Rightarrow\; |h(x)-b|<1,
\]
so for $x>R$ we have $|h(x)|\le |b|+1=:M_3$.

On the closed, bounded interval $[\delta,R]$, the function $h$ is continuous; by Weierstrass, $h$ attains a maximum of its absolute value there. Thus there exists $M_2>0$ such that
\[
|h(x)|\le M_2 \quad \text{for all } x\in[\delta,R].
\]

Combining the three regions, for all $x>0$ we have
\[
|h(x)| \le M \triangleq \max\{M_1,M_2,M_3\},
\]
so $h$ is bounded on $(0,\infty)$.
\end{proof}

\pagebreak

\begin{lemma}
\label{lem:6}
If $\ratio<1$ then $\frac{\dd}{\dd\tilde{\lambda}}a$, $\frac{\dd}{\dd\tilde{\lambda}}b$ are bounded for all $\tilde \lambda\geq 0$.
\end{lemma}

\begin{proof}
Recall the notations from \cref{tab:notation}. 
First, we show that,
\begin{align}
\frac{\dd}{\dd\tilde{\lambda}}2a
&=\left(1+\ratio\right)D-\frac{\left(\tilde{\lambda}+\left(1+\ratio\right)\right)\left(\tilde{\lambda}\left(1+\ratio\right)+\left(1-\ratio\right)^{2}\right)}{D^{3}}\notag
\\&=\frac{\left(1+\ratio\right)D^{2}-\left(\tilde{\lambda}+\left(1+\ratio\right)\right)\left(\tilde{\lambda}\left(1+\ratio\right)+\left(1-\ratio\right)^{2}\right)}{D^{3}}
\notag
\\&
=\frac{\tilde{\lambda}\left(1+\ratio\right)^{2}-\tilde{\lambda}\left(1-\ratio\right)^{2}}{D^{3}}=\frac{4\ratio\tilde\lambda}{D^{3}}\,.
\label{eq:dadlt}
\end{align}
Furthermore,
\begin{align}
\frac{\dd}{\dd\tilde{\lambda}}2b
&=\frac{-D+\left(\tilde{\lambda}+\ratio+1\right)}{D}=\frac{-D^{2}+\left(\tilde{\lambda}+\ratio+1\right)^{2}}{D\left(\tilde{\lambda}+\ratio+1+D\right)}\notag\\&=\frac{-\left(1-\ratio\right)^{2}+\left(1+\ratio\right)^{2}}{D\left(\tilde{\lambda}+\ratio+1+D\right)}=\frac{4\ratio}{D\left(\tilde{\lambda}+\ratio+1+D\right)}\,.
\label{eq:dbdlt}
\end{align}
Finally, invoking~\cref{lem:5},
\begin{align*}
\frac{\dd}{\dd\tilde{\lambda}}a&
=
\frac{2\ratio\tilde{\lambda}}{D^{3}}
\stackrel{D\geq1-\ratio,D\geq\tilde{\lambda}}{\leq}
\frac{2\ratio}{\left(1-\ratio\right)^{2}}
\notag
\\
\frac{\dd}{\dd\tilde{\lambda}}b
&=\frac{2\ratio}{D\left(\tilde{\lambda}+\ratio+1+D\right)}
\stackrel{D\geq1-\ratio}{\leq}
\frac{\ratio}{1-\ratio}
\,.
\end{align*}
\end{proof}

\begin{lemma}
\label{lem:a->0b->0}
If $\ratio=1$ then $\lim_{\tilde{\lambda}\to0^{+}}a=0$, $\lim_{\tilde{\lambda}\to0^{+}}b=0$.
\end{lemma}
\begin{proof}
By the notations from \cref{tab:notation}, if $\ratio=1$ then,
\begin{equation}   
a=\frac{1}{2}\left(\frac{2\tilde{\lambda}}{\sqrt{\tilde{\lambda}^{2}+4\tilde{\lambda}}}\right)=\frac{\cancel{\tilde{\lambda}}}{\cancel{\tilde{\lambda}}\sqrt{1+\frac{4}{\tilde{\lambda}}}}=\frac{1}{\sqrt{1+\frac{4}{\tilde{\lambda}}}}\,,
\label{eq:31}
\end{equation}
and
\begin{equation}
b=\frac{1}{2}\left(-\tilde{\lambda}+\sqrt{\tilde{\lambda}^{2}+4\tilde{\lambda}}\right)
\,.
\label{eq:32}
\end{equation}
So, $\lim_{\tilde{\lambda}\to0^{+}}a=0$ and $\lim_{\tilde{\lambda}\to0^{+}}b=0$.
\end{proof}

\pagebreak

\begin{lemma}
\label{lem:7}
If $\ratio=1$ then $\forall n_{1}\geq 2,n_{2}\geq 3,$ we have $a^{n_{1}}\frac{\dd a}{\dd\tilde{\lambda}}$,$b^{n_{2}}\frac{\dd b}{\dd\tilde{\lambda}}$ are bounded for all $\tilde \lambda >0$. 
\end{lemma}
\begin{proof}
By~\cref{lem:0<ab<1}, it is sufficient to show this is true for $n_1=2,n_2=3$.

We use the closed-from expression for $\frac{\dd a}{\dd{\lambda}},\frac{\dd b}{\dd{\lambda}}$ given in~\cref{lem:da:dl,lem:db:dl}.

Using~\cref{eq:dadlt,eq:dbdlt,eq:31,eq:32},
\begin{align*}
\lim_{\tilde{\lambda}\to0^{+}}a^{2}\frac{\dd a}{\dd\tilde{\lambda}}&=\lim_{\tilde{\lambda}\to0^{+}}\frac{2}{\left(1+\frac{4}{\tilde{\lambda}}\right)}\frac{\tilde{\lambda}}{\left(\tilde{\lambda}^{2}+4\tilde{\lambda}\right)^{\hfrac{3}{2}}}=0\\\lim_{\tilde{\lambda}\to\infty}a^{2}\frac{\dd a}{\dd\tilde{\lambda}}&=\lim_{\tilde{\lambda}\to\infty}a^{2}\frac{\tilde{\lambda}}{2\left(\tilde{\lambda}^{2}+4\tilde{\lambda}\right)^{\hfrac{3}{2}}}\stackrel{[1]}{=}0\,,
\end{align*}
where [1] follows since $a$ by~\cref{lem:0<ab<1} is bounded for all $\tilde \lambda >0$. Furthermore,
\begin{align*}
\lim_{\tilde{\lambda}\to0^{+}}b^{3}\frac{\dd b}{\dd\tilde{\lambda}}
&
=\lim_{\tilde{\lambda}\to0^{+}}\frac{\left(-\tilde{\lambda}+\sqrt{\tilde{\lambda}^{2}+4\tilde{\lambda}}\right)^{3}}{\sqrt{\tilde{\lambda}^{2}+4\tilde{\lambda}}\left(\tilde{\lambda}+2+\sqrt{\tilde{\lambda}^{2}+4\tilde{\lambda}}\right)}=0
\\
\lim_{\tilde{\lambda}\to\infty}b^{3}\frac{\dd b}{\dd\tilde{\lambda}}
&
=\lim_{\tilde{\lambda}\to\infty}b^{3}\frac{2}{\sqrt{\tilde{\lambda}^{2}+4\tilde{\lambda}}\left(\tilde{\lambda}+2+\sqrt{\tilde{\lambda}^{2}+4\tilde{\lambda}}\right)}\stackrel{[2]}{=}0,
\end{align*}
where [2] follows since $b$ by~\cref{lem:0<ab<1} is bounded for all $\tilde \lambda >0$. 

Both of terms are bounded by~\cref{lem:c=1bound} as each one of the terms is continuous for all $\tilde\lambda>0$.
\end{proof}

For the next few Lemmas we define $f\triangleq \labelnoise\frac{1}{2\featnoise}\left(\frac{\left(1+\ratio\right)+\tilde{\lambda}}{\sqrt{\tilde{\lambda}^{2}+2\tilde{\lambda}\left(1+\ratio\right)+\left(1-\ratio\right)^{2}}}-1\right)=\labelnoise c$.
\begin{lemma}
\label{lem:C>0}
$f>0$ for all $\tilde \lambda >0$.
\end{lemma}
\begin{proof}
It is sufficient to show that $(1+\ratio+\tilde \lambda)^2 > \tilde\lambda^2 + 2\tilde\lambda(1+\ratio) + (1-\ratio)^2$ which is true if and only if $(1+\ratio)^2 > (1-\ratio)^2 \;\;\iff\;\; 4\ratio>0 \;\;\iff\;\; \ratio>0$.
\end{proof}

\begin{lemma}
\label{lem:c=1dc}
If $\ratio=1$ then for all $n_1\geq4,n_2\geq 2$ we have $a^{n_1}\frac{\dd f}{\dd\tilde{\lambda}},a^{n_2}f$ is bounded for all $\tilde \lambda >0$.
\end{lemma}
\begin{proof}
By~\cref{lem:0<ab<1} it is sufficient to show this is true for $n_1=4,n_2=2$.

Note that using~\cref{eq:31}, we get,
\begin{align}
& a^{4}\,\frac{\dd f}{\dd\tilde{\lambda}}= \left(\frac{\tilde{\lambda}}{\tilde{\lambda}+4}\right)^{\!2}\left(-\frac{2\labelnoise}{\featnoise\big(\tilde{\lambda}^{2}+4\tilde{\lambda}\big)^{3/2}}\right) 
= -\frac{2\labelnoise}{\featnoise}\,  \frac{\sqrt{\tilde{\lambda}}}{(\tilde{\lambda}+4)^{7/2}}.&
\label{eq:17}
\end{align}
Furthermore,
\begin{align}
&\,a^{2}f=\featnoise\frac{1}{2\labelnoise}\left(\frac{\tilde{\lambda}+2}{\tilde{\lambda}\sqrt{1+\frac{4}{\tilde{\lambda}}}}-1\right)\left(\frac{\tilde{\lambda}}{\tilde{\lambda}+4}\right).&
\label{eq:18}
\end{align}
By~\cref{eq:17,eq:18}, we have $a^4\frac{\dd f}{\dd\tilde{\lambda}}\to0$, $a^2f\to 0$ as $\tilde{\lambda}\to0^{+}$, but note that also $a^4 \frac{\dd f}{\dd\tilde{\lambda}}\to0$, $a^2f\to 0$ as $\tilde{\lambda}\to\infty$ by the same equations.
Since both terms are continuous for all $\tilde \lambda >0$ they are bounded in this region by~\cref{lem:c=1bound}.
\end{proof}

\pagebreak

\begin{lemma}
\label{eq:gamma=1a}
If $\ratio=1$ then $\frac{1-b}{1-a}=\frac{1}{1+a}$.
\end{lemma}
\begin{proof}
Recall the notations from \cref{tab:notation},
$$
a=\frac12\!\left(1-\ratio+\frac{(1+\ratio)\tilde{\lambda}+(1-\ratio)^2}{D}\right)
=\frac{\tilde{\lambda}}{D},\qquad
b=\frac12\!\left(1-\ratio-\tilde{\lambda}+D\right)
=\frac{D-\tilde{\lambda}}{2}.
$$
Thus
$$
\frac{1-b}{1-a}
=\frac{1-\frac{D-\tilde{\lambda}}{2}}{1-\frac{\tilde{\lambda}}{D}}
=\frac{D(2-D+\tilde{\lambda})}{2(D-\tilde{\lambda})}
=\frac{D}{D+\tilde{\lambda}}
=\frac{1}{1+\frac{\tilde{\lambda}}{D}}
=\frac{1}{1+a},
$$
where the penultimate equality uses 
$(2-D+\tilde{\lambda})(D+\tilde{\lambda})=2(D-\tilde{\lambda})
\iff D^{2}=\tilde{\lambda}^{2}+4\tilde{\lambda}$, true by the definition of $D$.
\end{proof}

\bigskip

\begin{lemma}
\label{lem:1:a}
If $\ratio=1$ then $\frac{1-b}{1-a},\frac{\dd}{\dd\tilde\lambda}\frac{1-b}{1-a}$ are bounded for all $\tilde \lambda >0$.
\end{lemma}
\begin{proof}
Since $a\to0$, $b\to 0$ as $\tilde \lambda \to 0^+$ by~\cref{lem:a->0b->0} then $\lim_{\tilde\lambda\to0^+}\frac{1-b}{1-a}=1$, furthermore, we know from~\cref{lem:laur1} that $\lim_{\tilde \lambda\to\infty}\frac{1-b}{1-a}=\frac{1}{2}$, invoking~\cref{lem:c=1bound} $\frac{1-b}{1-a}$ is bounded for all $\tilde \lambda >0$.

When $\ratio=1$ we obtain the following equality $\frac{1-b}{1-a}=\frac{1}{1+a}$ by~\cref{eq:gamma=1a}, then by~\cref{lem:0<ab<1} this quantity is bounded.
\end{proof}

\bigskip

\begin{lemma}
\label{lem:diffrentiable}
Let $h$ be a function analytic at $\infty$, where we define $h\left(x\right)=g\left(\frac{1}{x}\right)$, consider the Taylor series around $0$ of $g\left(x\right)=P_{k}\left(x\right)+R\left(x\right)$ where $R\left(x\right)=O\left(x^{k}\right)$ and $P_{k}\left(x\right)$ is a polynomial of degree at most $k$ then $\frac{\dd}{\dd x}h\left(x\right)=-\frac{\dd}{\dd x}P_{k}\left(\frac{1}{x}\right)\frac{1}{x^{2}}+O\left(\frac{1}{x^{k+2}}\right)$.
\end{lemma}

\begin{proof}
From the assumptions $g$ is analytic around $0$ then for some $\rho>0$,
$$
h\left(\frac{1}{x}\right)=g\left(x\right)=\sum_{n=0}^{k}\frac{g^{(n)}\left(0\right)}{n!}x^{n}+\sum_{n=k+1}^{\infty}\frac{g^{(n)}\left(0\right)}{n!}x^{n}\quad\forall\left|x\right|<\rho$$
A power series can be differentiated term by term within its radius of convergence, as the series of derivatives converges uniformly on any closed interval $\left|x\right|\leq r<\rho$, differentiating,
$$
\frac{\dd}{\dd x}g\left(x\right)=\frac{\dd}{\dd x}\sum_{n=0}^{k}\frac{g^{(n)}\left(0\right)}{n!}x^{n}+\sum_{n=k+1}^{\infty}\frac{g^{(n)}\left(0\right)}{\left(n-1\right)!}x^{n-1}\,,
\quad\forall\left|x\right|\leq r
\,.$$
For any $\left|x\right|<r$ write $\left|x\right|=r\theta $ where $\theta\in\left[0,1\right]$
\begin{align*}
\left|\frac{\dd}{\dd x}g\left(x\right)-\frac{\dd}{\dd x}\sum_{n=0}^{k}\frac{g^{(n)}\left(0\right)}{n!}x^{n}\right|&=\sum_{n=k+1}^{\infty}\frac{g^{(n)}\left(0\right)}{\left(n-1\right)!}\left|x\right|^{n-1}
\stackrel{{r\theta=\left|x\right|}}{=}
\sum_{n=k+1}^{\infty}\frac{g^{(n)}\left(0\right)}{\left(n-1\right)!}\left(r\theta\right)^{n-1}
\\ &\leq\theta^{k}\underbrace{\sum_{n=k+1}^{\infty}\frac{g^{(n)}\left(0\right)}{\left(n-1\right)!}r^{n-1}}_{\text{convergent sum}}=C\theta^{k}=\frac{C}{r^{k}}\left|x\right|^{k}
\,.
\end{align*}
Overall,
$$
\frac{\dd}{\dd x}g\left(x\right)=\frac{\dd}{\dd x}\sum_{n=0}^{k}\frac{g^{(n)}\left(0\right)}{n!}x^{n}+O\left(x^{k}\right)\quad\forall\left|x\right|\leq r$$
So, 
\begin{align*}
\frac{\dd}{\dd x}h\left(x\right)
&
=g'\left(\frac{1}{x}\right)\cdot\left(-\frac{1}{x^{2}}\right)=\left[P_{k}'\left(\frac{1}{x}\right)+O\left(\frac{1}{x^{k}}\right)\right]\cdot\left(-\frac{1}{x^{2}}\right)\\&=\frac{\dd}{\dd x}P_{k}\left(\frac{1}{x}\right)+O\left(\frac{1}{x^{k+2}}\right)\quad\forall\left|x\right|\geq\frac{1}{r}
\end{align*}

\end{proof}

\begin{lemma}
\label{lem:bounded}
Let $f$ be a function analytic at $\infty$, so that as $x\to\infty$ it admits the expansion of the form $f\left(x\right)=C+O\left(x^{-n}\right)$ for some constant $C$ and $n\in\mathbb{N}$, if in addition $f$ is continuous for all $x\geq \delta$ for some $\delta\in \mathbb R$ then $f$ is bounded for all $x\geq \delta$.
\end{lemma}

\begin{proof}
There exists some $R>0, M>0$ which for all $x\in\left[R,\infty\right)$ we have $\left|f\left(x\right)-C\right|\leq MR^{-n}$. On the compact interval $x\in\left[\delta,R\right]$ $f$ is continuous thus bounded, denote $\sup_{x\in[\delta,R]}|f(x)|=B<\infty$. Therefore $|f(x)|\le \max\{B,\,|C|+M R^{-n}\}$
for all $x\geq \delta$, proving boundedness.
\end{proof}

\begin{lemma}
\label{lem:Laurent-expansion-around}
We derive the following Laurent expansion for $\tilde \lambda\to \infty$.

Recall the notations from \cref{tab:notation}, using the standard Taylor expansion around $u=0$.
\begin{align*}
a&=\frac{1}{2}\left(1-\ratio+\frac{\left(1+\ratio\right)+\frac{\left(1-\ratio\right)^{2}}{\tilde{\lambda}}}{\sqrt{1+\frac{2\left(1+\ratio\right)}{\tilde{\lambda}}+\frac{\left(1-\ratio\right)^{2}}{\tilde{\lambda}^{2}}}}\right)=
\\\explain{\left(1+u\right)^{-\hfrac{1}{2}}\\=1-\frac{1}{2}u+O\left(u^{2}\right)}
&=\frac{1}{2}\left(1-\ratio+\left(\left(1+\ratio\right)+\frac{\left(1-\ratio\right)^{2}}{\tilde{\lambda}}\right)\left(1-\frac{1+\ratio}{\tilde{\lambda}}+O\left(\tilde{\lambda}^{-2}\right)\right)\right)
\\
&=\frac{1}{2}\left(2-\frac{4\ratio}{\tilde{\lambda}}+O\left(\tilde{\lambda}^{-2}\right)\right)=1-\frac{2\ratio}{\tilde{\lambda}}+O\left(\tilde{\lambda}^{-2}\right),
\\
a^{t}&=\exp\left(t\ln\left(a\right)\right)=\exp\left(-\frac{2{\ratio}t}{\tilde{\lambda}}+O\left(t\tilde{\lambda}^{-2}\right)\right)=1-\frac{2\ratio t}{\tilde{\lambda}}+O\left(t^{2}\tilde{\lambda}^{-2}\right),
\\
b&=\frac{1}{2}\left(1-\ratio-\tilde{\lambda}+\tilde{\lambda}\sqrt{1+\frac{2\left(1+\ratio\right)}{\tilde{\lambda}}+\frac{\left(1-\ratio\right)^{2}}{\tilde{\lambda}^{2}}}\right),
\\
\explain{(1+u)^{\hfrac{1}{2}}\\=1+\frac{1}{2}u-\frac{1}{8}u^{2}+O\left(u^{3}\right)}
&
=\frac{1}{2}\left(1-\ratio-\tilde{\lambda}+\tilde{\lambda}\left(1+\frac{\left(1+\ratio\right)}{\tilde{\lambda}}-\frac{2\ratio}{\tilde{\lambda}^{2}}\right)+O\left(\tilde{\lambda}^{-2}\right)\right)
\\
&=\frac{1}{2}\left(2-\frac{2\ratio}{\tilde{\lambda}}+O\left(\tilde{\lambda}^{-2}\right)\right)=1-\frac{\ratio}{\tilde{\lambda}}+O\left(\tilde{\lambda}^{-2}\right)
\\
b^{t}&=\exp\left(t\ln\left(b\right)\right)=\exp\left(-\frac{\ratio t}{\tilde{\lambda}}+O\left(t\tilde{\lambda}^{-2}\right)\right)=1-\frac{\ratio t}{\tilde{\lambda}}+O\left(t^{2}\tilde{\lambda}^{-2}\right).
\end{align*}
\end{lemma}

\pagebreak

\begin{lemma}
\label{lem:laur1}
The Laurent expansion of $\frac{1-b}{1-a}=\frac{1}{2}+O\left(\tilde{\lambda}^{-1}\right)$.
\end{lemma}

\begin{proof}
Using \cref{lem:Laurent-expansion-around}, we have
\begin{align*}
\frac{1-b}{1-a}&=\frac{\frac{\ratio}{\tilde{\lambda}}+O\left(\tilde{\lambda}^{-2}\right)}{\frac{2\ratio}{\tilde{\lambda}}+O\left(\tilde{\lambda}^{-2}\right)}=\frac{1}{2}+O\left(\tilde{\lambda}^{-1}\right)
\,.
\end{align*}
\end{proof}

\begin{lemma}
\label{lem:laur2}
The Laurent expansion of $f=O\left(\tilde{\lambda}^{-2}\right)$.
\end{lemma}

\begin{proof}
Recall the notations from \cref{tab:notation}, namely, $f\triangleq\frac{\labelnoise}{2\featnoise}\left(\frac{1+\ratio+\tilde\lambda}{D}-1\right)$, then,
\begin{align*}
\frac{2\featnoise}{\labelnoise}f
&=\frac{\tilde{\lambda}+\ratio+1}{\sqrt{\left(\tilde{\lambda}+\ratio+1\right)^{2}-4\ratio}}-1=\frac{1}{\sqrt{1-\frac{4\ratio}{\left(\tilde{\lambda}+\ratio+1\right)^{2}}}}-1
\\
\explain{\left(1-u\right)^{-\frac{1}{2}}=1+\frac{1}{2}u+O\left(u^{2}\right)}
&
=1+\frac{2\ratio}{\left(\tilde{\lambda}+\ratio+1\right)^{2}}+O\left(\tilde{\lambda}^{-4}\right)-1=O\left(\tilde{\lambda}^{-2}\right)
\,,
\end{align*}
so $f = \frac{\labelnoise}{2\featnoise}O\left(\tilde \lambda^{-2}\right)=O\left(\tilde \lambda^{-2}\right).$
\end{proof}

\begin{lemma}
\label{lem:laur3}
The Laurent expansion of $D=\tilde{\lambda}+\ratio+1-\frac{2\ratio}{\tilde{\lambda}+\ratio+1}+O\left(\tilde{\lambda}^{-3}\right)$.
\end{lemma}

\begin{proof}
Recall the notations from \cref{tab:notation}, namely, $D\triangleq\sqrt{\left(\tilde{\lambda}+\ratio+1\right)^{2}-4\ratio}$, then,
\begin{align*}
D&=\sqrt{\left(\tilde{\lambda}+\ratio+1\right)^{2}-4\ratio}=\left(\tilde{\lambda}+\ratio+1\right)\sqrt{1-\frac{4\ratio}{\left(\tilde{\lambda}+\ratio+1\right)^{2}}}
\\\explain{\sqrt{1-u}=1-\frac{1}{2}u+O\left(u^{2}\right)}
&
=\left(\tilde{\lambda}+\ratio+1\right)\left(1-\frac{2\ratio}{\left(\tilde{\lambda}+\ratio+1\right)^{2}}+O\left(\tilde{\lambda}^{-4}\right)\right)
\\
&
=\tilde{\lambda}+\ratio+1-\frac{2\ratio}{\tilde{\lambda}+\ratio+1}+O\left(\tilde{\lambda}^{-3}\right)
\,.
\end{align*}
\end{proof}

\newpage

\section{Proofs for~\cref{sec:noisy_teacher} (Special Case: Multiple i.i.d.~Teachers)}
\label{allproofs}

\subsection{Proof of~\cref{cor:closed}: Generalization under i.i.d.~Teachers}
In this subsection, we derive a closed-form expression for
$\mathbb{E}\!\left[\|\vw_T-\teacher_i\|^2\right]$ for any $1 \le i \le T$, with $T>1$,
under the i.i.d.\ teacher model introduced in~\cref{assump:iid_teachers}.
This result serves as the key intermediate step toward obtaining a closed-form
expression for the average generalization error.
Under~\cref{assump:iid_teachers}, each task-specific teacher can be written as
$$
\vw_i^\star = \vw^\star + \boldsymbol{\xi}_i, \qquad i \in [T],
$$
where the random vectors $\boldsymbol{\xi}_i$ are i.i.d.\ with mean $\0$ and covariance
$\mathrm{Cov}(\boldsymbol{\xi}_i)=\boldsymbol{\Sigma}$.
Substituting this representation into the general expression derived in
\cref{eq:mainresult1}, and using the initialization $\vw_0=\0$, yields the following
explicit decomposition of the squared error,
\begin{align*}
\mathbb{E}\left[\left\Vert \vw_{T}-\vw_{i}^{\star}\right\Vert ^{2}\right]&=a^{T}\left\Vert \vw^{\star}+\boldsymbol{\xi}_{1}\right\Vert^{2}+2b^{T}\left(\boldsymbol{\xi}_{i}-\boldsymbol{\xi}_{T}\right)^{\top}\left(\vw^{\star}+\boldsymbol{\xi}_{1}\right)\\&\qquad
+\labelnoise\frac{1}{2\featnoise}\left(\frac{\featnoise\left(1+\ratio\right)+\lambda}{\sqrt{\lambda^{2}+2\lambda\featnoise(1+\ratio)+\featnoise^{2}\left(1-\ratio\right)^{2}}}-1\right)\frac{1-a^{T}}{1-a}\\&\qquad+2\sum_{k=2}^{T}a^{T-k+1}b^{k-1}\left(\boldsymbol{\xi}_{k}-\boldsymbol{\xi}_{k-1}\right)^{\top}\left(\vw^{\star}+\boldsymbol{\xi}_{1}\right)\\&\quad\quad+\sum_{k=2}^{T}\sum_{k'=2}^{T}a^{T-\max\left(k,k'\right)+1}b^{\max\left(k,k'\right)-\min\left(k,k'\right)}\left(\boldsymbol{\xi}_{k-1}-\boldsymbol{\xi}_{k}\right)^{\top}\left(\boldsymbol{\xi}_{k'-1}-\boldsymbol{\xi}_{k'}\right)\\&\qquad+2\sum_{k=2}^{T}b^{T-k+1}\left(\boldsymbol{\xi}_{T}-\boldsymbol{\xi}_{i}\right)^{\top}\left(\boldsymbol{\xi}_{k-1}-\boldsymbol{\xi}_{k}\right)+\left\Vert \boldsymbol{\xi}_{T}-\boldsymbol{\xi}_{i}\right\Vert ^{2}.&
\end{align*}
Additionally, for all $i\in [T]$ we have,
\begin{align}
\mathbb{E}\left[\boldsymbol{\xi}_{i}^{\top}\boldsymbol{\xi}_{i}\right]
=
\mathbb{E}\sum_{j=1}^{d}{\xi}_{i,j}^{2}
=
\sum_{j=1}^{d}\mathbb{E}{\xi}_{i,j}^{2}
=
\sum_{j=1}^{d}\mathrm{Var}\left[{\xi}_{i,j}\right]=\trace\left(\boldsymbol{\Sigma}\right)
\,.
\label{eq:xi:id}
\end{align}
where the equality follows from independence across coordinates and the definition of the covariance matrix.
Using the identity in~\cref{eq:xi:id}, and taking expectations with respect to the random variables 
$\boldsymbol{\xi}_1,\ldots,\boldsymbol{\xi}_T$, we evaluate the resulting expression
by separating the analysis into three cases, depending on the relative position of
the index $i$.

In the special case, when $i=T$:
\begin{align}
&
\mathbb{E}\left[\left\Vert \vw_{T}-\vw_{i}^{\star}\right\Vert ^{2}\right]\notag
\\
&=\left(a^{T}-2a^{T-1}b\right)\trace\left(\boldsymbol{\Sigma}\right)+a^{T}\left\Vert \vw^{\star}\right\Vert ^{2}
+c\labelnoise \frac{1-a^{T}}{1-a}+2\trace\left(\boldsymbol{\Sigma}\right)\left(\sum_{k=2}^{T}a^{T-k+1}-b\sum_{k=3}^{T}a^{T-k+1}\right)\notag
\\
&=\left(a^{T}-2a^{T-1}b\right)\trace\left(\boldsymbol{\Sigma}\right)+a^{T}\left\Vert \vw^{\star}\right\Vert ^{2}\notag
\\
&\hspace{3em}+c\labelnoise\frac{1-a^{T}}{1-a}+2\trace\left(\boldsymbol{\Sigma}\right)\left(\frac{a\left(1-a^{T-1}\right)}{1-a}-b\cdot\frac{a\left(1-a^{T-2}\right)}{1-a}\right)\,.
\label{eq:i=t}
\end{align}

When $1<i<T$:
\begin{align}
    &\mathbb{E}\left[\left\Vert \vw_{T}-\vw_{i}^{\star}\right\Vert ^{2}\right]=\left(a^{T}-2a^{T-1}b\right)\trace\left(\boldsymbol{\Sigma}\right)+a^{T}\left\Vert \vw^{\star}\right\Vert ^{2}+c\labelnoise\frac{1-a^{T}}{1-a}\notag\\&\hspace{1em}+2\trace\left(\boldsymbol{\Sigma}\right)\left(\frac{a\left(1-a^{T-1}\right)}{1-a}-b\cdot\frac{a\left(1-a^{T-2}\right)}{1-a}+1\right)+2\left(b^{T-i+1}-b^{T-i}-b\right)\trace\left(\boldsymbol{\Sigma}\right).
    \label{eq:i>1}
\end{align}

When $i=1$:
\begin{align}
    &\mathbb{E}\left[\left\Vert \vw_{T}-\vw_{i}^{\star}\right\Vert ^{2}\right]=\left(a^{T}-2a^{T-1}b\right)\trace\left(\boldsymbol{\Sigma}\right)+a^{T}\|\vw^{\star}\|^{2}+2b^{T}\trace\left(\boldsymbol{\Sigma}\right)+c\labelnoise\frac{1-a^{T}}{1-a}\notag\\&\hspace{3em}+2\trace\left(\boldsymbol{\Sigma}\right)\left(\frac{a\left(1-a^{T-1}\right)}{1-a}-b\cdot\frac{a\left(1-a^{T-2}\right)}{1-a}+1\right)+2\bigl(b^{T}-b^{T-1}-b\bigr)\trace\left(\boldsymbol{\Sigma}\right).
    \label{eq:t>1}
\end{align}

\bigskip
\begin{recall}[\cref{cor:closed}]
Under the i.i.d.~teacher setting  (\cref{assump:iid_teachers}), and for all $T>1$, the expected generalization loss of \cref{schm:regularized} becomes, 
\begin{align*}
\mathbb E[G_{T}]=
&
\underbrace{
\labelnoise{c}\frac{1-a^{T}}{1-a}
}_{\text{Label noise}}
+
\underbrace{
a^{T}\left\Vert \vw^{\star}\right\Vert ^{2}
}_{\text{Teacher scale}}
+
\underbrace{
2\trace(\boldsymbol{\Sigma})
}_{\substack{\text{Teacher}\\\text{variance}}}
\prn{\frac{1-b}{1-a}+\frac{2b^{T}-1}{T}+\frac{a^{T}\left(-a+2b-1\right)}{2\left(1-a\right)}}
\,.
\end{align*}
\end{recall}
\begin{proof}
Using~\cref{eq:i=t,eq:i>1,eq:t>1} when summing over all cases,
\begin{align}
&\mathbb{E}\left[G_{T}\right]=\frac{1}{T}\sum_{i=1}^{T}\mathbb{E}\left[\left\Vert \vw_{T}-\vw_{i}^{\star}\right\Vert ^{2}\right]&\notag\\=&\left(a^{T}-2a^{T-1}b\right)\trace\left(\boldsymbol{\Sigma}\right)+a^{T}\left\Vert \vw^{\star}\right\Vert ^{2}+2\trace\left(\boldsymbol{\Sigma}\right)\left(\frac{a\left(1-a^{T-1}\right)}{1-a}-b\cdot\frac{a\left(1-a^{T-2}\right)}{1-a}\right)\notag+\\&c\labelnoise\frac{1-a^{T}}{1-a}+\frac{2}{T}\left(2b^{T}+T-1-Tb\right)\trace\left(\boldsymbol{\Sigma}\right)\label{eq:finite:t}\\=&c\labelnoise\frac{1-a^{T}}{1-a}+a^{T}\left\Vert \vw^{\star}\right\Vert ^{2}+\notag\\&2\trace(\boldsymbol{\Sigma})\prn{1-b+\frac{2b^{T}-1}{T}+\frac{a}{1-a}\left(1-a^{T-1}-b\left(1-a^{T-2}\right)\right)+\frac{a^{T-1}\left(a-2b\right)}{2}}%
\notag\\%
=&c\labelnoise\frac{1-a^{T}}{1-a}+a^{T}\left\Vert \vw^{\star}\right\Vert ^{2}+\notag\\&2\trace(\boldsymbol{\Sigma})\prn{1-b+\frac{2b^{T}-1}{T}+\frac{a}{1-a}\prn{1-b+a^{T-2}\prn{b-a}}+\frac{a^{T-1}\left(a-2b\right)}{2}}%
\notag\\%
=&c\labelnoise\frac{1-a^{T}}{1-a}+a^{T}\left\Vert \vw^{\star}\right\Vert ^{2}+\notag\\&2\trace(\boldsymbol{\Sigma})\prn{\frac{(1-b)(1-a)}{1-a}+\frac{a(1-b)}{1-a}+\frac{2b^{T}-1}{T}+\frac{a^{T-1}}{1-a}\prn{b-a}+\frac{a^{T-1}\left(a-2b\right)}{2}}%
\notag\\%
=&c\labelnoise\frac{1-a^{T}}{1-a}+a^{T}\left\Vert \vw^{\star}\right\Vert ^{2}+2\trace(\boldsymbol{\Sigma})\prn{\frac{1-b}{1-a}+\frac{2b^{T}-1}{T}+\frac{a^{T}\left(-a+2b-1\right)}{2\left(1-a\right)}}.%
\notag
\end{align}
\end{proof}

\newpage
\subsection{Generalization with an Infinite Task Horizon}
We simplify \cref{cor:closed} further, by considering the limit $T\to\infty$.

\begin{corollary}
\label{thm:infinite_budget} 
Under i.i.d.~teachers (\cref{assump:iid_teachers}) we have,
\begin{align*}
&
\lim_{T\to\infty}\mathbb{E}\left[G_{T}\right]
=
2\trace\left(\mathbf{\Sigma}\right)\left(\frac{1-b}{1-a}\right)+
\frac{\labelnoise}{2\featnoise}\left(\frac{\featnoise\left(1+\ratio\right)+\lambda}{\sqrt{\lambda^{2}+2\lambda\featnoise\left(1+\ratio\right)+\featnoise^{2}\left(1-\ratio\right)^{2}}}-1\right)\frac{1}{1-a}
\,.
\end{align*}
\end{corollary}
\begin{proof}
This is readily proved by taking $T\to \infty$ in~\cref{cor:closed},
\begin{align}
&\lim\limits_{T\to\infty}\mathbb{E}\left[G_{T}\right]
\notag
\\
&
\quad
=
\lim\limits_{T\to\infty}
\prn{
\labelnoise{c}\frac{1-a^{T}}{1-a}
+
a^{T}\left\Vert \vw^{\star}\right\Vert ^{2}
+
2\trace(\boldsymbol{\Sigma})
\prn{
\frac{1-b}{1-a}
+
\frac{2b^{T}-1}{T}
+
\frac{a^{T}\prn{-a +2b -1}
}{2(1-a)}
}
}
\notag\\
&
\quad
=2\trace\left(\boldsymbol{\Sigma}\right)\left(\frac{1-b}{1-a}\right)+\frac{\labelnoise}{2\featnoise}\left(\frac{\featnoise\left(1+\ratio\right)+\lambda}{\sqrt{\lambda^{2}+2\lambda\featnoise\left(1+\ratio\right)+\featnoise^{2}\left(1-\ratio\right)^{2}}}-1\right)\frac{1}{1-a}
\,.
\notag
\end{align}

We note in passing that, for any \emph{fixed} $i>1$, the same limit holds for  $\mathbb{E}\left[\left\Vert \vw_{T}-\vw_{i}^{\star}\right\Vert^2 \right]
$ by taking $T\to\infty$ in~\cref{eq:i>1}. That is,
\begin{align}
&\lim\limits_{T\to\infty}\mathbb{E}\left[\left\Vert \vw_{T}-\vw_{i}^{\star}\right\Vert^2 \right]
\notag
\\
&
\quad
=2\trace\left(\boldsymbol{\Sigma}\right)\left(\frac{1-b}{1-a}\right)+\frac{\labelnoise}{2\featnoise}\left(\frac{\featnoise\left(1+\ratio\right)+\lambda}{\sqrt{\lambda^{2}+2\lambda\featnoise\left(1+\ratio\right)+\featnoise^{2}\left(1-\ratio\right)^{2}}}-1\right)\frac{1}{1-a}
\,.
\label{cor:6o}
\end{align}
\end{proof}
The term depending on $\trace\left(\mathbf{\Sigma}\right)$ is an irreducible teacher-variance term, stemming from our definition of generalization,
i.e.,
$
\mathbb{E}
\left[G_T\right]
\triangleq
\frac{1}{T}\sum_{i=1}^{T}
\mathbb{E}
\left\Vert\vw_T-\teacher_i\right\Vert^{2}
$
(\cref{def:generlization:error}),
similar to the irreducible label-noise term in \cref{eq:expected_loss}.
The term here can be avoided by analyzing 
$\mathbb{E} \left\Vert\vw_T-\teacher\right\Vert^{2}$ instead.

\newpage

\subsection{Proof of~\cref{thm:asymptotic_mse}: Asymptotic Generalization under Strong Regularization}

\begin{recall}[\cref{thm:asymptotic_mse}]
Under the i.i.d.~teacher setting  (\cref{assump:iid_teachers}), when the task horizon is $T \to \infty$, the expected generalization loss is monotonically decreasing in the regularization strength $\lambda$.
Furthermore, the iterates converge in mean square to the ``global'' teacher $\teacher$,
i.e.,
$$
\lim_{\lambda\to\infty} \; \lim_{T \to \infty} \; 
\mathbb{E}
\left\Vert \vw_{T} - \vw^{\star}\right\Vert^{2} = 0.
$$
\end{recall}

\begin{proof}
We split the proof into two parts.

\subsubsection{Proving convergence}

Note that for any fixed $i$,
\begin{align*}
\lim_{\lambda\to\infty}
\lim\limits_{T\to\infty}
\mathbb{E}
\left\Vert \vw_{T}-\vw^{\star}\right\Vert ^{2}
&
=\lim_{\lambda\to\infty}\lim\limits_{T\to\infty}\mathbb{E}\left[\left\Vert \vw_{T}-\left(\vw_{i}^{\star}+\boldsymbol{\xi}_{i}\right)\right\Vert ^{2}\right]
\\
&=
\lim_{\lambda\to\infty}\lim\limits_{T\to\infty}
\biggprn{
\mathbb{E}
\left\Vert \vw_{T}-\vw_{i}^{\star}\right\Vert ^{2}
+
2\mathbb{E}\left[\left(\vw_{T}-\vw_{i}^{\star}\right)^{\top}\boldsymbol{\xi}_{i}\right]
+
\underbrace{
\mathbb{E}
\left\Vert \boldsymbol{\xi}_{i}\right\Vert ^{2}
}_{\trace\left(\boldsymbol{\Sigma}\right)}
}
.
\end{align*}

We proceed by calculating $\lim_{\lambda\to\infty}\lim\limits_{T\to\infty}
{\mathbb{E}\left[\left\Vert \vw_{T}-\vw_{i}^{\star}\right\Vert ^{2}\right]}$.
Recall that by \cref{cor:6o}, for any fixed $i\ge 2$,
\begin{align*}
&\lim\limits_{T\to\infty}\mathbb{E}\!\left[\left\Vert \vw_{T}-\vw_{i}^{\star}\right\Vert^2 \right]
\\
&
= \underbrace{\frac{\labelnoise}{2\featnoise}\!\left(\frac{\featnoise(1+\ratio)+\lambda}{\sqrt{\lambda^{2}+2\lambda\featnoise(1+\ratio)+\featnoise^{2}(1-\ratio)^{2}}}-1\right)\frac{1}{1-a}}_{\text{term 1}}
+ 
\underbrace{2\operatorname{Tr}(\boldsymbol{\Sigma})\left(\frac{1-b}{1-a}\right)}_{\text{term 2}} 
\,.
\end{align*}
Then, by \cref{lem:laur1}, as $\lambda\to\infty$, we have
$\frac{1-b}{1-a}=\frac{1}{2}+O\left(\lambda^{-1}\right),$
implying that,
$\lim_{\lambda \to \infty}\frac{1-b}{1-a}=\frac{1}{2}$.
Furthermore, \cref{lem:laur2,lem:3} suggest that as $\lambda \to \infty$, it holds that, 
$$\frac{\labelnoise}{2\featnoise}\left(\frac{\featnoise\left(1+\ratio\right)+\lambda}{\sqrt{\lambda^{2}+2\lambda\featnoise\left(1+\ratio\right)+\featnoise^{2}\left(1-\ratio\right)^{2}}}-1\right)\frac{1}{1-a}=O\left(\lambda^{-1}\right),$$
implying $\lim_{\lambda\to\infty}\frac{\labelnoise}{2\featnoise}\left(\frac{\featnoise\left(1+\ratio\right)+\lambda}{\sqrt{\lambda^{2}+2\lambda\featnoise\left(1+\ratio\right)+\featnoise^{2}\left(1-\ratio\right)^{2}}}-1\right)\frac{1}{1-a}=0.$
So overall, 
$$\lim_{\lambda\to\infty}\lim\limits_{T\to\infty}
{\mathbb{E}\left[\left\Vert \vw_{T}-\vw_{i}^{\star}\right\Vert ^{2}\right]}=\trace\left(\boldsymbol{\Sigma}\right)
\,.$$

Combining all of the above, for some fixed $i\ge 3$:
\begin{align*}
\lim_{\lambda\to\infty}\lim\limits_{T\to\infty}
\mathbb{E}
\left\Vert \vw_{T}-\vw^{\star}\right\Vert ^{2}
&=
\lim_{\lambda\to\infty}\lim\limits_{T\to\infty}
\Bigprn{
2\trace\left(\boldsymbol{\Sigma}\right)
+2\mathbb{E}\left[\left(\vw_{T}-\vw_{i}^{\star}\right)^{\top}\boldsymbol{\xi}_{i}\right]
}
\\{\explain{1}}
&=2\trace\left(\boldsymbol{\Sigma}\right)-2\mathbb{E}\left[\left(\vw_{i}^{\star}\right)^{\top}\boldsymbol{\xi}_{i}\right]=2\trace\left(\boldsymbol{\Sigma}\right)-2\mathbb{E}\left[\left(\vw^{\star}+\boldsymbol{\xi}_{i}\right)^{\top}\boldsymbol{\xi}_{i}\right]
\\
&=2\trace\left(\boldsymbol{\Sigma}\right)-2\underbrace{\mathbb{E}\left[\left\Vert \boldsymbol{\xi}_{i}\right\Vert ^{2}\right]}_{\trace\left(\boldsymbol{\Sigma}\right)}=0
\,,
\end{align*}
where [1] follows by~\cref{eq:seq} where,
$$\vw_{T}=\left[\prod_{m=T}^{1}\mathbf{P}_{m}\right]\left(\vw_{0}-\vw_{1}^{\star}\right)+\frac{1}{\lambda d}\sum_{k=1}^{T}\left[\prod_{m=T}^{k}\mathbf{P}_{m}\right]\mathbf{X}_{k}^{\top}\mathbf{z}_{k}+\sum_{k=2}^{T}\left[\prod_{m=T}^{k}\mathbf{P}_{m}\right]\left(\vw_{k-1}^{\star}-\vw_{k}^{\star}\right)+\vw_{T}^{\star}.$$

Recall the identity and~\cref{eq:E[P]},
\begin{align*}
\mathbb{E}\left[{\prod}_{m=T}^{i}\mathbf{P}_{m}\right]
&
=
\mathbb{E}\left[\mathbf{P}_{T}\cdots\mathbf{P}_{i}\right]
=
\mathbb{E}_{\mathbf{P}_{i},\ldots,\mathbf{P}_{T-1}}\left[\mathbb{E}_{\mathbf{P}_{T}}\left[\mathbf{P}_{T}\cdots\mathbf{P}_{i}\Bigl|\mathbf{P}_{1},\ldots,\mathbf{P}_{T-1}\right]\right]\\&
=
\mathbb{E}\left[\mathbb{E}\left[\mathbf{P}_{T}\right]\mathbf{P}_{T-1}\cdots\mathbf{P}_{i}\right]%
=
\mathbb{E}\left[\mathbf{P}_{T}\right]\mathbb{E}\left[\mathbf{P}_{T-1}\ldots\mathbf{P}_{i}\right]\\&
=
\ldots%
=b^{T-i+1}\I\,.
\end{align*}
For any fixed $\lambda >0$, we have $0<b<1$ (by~\cref{lem:0<ab<1}).

Hence, for any fixed $i$, by independence of $\boldsymbol \xi_i$ of all other random variables we have, 
\begin{align*}
&\lim\limits_{T\to\infty}\mathbb E\left[\boldsymbol{\xi}_{i}^{\top}\mathbf{w}_{T}\right]=\lim\limits_{T\to\infty}\mathrm{Tr}\left(\mathbb{E}\left[\left(\Pi_{m=T}^{i+1}\mathbf{P}_{m}-\Pi_{m=T}^{i}\mathbf{P}_{m}\right)\right]\boldsymbol{\Sigma}\right)=\lim\limits_{T\to\infty}\mathrm{Tr}\left(b^{T-i}(1-b)\boldsymbol{\Sigma}\right)\to0,
\end{align*}
here, we used the fact that $\|\boldsymbol{\Sigma}\|_*\leq C<\infty$.

\subsubsection{Proving monotonicity}
We show that each one of the terms in the following expression is monotonic decreasing with $\lambda$.

Recall~\cref{thm:infinite_budget}, according to which, 
\begin{align}
\label{eq:limits}
&\lim\limits_{T\to\infty}\mathbb{E}\left[G_{T}\right]
=\underbrace{2\trace\left(\boldsymbol{\Sigma}\right)\left(\tfrac{1-b}{1-a}\right)}_{\text{term 1}}+\underbrace{\frac{\labelnoise}{2\featnoise}\left(\frac{\featnoise\left(1+\ratio\right)+\lambda}{\sqrt{\lambda^{2}+2\lambda\featnoise\left(1+\ratio\right)+\featnoise^{2}\left(1-\ratio\right)^{2}}}-1\right)\tfrac{1}{1-a}}_{\text{term 2}}.
\end{align}
First, we analyze the \textbf{first term} in the summation:
\begin{align*}
    \frac{\dd}{\dd\lambda}2\trace\left(\boldsymbol{\Sigma}\right)\left(\frac{1-b}{1-a}\right)&=
    2\trace\left(\boldsymbol{\Sigma}\right)
    \frac{\dd}{\dd\lambda}
    \frac{1-b}{1-a}\,.
\end{align*}

Next, show that $\frac{\dd}{\dd\lambda}\frac{1-b}{1-a}<0$.
Recall the notations from \cref{tab:notation},
\begin{align*}
    &1-b=1-\frac{1}{2}\left(1-\ratio-\tilde{\lambda}+D\right)=\frac{1}{2}\left(1+\ratio+\tilde{\lambda}-D\right)&\\&1-a=1-\frac{1}{2}\left(1-\ratio+\frac{N}{D}\right)=\frac{1}{2}\left(1+\ratio-\frac{N}{D}\right)\,.
\end{align*}

Using the above closed-form expression and by~\cref{lem:da:dl,lem:db:dl},
\begin{align}
&
\frac{\dd}{\dd\lambda}\frac{1-b}{1-a}
=
\frac{1-b}{\left(1-a\right)^{2}}\cdot\frac{\dd a}{\dd \lambda}-\frac{1}{1-a}\cdot\frac{\dd b}{\dd \lambda}
\propto
\left(1-b\right)\frac{\dd a}{\dd\lambda}-\left(1-a\right)\frac{\dd b}{\dd\lambda}
\notag\\
&=\frac{1}{2}\left(1+\ratio+\tilde{\lambda}-D\right)\frac{1}{2\featnoise}\frac{4\tilde{\lambda}\ratio}{D^{3}}-\frac{1}{2}\left(1+\ratio-\frac{N}{D}\right)\left(\frac{1}{2\featnoise}\Bigl(\frac{-D+\tilde{\lambda}+1+\ratio}{D}\Bigr)\right)
\notag\\
&
=\frac{1}{4\featnoise}\left(\frac{1+\ratio+\tilde{\lambda}-D}{D}\right)\frac{4\tilde{\lambda}\ratio}{D^{2}}-\left(\frac{1}{4\featnoise}\Bigl(\frac{1+\ratio+\tilde{\lambda}-D}{D}\Bigr)\right)\left(1+\ratio-\frac{N}{D}\right)
\notag\\
&=\frac{1}{4\featnoise}\Bigl(\frac{1+\ratio+\tilde{\lambda}-D}{D}\Bigr)\left(\frac{4\tilde{\lambda}\ratio}{D^{2}}-\left(\frac{\left(1+\ratio\right)D^{2}-DN}{D^{2}}\right)\right)
\notag\\&
\stackrel{\left(1\right)}{=}\frac{1}{4\featnoise}\left(\frac{1+\ratio+\tilde{\lambda}-D}{D}\right)\left(\frac{\left(1+\ratio\right)D^{2}-N\left(\tilde{\lambda}+1+\ratio\right)-\left(1+\ratio\right)D^{2}+DN}{D^{2}}\right)
\notag\\
&
=\left(\frac{1}{4\featnoise}\left(\frac{1+\ratio+\tilde{\lambda}-D}{D}\right)\right)\left(\frac{-N\left(1+\ratio+\tilde{\lambda}-D\right)}{D^{2}}\right)
\notag\\&
=\frac{-N}{4\featnoise}\frac{\left(1+\ratio+\tilde{\lambda}-D\right)^{2}}{D^{3}}<0\label{d1-a:dl},
\end{align}
where [1] follows from~\cref{eq:23}.

Overall, since $\frac{\dd}{\dd\lambda}\frac{1-b}{1-a}<0$ and $\trace(\mathbf{\Sigma})\geq 0$, the derivative of the first term with respect to $\lambda$ is non-positive.

\bigskip

For the \textbf{second term} in the summation, we define
$$
k\triangleq\labelnoise\frac{1}{2\featnoise}\left(\frac{\featnoise\left(1+\ratio\right)+\lambda}{\sqrt{\lambda^{2}+2\lambda\featnoise\left(1+\ratio\right)+\featnoise^{2}\left(1-\ratio\right)^{2}}}-1\right)\frac{1}{1-a}.
$$
We wish to show that $\frac{\dd}{\dd\lambda}k<0$,
\begin{align*}
\frac{\dd}{\dd\lambda}k=&\labelnoise\frac{1}{2\featnoise}\underbrace{\frac{\dd}{\dd\lambda}\left(\frac{\featnoise\left(1+\ratio\right)+\lambda}{\sqrt{\lambda^{2}+2\lambda\featnoise\left(1+\ratio\right)+\featnoise^{2}\left(1-\ratio\right)^{2}}}-1\right)}_{\text{term 1}}\frac{1}{1-a}\\&+\labelnoise\frac{1}{2\featnoise}\left(\frac{\featnoise\left(1+\ratio\right)+\lambda}{\sqrt{\lambda^{2}+2\lambda\featnoise\left(1+\ratio\right)+\featnoise^{2}\left(1-\ratio\right)^{2}}}-1\right)\frac{\dd}{\dd\lambda}\frac{1}{1-a}.
\end{align*}

Using~\cref{lem:da:dl}, we get,
$$
\frac{\dd\frac{1}{1-a}}{\dd\lambda}=\frac{1}{\left(1-a\right)^{2}}\frac{2\tilde{\lambda}\ratio}{\featnoise D^{3}}>0.$$

Evaluating the derivative with respect to $\lambda$ of the first term,
\begin{align*}
&\frac{\dd}{\dd\lambda}\left[\frac{\featnoise\left(1+\ratio\right)+\lambda}{\sqrt{\lambda^{2}+2\lambda\featnoise\left(1+\ratio\right)+\featnoise^{2}\left(1-\ratio\right)^{2}}}-1\right]
&\\
&=
\frac{1}{\sqrt{\lambda^{2}+2\lambda\featnoise\left(1+\ratio\right)+\featnoise^{2}\left(1-\ratio\right)^{2}}}-\frac{\left(\lambda+\featnoise\left(1+\ratio\right)\right)\left(\featnoise\left(1+\ratio\right)+\lambda\right)}{\left(\lambda^{2}+2\lambda\featnoise\left(1+\ratio\right)+\featnoise^{2}\left(1-\ratio\right)^{2}\right)^{\frac{3}{2}}}
\\\
&=
\frac{\lambda^{2}+2\lambda\featnoise\left(1+\ratio\right)+\featnoise^{2}\left(1-\ratio\right)^{2}-\left(\lambda^{2}+\left(\featnoise\left(1+\ratio\right)\right)^{2}+2\lambda\featnoise\left(1+\ratio\right)\right)}{\left(\lambda^{2}+2\lambda\featnoise\left(1+\ratio\right)+\featnoise^{2}\left(1-\ratio\right)^{2}\right)^{\frac{3}{2}}}\\&=\frac{\featnoise^{2}\left(1-\ratio\right)^{2}-\featnoise^{2}\left(1+\ratio\right)^{2}}{\left(\lambda^{2}+2\lambda\featnoise\left(1+\ratio\right)+\featnoise^{2}\left(1-\ratio\right)^{2}\right)^{\frac{3}{2}}}=\frac{-4\featnoise^{2}\ratio}{\left(\lambda^{2}+2\lambda\featnoise\left(1+\ratio\right)+\featnoise^{2}\left(1-\ratio\right)^{2}\right)^{\frac{3}{2}}}<0.
\end{align*}
Combining the derivations above, we obtain
\begin{align}
\frac{\dd k}{\dd\lambda}
&=\frac{\dd}{\dd\lambda}\left[\labelnoise\frac{1}{2\featnoise}\left(\frac{\featnoise\left(1+\ratio\right)+\lambda}{\sqrt{\lambda^{2}+2\lambda\featnoise\left(1+\ratio\right)+\featnoise^{2}\left(1-\ratio\right)^{2}}}-1\right)\right]\cdot\frac{1}{1-a}+\notag\\&
\hspace{2em}\frac{\dd}{\dd\lambda}\left(\frac{1}{1-a}\right)\cdot\labelnoise\frac{1}{2\featnoise}\left(\frac{\featnoise\left(1+\ratio\right)+\lambda}{\sqrt{\lambda^{2}+2\lambda\featnoise\left(1+\ratio\right)+\featnoise^{2}\left(1-\ratio\right)^{2}}}-1\right)\notag\\&=\frac{\labelnoise}{2\featnoise}\left(\frac{-4\featnoise^{2}\ratio}{\featnoise^{3}D^{3}}\right)\cdot\frac{1}{1-a}+\frac{1}{\left(1-a\right)^{2}}\frac{4\tilde{\lambda}\ratio}{2\featnoise D^{3}}\frac{\labelnoise}{2\featnoise}\left(\frac{\featnoise\left(1+\ratio\right)+\lambda}{\featnoise D}-1\right)\notag\\&=\frac{2\labelnoise}{\featnoise^{2}}\frac{\ratio}{\left(1-a\right)D^{3}}\left(\frac{1}{\left(1-a\right)}\frac{\tilde{\lambda}}{2\featnoise D}\left(\featnoise\left(1+\ratio-D\right)+\lambda\right)-1\right)\notag\\&=\frac{2\labelnoise}{\featnoise^{2}}\frac{\ratio}{\left(1-a\right)^{2}D^{3}}\left(\frac{\tilde{\lambda}}{2\featnoise D}\left(\featnoise\left(1+\ratio-D\right)+\lambda\right)-\frac{1+\ratio}{2}-\frac{\tilde{\lambda}\left(1+\ratio\right)+\left(1-\ratio\right)^{2}}{2D}\right)
\notag\\
&=\frac{\labelnoise}{\featnoise^{2}}\frac{\ratio}{\left(1-a\right)^{2}D^{4}}\left(\tilde{\lambda}\left(\left(1+\ratio-D\right)+\tilde{\lambda}\right)-\left(D(1+\ratio)-\tilde{\lambda}(1+\ratio)-\left(1-\ratio\right)^{2}\right)\right)&\notag\\&=\frac{\labelnoise}{\featnoise^{2}}\frac{\ratio}{\left(1-a\right)D^{4}\left(1-a\right)}\left(\tilde{\lambda}^{2}+2\left(1+\ratio\right)\tilde{\lambda}+\left(1-\ratio\right)^{2}-\left(1+\ratio+\tilde{\lambda}\right)D\right)\notag\\&=\frac{\labelnoise}{\featnoise^{2}}\frac{\ratio}{\left(1-a\right)^{2}D^{3}}\left(D-\left(1+\ratio+\tilde{\lambda}\right)\right)
\label{eq:dkdl}\\
\explain{\ratio>0}
&{\leq}
\frac{\labelnoise}{\featnoise^{2}}\frac{\ratio}{\left(1-a\right)^{2}D^{3}}\left(\sqrt{\left(\tilde{\lambda}+\left(1+\ratio\right)\right)^{2}}-\left(1+\ratio+\tilde{\lambda}\right)\right)=0
\,\label{eq:strict1},
\end{align}
implying that $\lim\limits_{T\to\infty}\mathbb{E}\left[G_{T}\right]$ is monotonic decreasing as a function of $\lambda$ for all $\lambda>0$.
\end{proof}
\begin{remark}
\label{rem:1}
If either $\mathrm{Tr}(\boldsymbol\Sigma)>0$ or $\labelnoise>0$, then $\lim\limits_{T\to\infty}\mathbb{E}[G_T]$ is strictly decreasing in $\lambda$ for all $\lambda>0$: when $\labelnoise>0$, the inequality in \cref{eq:strict1} becomes strict; and when $\mathrm{Tr}(\boldsymbol\Sigma)>0$, the derivative of the first term in \cref{eq:limits} is strictly negative.
\end{remark}

\newpage

\subsection{Proof of~\cref{thm:more_tasks}: More Tasks Can Be Useful or Harmful}

\begin{recall}[\cref{thm:more_tasks}]
Consider the i.i.d.~teacher setting (\cref{assump:iid_teachers}) with a label noise $\labelnoise>0$.
\begin{itemize}[leftmargin=0.5cm, itemindent=0.05cm, labelsep=0.2cm, itemsep=1pt,topsep=-1pt]
    \item \textbf{Single teacher:}
    When $\vw^{\star}_1=\dots=\vw^{\star}_{T}$ (i.e., $\trace\left(\mathbf{\Sigma}\right)=0$), the expected generalization loss with an \emph{optimal} horizon-dependent strength $\lambda^\star(T)$ is monotonically \emph{decreasing} in the task horizon $T$. 
    \item \textbf{Multiple i.i.d.~teachers:} When teachers vary across tasks (i.e., $\trace\!\left(\mathbf{\Sigma}\right)>0$), then under any \emph{fixed} strength $\lambda>0$, the expected generalization loss is monotonically \emph{increasing} in the task horizon $T$ for all $T \geq T'$ for a sufficiently large $T'$.
\end{itemize}
\end{recall}

Next, we prove each case separately.

\subsubsection{Proving \cref{thm:more_tasks} in the single teacher case}


\begin{proof}
Recall~\cref{eq:main_result} when $\mathbf w_t^\star=\mathbf w^\star$ for all $t\in [T]$, and $\labelnoise>0$ the expression corresponding to $\mathbb E[G_T]$ can be simplified to a more compact form,
\begin{align}
\mathbb{E}[G_T]
&=\mathbb{E}\left\Vert \vw_{T}-\vw^{\star}\right\Vert ^{2}
\notag
\\
&=a^{T}\left\Vert \vw_{0}-\vw^{\star}\right\Vert ^{2}+
\frac{\labelnoise}{\featnoise}\frac{1-a^{T}}{\frac{\lambda}{\featnoise}+\sqrt{\left(\frac{\lambda}{\featnoise}\right)^{2}+2\frac{\lambda}{\featnoise}\left(1+\ratio\right)+\left(1-\ratio\right)^{2}}},&
\label{eq:singleteacher}
\end{align}
Define $s=\frac{\lambda}{\featnoise}+\sqrt{\left(\frac{\lambda}{\featnoise}\right)^{2}+2\frac{\lambda}{\featnoise}\left(1+\ratio\right)+\left(1-\ratio\right)^{2}}$, and 
\begin{align}
J\left(t,\lambda\right)&=a^{t}\left\Vert \vw_{0}-\vw^{\star}\right\Vert ^{2}+\frac{\labelnoise}{\featnoise}\frac{1-a^{t}}{s}\notag\\	\frac{\dd}{\dd \lambda}J\left(t,\lambda\right)&=\underbrace{\left[\left\Vert \vw_{0}-\vw^{\star}\right\Vert ^{2}-\frac{\labelnoise}{\featnoise}\frac{1}{s}\right]ta^{t-1}\frac{\dd a}{\dd \lambda}}_{\text{term 1}}-\underbrace{\frac{\labelnoise}{\featnoise}\frac{1-a^{t}}{s^{2}}\frac{\dd s}{\dd\lambda}}_{\text{term 2}}.
\label{eq:djdl}
\end{align}
We will show that $\left.\frac{\dd}{\dd\lambda}J\left(t,\lambda\right)\right|_{\lambda=0}<0$ when $\ratio<1$, otherwise if $\ratio=1$, the existence of $\bar \lambda>0 $ such that $\left.\frac{\dd}{\dd\lambda}J\left(t,\lambda\right)\right|_{\lambda=\bar \lambda}<0$ and $\lim_{\tilde{\lambda}\to\infty}\left.\frac{\dd}{\dd\lambda}J\left(t,\lambda\right)\right|_{\lambda=\tilde{\lambda}}>0$ for the defined scaled parameter $\tilde{\lambda}\triangleq\frac{\lambda}{\featnoise}$, note that if $\tilde{\lambda}\to\infty$  then $\lambda\to\infty $, continuity implies at least one stationary point $\lambda^{\star}\in\left(0,\infty\right)$ and it is a minimum.
\begin{align}
&\left.\frac{\dd}{\dd\lambda}J\left(t,\lambda\right)\right|_{\lambda=0}=_{\left.\frac{\dd a}{\dd \lambda}\right|_{\lambda=0}=0}-\frac{\labelnoise}{\featnoise}\frac{1-\left(1-\ratio\right)^{t}}{\left(1-\ratio\right)^{2}}\frac{\dd s}{\dd\lambda}\overset{\explain{0<\ratio<1}}{<}0\label{eq:11-1}\\&
\lim_{\tilde{\lambda}\to\infty}\left.\frac{\dd}{\dd\lambda}J\left(t,\lambda\right)\right|_{\lambda=\tilde{\lambda}}\notag=\lim_{\tilde{\lambda}\to\infty}\left[\left\Vert \vw_{0}-\vw^{\star}\right\Vert ^{2}-\underbrace{\frac{\labelnoise}{\featnoise}\frac{1}{s}}_{\to0}\right]ta^{t-1}\frac{\dd a}{\dd\tilde{\lambda}}-\frac{\labelnoise}{\featnoise}\frac{1-a^{t}}{s^{2}}\frac{\dd s}{\dd\tilde{\lambda}}\notag
\\
&\overset{\explain{1}}{\geq}
\lim_{\tilde{\lambda}\to\infty}\tfrac{1}{2}\left\Vert \vw_{0}-\vw^{\star}\right\Vert ^{2}t\left(1\!+\!O\!\left(\tilde{\lambda}^{-1}\right)\right)\left(\frac{2\ratio}{\tilde{\lambda}^{2}}+O\!\left(\tilde{\lambda}^{-3}\right)\right)-
\tfrac{\labelnoise}{\featnoise}\frac{\frac{2\ratio t}{\tilde{\lambda}}+O\!\left(\tilde{\lambda}^{-2}\right)}{\left(2\tilde{\lambda}+O\!\left(1\right)\right)^{2}}\left(1\!+\!O\!\left(\tilde{\lambda}^{-2}\right)\right)\notag
\\
&=
\lim_{\tilde{\lambda}\to\infty}\frac{1}{2}\left\Vert \vw_{0}-\vw^{\star}\right\Vert ^{2}t\left(\frac{2\ratio}{\tilde{\lambda}^{2}}+O\left(\tilde{\lambda}^{-3}\right)\right)-\frac{\labelnoise}{\featnoise}O\left(\tilde{\lambda}^{-3}\right)\left(1+O\left(\tilde{\lambda}^{-2}\right)\right)>0\,,&
\label{eq:10-1}
\end{align}
where [1] follows by~\cref{lem:Laurent-expansion-around,lem:diffrentiable} (as all terms are analytic at $\infty$ being a concatenation, sum, and division of analytic functions) and $s\to \infty$ as $\tilde \lambda \to \infty$.

For $\ratio=1$, we note that $a\to 0$ and $\tfrac{da}{d\lambda}\to 0$ as $\lambda\to 0^{+}$ by~\cref{lem:a->0b->0,lem:da:dl}, since the exponential decay in $a^{t}$ dominates any polynomial growth in $t$. Moreover,
$
\frac{\dd s}{\dd\lambda}
>\frac{1}{\featnoise}>0
$
for all $\lambda>0$.
Thus, the second term in $\tfrac{d}{d\lambda}J(t,\lambda)$~(\cref{eq:djdl}) dominates as $\lambda\to 0^{+}$, while remaining strictly negative. By continuity, there exists $\bar{\lambda}>0$ such that for all $\lambda\in(0,\bar{\lambda}]$ we have
$
\frac{\dd}{\dd\lambda}J(t,\lambda)<0.
$
Hence, near $\lambda=0$ the derivative is negative and the second term dominates.

At a stationary point we have 
\begin{align*}
\frac{\dd}{\dd\lambda}J\left(t,\lambda^{\star}\right)=0\iff\left[\left\Vert \vw_{0}-\vw^{\star}\right\Vert ^{2}-\frac{\labelnoise}{\featnoise}\frac{1}{s}\right]ta^{t-1}\left.\frac{\dd a}{\dd \lambda}\right|_{\lambda=\lambda^{\star}}=\frac{\labelnoise}{\featnoise}\frac{1-a^{t}}{s^{2}}\left.\frac{\dd s}{\dd\lambda}\right|_{\lambda=\lambda^{\star}},
\end{align*}

When differentiating with respect to $t$,
\begin{align}
&\frac{\dd}{\dd t}J\left(t,\lambda^{\star}\right)=a^{t}\ln a\left[\left\Vert \vw_{0}-\vw^{\star}\right\Vert^{2}-\frac{\labelnoise}{\featnoise}\frac{1}{s}\right]
\overset{\explain{2}}{=}\underbrace{\frac{a}{t\left.\frac{\dd a}{\dd \lambda}\right|_{\lambda=\lambda^{\star}}}}_{>0}\underbrace{\ln a}_{<0}\underbrace{\frac{\labelnoise}{\featnoise}\frac{1-a^{t}}{s^{2}}\left.\frac{\dd s}{\dd\lambda}\right|_{\lambda=\lambda^{\star}}}_{>0}<0.&
\label{eq:9-1}
\end{align}
Where [2] follows by~\cref{lem:0<ab<1,lem:da:dl} as $\lambda^\star>0$ and since $\labelnoise>0$.

From~\cref{eq:10-1} it follows that there exists $\Lambda$  such that for all $\lambda\geq\Lambda$  we have $\tfrac{d}{d\lambda}J(t,\lambda)>0$. Moreover, if $\ratio<1$ then 
$\left.\tfrac{d}{d\lambda}J(t,\lambda)\right|_{\lambda=0}<0$, 
while if $\ratio=1$ there exists $\lambda'>0$ such that 
$\left.\tfrac{d}{d\lambda}J(t,\lambda)\right|_{\lambda=\lambda'}<0$.
from \cref{eq:11-1}, the Intermediate Value Theorem guarantees a point $\lambda^{\star}\in\left(0,\Lambda\right]$ such that $\frac{\dd}{\dd\lambda^{\star}}J\left(t,\lambda^{\star}\right)=0$. On the compact interval $I=[0,\Lambda]$ if $\ratio<1$, and $I=[\bar{\lambda},\Lambda]$ if $\ratio=1$, the function $J(t,\cdot)$ attains a global minimum by Weierstrass, since the derivative is negative at the left end-point and is positive at $\Lambda$, the global minimizer is a stationary point, by \cref{eq:9-1} the global minimum is decreasing with respect to $t$.
\end{proof}

\newpage

\subsubsection{Proving \cref{thm:more_tasks} in the case of multiple i.i.d.~teachers}

\begin{proof}

Recall~\cref{eq:finite:t}, for all $T>1$ we have, 
\begin{align*}
    &\mathbb E[G_T]=\left(a^{T}-2a^{T-1}b\right)\trace\left(\boldsymbol{\Sigma}\right)+a^{T}\left\Vert \vw^{\star}\right\Vert ^{2}+2\trace\left(\boldsymbol{\Sigma}\right)\left(\frac{a\left(1-a^{T-1}\right)}{1-a}-b\frac{a\left(1-a^{T-2}\right)}{1-a}\right)\\&\qquad+\labelnoise\frac{1}{2\featnoise}\left(\frac{\featnoise\left(1+\ratio\right)+\lambda}{\sqrt{\lambda^{2}+2\lambda\featnoise\left(1+\ratio\right)+\featnoise^{2}\left(1-\ratio\right)^{2}}}-1\right)\frac{1-a^{T}}{1-a}\\&\qquad+\frac{2}{T}\left(2b^{T}+T-1-Tb\right)\trace\left(\boldsymbol{\Sigma}\right).&
\end{align*}

Define $g(t)=\mathbb{E}\left[G_t\right]$, by the assumptions $\trace\left(\boldsymbol{\Sigma}\right)>0$.
Then,
\begin{align*}
g'\left(t\right)
&
=
\trace\left(\boldsymbol{\Sigma}\right)\left(a^{t}\ln a-2ba^{t-1}\ln a\right)+\left\Vert \vw^{\star}\right\Vert ^{2}a^{t}\ln a+
\\
&\quad+\frac{2\trace\left(\boldsymbol{\Sigma}\right)\left(-a^{t}\ln a+ba^{t-1}\ln a\right)}{1-a}+\frac{4\trace\left(\boldsymbol{\Sigma}\right)}{t}b^{t}\left(\ln\left(b\right)-\frac{1}{t}\right)+\frac{2}{t^{2}}\trace\left(\boldsymbol{\Sigma}\right)-\frac{c\labelnoise a^{t}\ln a}{1-a}
\\
&=a^{t}\biggr[\trace\left(\boldsymbol{\Sigma}\right)\left(\ln a-\frac{2b}{a}\ln a\right)+\left\Vert \vw^{\star}\right\Vert ^{2}\ln a+\frac{2\trace\left(\boldsymbol{\Sigma}\right)\left(-\ln a+\frac{b}{a}\ln a\right)}{1-a}
\\
&\hspace{3em}+\frac{4\trace\left(\boldsymbol{\Sigma}\right)}{a^{t}t}b^{t}\left(\ln\left(b\right)-\frac{1}{t}\right)+\frac{2}{a^{t}t^{2}}\trace\left(\boldsymbol{\Sigma}\right)-\frac{c\labelnoise\ln a}{1-a}\biggl]
\,.
\end{align*}

Defining $A=\trace\left(\boldsymbol{\Sigma}\right)\ln a\left[\left(1-\frac{2b}{a}\right)+\frac{\left\Vert \vw^{\star}\right\Vert ^{2}}{\trace\left(\boldsymbol{\Sigma}\right)}+\frac{2\left(\frac{b}{a}-1\right)}{1-a}\right]-\frac{c\labelnoise \ln a}{1-a}$, we have,
\begin{align*}
\frac{g'\left(t\right)}{a^{t}}=A+\frac{4\trace\left(\boldsymbol{\Sigma}\right)}{t}\left(\frac{b}{a}\right)^{t}\left(\ln\left(b\right)-\frac{1}{t}+\frac{1}{2tb^{t}}\right)\,.
\end{align*}

There exists $t'$ such that for all $t\geq t'$ we have $\ln\left(b\right)-\frac{1}{t}+\frac{1}{2tb^{t}}>0$ since $0<b<1$ by~\cref{lem:0<ab<1}. Since $b>a$ by~\cref{lem:b>a} for all $\lambda>0$, then there exists $t^{\star}>t'$ which for all $t\geq t^{\star}$ we have $\underbrace{A}_{\text{constant in t}}+\frac{4\trace\left(\boldsymbol{\Sigma}\right)}{t}\left(\frac{b}{a}\right)^{t}
\Bigprn{\underbrace{\ln\left(b\right)-\frac{1}{t}+\frac{1}{2tb^{t}}}_{>0}}>0$ then $g'\left(t\right)>0$ implying that $g\left(t\right)$ is increasing for all $t\geq t^{\star}$, thus $\mathbb E[G_t]$ is increasing for all $t\geq \max(2,t^{\star})$. 
\end{proof}

\newpage
Before proving \cref{thm:optimal}, we introduce several auxiliary lemmas.
\begin{lemma}
\label{lem:laurEgt}
$\frac{\dd}{\dd \lambda}\lim\limits_{T\to\infty}\mathbb E\left[G_T\right]=-\frac{1}{2v_x}\left(\frac{v_z}{v_x}+\mathrm{Tr}(\boldsymbol \Sigma)\left(\ratio+1\right)\right) \tilde\lambda^{-2}+O\left(\tilde\lambda^{-3}\right)$.
\end{lemma}

\begin{proof}
In the proof of \cref{thm:asymptotic_mse}, we have shown that $\frac{\dd}{\dd\lambda}\lim\limits_{T\to\infty}\mathbb E[G_t]$ decomposes into two terms:
\begin{align}
&\frac{\dd}{\dd\lambda}\lim\limits_{T\to\infty}\mathbb E[G_t]=\frac{v_{z}}{v_{x}^{2}}\frac{\ratio}{(1-a)^{2}D^{3}}\Big(D-(1+\ratio+\tilde{\lambda})\Big)+\mathrm{Tr}\!\left(\boldsymbol{\Sigma}\right)\frac{-N}{2v_{x}}
\frac{\left(1+\ratio+\tilde{\lambda}-D\right)^{2}}{\left(1-a\right)^{2}D^{3}}\,.\label{eq:dinfdl}
\end{align}
From \cref{lem:laur3}, we know that
$D=\tilde{\lambda}+\ratio+1-\frac{2\ratio}{\tilde{\lambda}+\ratio+1}+O\left(\tilde{\lambda}^{-3}\right)$.
Thus,
$$
1+\ratio+\tilde{\lambda}-D=\frac{1}{\tilde{\lambda}}\frac{2\ratio}{1+\frac{\ratio+1}{\tilde{\lambda}}}+O\left(\tilde{\lambda}^{-3}\right)=\frac{2\ratio}{\tilde{\lambda}}+O\left(\tilde{\lambda}^{-2}\right).
$$
Expanding them term by term, using \cref{lem:Laurent-expansion-around}, we begin with the first term in \cref{eq:dinfdl},
\begin{align*}
\frac{v_{z}}{v_{x}^{2}}\frac{\ratio}{(1-a)^{2}D^{3}}\Big(D-(1+\ratio+\tilde{\lambda})\Big)
&=
-\frac{v_{z}}{v_{x}^{2}}
\left(\frac{\ratio\,\tilde{\lambda}^{2}}{4\ratio^{2}}+O(\tilde{\lambda})\right)
\left(\tilde{\lambda}^{-3}+O(\tilde{\lambda}^{-4})\right)
\left(\frac{2\ratio}{\tilde{\lambda}}+O(\tilde{\lambda}^{-2})\right)
\\
&=
-\frac{v_{z}}{2v_{x}^{2}}\,\tilde{\lambda}^{-2}
+O(\tilde{\lambda}^{-3}),
\end{align*}
and the second term in \cref{eq:dinfdl} becomes,
\begin{align*}
&\mathrm{Tr}\!\left(\boldsymbol{\Sigma}\right)\frac{-N}{2v_{x}}
\left(\frac{4\ratio^{2}}{\tilde{\lambda}^{2}}+O(\tilde{\lambda}^{-3})\right)
\left(\frac{\tilde{\lambda}^{2}}{4\ratio^{2}}+O(\tilde{\lambda})\right)
\left(\tilde{\lambda}^{-3}+O(\tilde{\lambda}^{-4})\right)
\\
&=\mathrm{Tr}\!\left(\boldsymbol{\Sigma}\right)\frac{-N}{2v_{x}}
\left(\tilde{\lambda}^{-3}+O(\tilde{\lambda}^{-4})\right).
\end{align*}
Since $N\triangleq (1+\ratio)\tilde{\lambda}+(1-\ratio)^2$,
\begin{align*}
&\mathrm{Tr}\!\left(\boldsymbol{\Sigma}\right)\frac{-1}{2v_{x}}
\Big((1+\ratio)\tilde{\lambda}+(1-\ratio)^2\Big)
\left(\tilde{\lambda}^{-3}+O(\tilde{\lambda}^{-4})\right)
=
-\mathrm{Tr}\!\left(\boldsymbol{\Sigma}\right)\frac{1+\ratio}{2v_{x}}\;\tilde{\lambda}^{-2}
+O(\tilde{\lambda}^{-3}).
\end{align*}
Overall, we showed that the second term in \cref{eq:dinfdl} is,
$$
\mathrm{Tr}\!\left(\boldsymbol{\Sigma}\right)\frac{-N}{2v_{x}}
\frac{\left(1+\ratio+\tilde{\lambda}-D\right)^{2}}{\left(1-a\right)^{2}D^{3}}
=
-\mathrm{Tr}\!\left(\boldsymbol{\Sigma}\right)\frac{1+\ratio}{2v_{x}}\;\tilde{\lambda}^{-2}
+O\left(\tilde{\lambda}^{-3}\right)\,.
$$
\end{proof}

\newpage
Additionally, we prove a slightly weaker Lemma than \cref{thm:optimal}. This establishes a nearly linear scaling of $\lambda^{\star}(T)$, which will help us derive the main result in \cref{thm:optimal}.

\subsection{Auxiliary Lemma: Optimal Regularization is Near Linear in $T$}
\begin{lemma}[Optimal regularization strength is near linear in $T$]
\label{thm:opt_finite}
Under the i.i.d.-teacher setting (\cref{assump:iid_teachers}) with non-zero mean teacher, if either $\mathrm{Tr}(\boldsymbol\Sigma)>0$ or $\labelnoise>0$ the optimal fixed 
regularization strength $\lambda^\star$ that minimizes the expected generalization 
loss after $T$ iterations satisfies
$
C_1 T^{\rho} < \lambda^\star < C_0 T,
$
for all sufficiently large $T$ and every $\rho<1$, where $C_0, C_1 > 0$ are constants depending only on the problem parameters.
\end{lemma}

We first provide a brief proof outline.
\paragraph{Proof outline.}
Set $\tilde\lambda\triangleq\lambda/\featnoise$ and fix a cutoff $\Lambda(T)>0$; we analyze $\tilde\lambda\in [0,\Lambda(T)]$.
\linebreak
Let $D=\sqrt{\tilde\lambda^2+2(1+\ratio)\tilde\lambda+(1-\ratio)^2}$. We show that $a,b\le \zeta\triangleq1-\frac{\ratio}{1+\ratio+\Lambda(T)}$.
\begin{enumerate}[leftmargin=0.7cm, labelsep=0.2cm, itemsep=3pt]
\item  \textbf{Limit objective is strictly decreasing.}
By the closed form of the infinite-horizon loss,
\[
\lim_{T'\to\infty}\mathbb E[G_{T'}]
=2\trace(\mathbf \Sigma)\Big(\tfrac{1-b}{1-a}\Big)
+\labelnoise\frac{1}{2\featnoise}\Big(\frac{1+\ratio+\tilde\lambda}{D}-1\Big)\frac{1}{1-a}.
\]
From~\remref{rem:1} we have
\[
\frac{\dd}{\dd\tilde\lambda}\,\lim_{T'\to\infty}\mathbb E[G_{T'}]\;<\;0\qquad\forall\,\tilde\lambda>0.
\]

\item \textbf{In the sublinear regime, the derivative is also negative.}
Using the decomposition (algebra is provided in the full proof),
\begin{align}
\label{eq:differen}   
\mathbb E[G_T]\!-\!\!\lim_{T'\to\infty}\!\!\mathbb E[G_{T'}]
= a^T\Big(\|\teacher\|^2+\trace \mathbf(\boldsymbol\Sigma)\!-\!2\trace (\boldsymbol\Sigma)\tfrac{1-b}{1-a}-\tfrac{f}{1-a}\Big)
\!
+\tfrac{2\trace \mathbf(\boldsymbol\Sigma)}{T}(2b^T\!-1),
\end{align}

Differentiating and using $a,b\le\zeta$, boundedness of
$\frac{1-b}{1-a}$ and its derivative, and boundedness of $f$ and its derivative (see details in the proof), gives the uniform bound
\[
\sup_{\tilde\lambda\in [0,\Lambda(T)]}
\Big|\tfrac{\dd}{\dd\tilde\lambda}\big(\mathbb E[G_T]-\lim_{T'\to\infty}\mathbb E[G_{T'}]\big)\Big|
\;\le\; C\,T\,\zeta^{\,T-5}\,(1-a)^{-2}\,.
\]
Since $(1-a)^{-2}=O(\Lambda(T)^2)$ and $T\zeta^{\,T-5}\!=\!T\exp(-\Theta(T/\Lambda(T)))$,
if $\Lambda(T)\in o(T^\rho)$ for any fixed $0<\rho<1$, the RHS decay exponentially.
Meanwhile,
\[
\inf_{\tilde\lambda\in [0,\Lambda(T)]}
\!\Big(-\tfrac{\dd}{\dd\tilde\lambda}\lim_{T'\to\infty}\mathbb E[G_{T'}]\Big)
=\Omega(\Lambda(T)^{-2})=\omega(T^{-2\rho})\,.
\]
Hence, for large $T$, the derivative stays negative in $[0,\Lambda(T)]$.
Therefore, $\tilde\lambda^\star\notin o(T^\rho)$ for every $0<\rho<1$.

\item \textbf{In superlinear regimes, the derivative turns positive.}
For $\tilde\lambda\!\in\!\omega(T)$, Laurent expansions~give
\begin{align*}
a^T&=1-\frac{2\ratio T}{\tilde\lambda}+O(T^2\tilde\lambda^{-2})\,,
\quad
b^T=1-\frac{\ratio T}{\tilde\lambda}+O(T^2\tilde\lambda^{-2})\,,
\quad
\frac{1-b}{1-a}=\frac12+O(\tilde\lambda^{-1})\,,
\\
c&=O(\tilde\lambda^{-2})\,, \quad(1-a)^{-1}=O(\tilde\lambda)\,.
\end{align*}
Plugging into the expression above (\cref{eq:differen}),
\begin{align*}
\mathbb E[G_T]-\lim_{T'\to\infty}\mathbb E[G_{T'}]
&= \left\Vert \teacher \right\Vert ^{2}-\frac{2\ratio T}{\tilde\lambda}\left\Vert \teacher \right\Vert ^{2}+\frac{2\trace\left (\boldsymbol \Sigma\right)}{T}
+O(T^2\tilde\lambda^{-2})+O(\tilde\lambda^{-1})\,,
\\
\frac{\dd}{\dd\tilde\lambda}\mathbb E[G_T]
&=\frac{2\ratio T}{\tilde\lambda^2}
+O(T^2\tilde\lambda^{-3})+O(\tilde\lambda^{-2})
+\frac{\dd}{\dd\tilde\lambda}\lim_{T'\to\infty}\mathbb E[G_{T'}]\;>\;0,
\end{align*}
and $\frac{\dd}{\dd\tilde\lambda}\lim_{T'\to\infty}\mathbb E[G_{T'}]=O(\tilde \lambda^{-2})$. Hence, no optimizer lies in $\omega(T)$.

\item \textbf{Conclusion.}
Combining (2) + (3), any minimizer satisfies 
$\tilde \lambda^\star \in O(T)$
and
$\tilde \lambda^\star \in \Omega(T^\rho)$ for all $0<\rho<1$, since for $\tilde \lambda=\frac{\lambda}{\featnoise}$ the same asymptotic bounds hold for $\lambda^\star$.
\end{enumerate}

\vspace{-0.5em}
\jmlrQED
\vspace{-0.5em}

\begin{proof}[for \cref{thm:opt_finite}]
For the actual proof, we first analyze the asymptotic regime $T \to \infty$.
\tparagraph{In the sublinear regime, the derivative is negative.} Let $\Lambda(T)$ be a positive function of $T$, whose form will be specified
later. We focus on regularization strengths satisfying
$\lambda \in [0,\Lambda(T)]$, and define the corresponding scaled parameter
$\tilde{\lambda} \triangleq \lambda / \featnoise$.

Let $\bar{D}=\sup_{[0,\Lambda(T)]}D$; $\underline{D}=\inf_{[0,\Lambda(T)]}D$; and $\zeta=1-\frac{\ratio}{1+\ratio+\Lambda\left(T\right)}$.
Note that $a,b\leq\zeta$ by~\cref{lem:3}.
Recall~\remref{rem:1} which explained that $\frac{\dd}{\dd\lambda}\lim\limits_{T\to\infty}\mathbb{E}\left[G_{T}\right]<0,~\forall\lambda>0$.
Clearly, the same holds for the scaled parameter $\tilde{\lambda}=\frac{\lambda}{\featnoise}$, i.e., $\frac{\dd}{\dd\tilde{\lambda}}\lim\limits_{T\to\infty}\mathbb{E}\left[G_{T}\right]<0,~\forall\tilde{\lambda}>0$.

We now wish to use this fact. If the difference between $\frac{\dd}{\dd\tilde{\lambda}}\lim_{T'\to\infty}\mathbb{E}\left[G_{T'}\right]$ and $\frac{\dd}{\dd\tilde{\lambda}}\mathbb{E}\left[G_{T}\right]$ is sufficiently small, then we can ensure that $\frac{\dd}{\dd\tilde{\lambda}}\mathbb{E}\left[G_{T}\right]<0$.

Recall $f\triangleq\labelnoise\frac{1}{2\featnoise}\left(\frac{\left(1+\ratio\right)+\tilde{\lambda}}{\sqrt{\tilde{\lambda}^{2}+2\tilde{\lambda}\left(1+\ratio\right)+\left(1-\ratio\right)^{2}}}-1\right)$ then by~\cref{thm:infinite_budget,eq:finite:t}, we have that

\begin{align}
&\mathbb{E}\left[G_{T}\right]-\lim_{T'\to\infty}\mathbb{E}\left[G_{T'}\right]=\notag&
\\
&
=\underbrace{\left(a^{T}-2a^{T-1}b\right)\trace\left(\boldsymbol{\Sigma}\right)+a^{T}\left\Vert\teacher \right\Vert ^{2}+2\trace\left(\boldsymbol{\Sigma}\right)\left(\frac{a\left(1-a^{T-1}\right)}{1-a}-b\cdot\frac{a\left(1-a^{T-2}\right)}{1-a}\right)}_{\text{first part of }\mathbb{E}\left[G_{T}\right]}
\notag
\\
&\hspace{1em}\underbrace{+f\frac{1-a^{T}}{1-a}+\frac{2}{T}\left(2b^{T}+T-1-Tb\right)\trace\left(\boldsymbol{\Sigma}\right)}_{\text{other part of }\mathbb{E}\left[G_{T}\right]}
-\underbrace{\left(2\trace\left(\boldsymbol{\Sigma}\right)\left(\frac{a-ba}{1-a}-b+1\right)+f\frac{1}{1-a}\right)}_{\lim_{T'\to\infty}\mathbb{E}\left[G_{T'}\right]}\notag
\\
&=\left(a^{T}-2a^{T-1}b\right)\trace\left(\boldsymbol{\Sigma}\right)+a^{T}\left\Vert \teacher \right\Vert ^{2}+2\trace\left(\boldsymbol{\Sigma}\right)\left(\frac{-a^{T}}{1-a}+b\cdot\frac{a^{T-1}}{1-a}\right)\notag\\&\hspace{4em}+f\frac{-a^{T}}{1-a}+\frac{2}{T}\left(2b^{T}-1\right)\trace\left(\boldsymbol{\Sigma}\right)\notag\\&=a^{T}\left(\left\Vert \teacher \right\Vert ^{2}+\trace\left(\boldsymbol{\Sigma}\right)\right)+f\frac{-a^{T}}{1-a}-2\trace\left(\boldsymbol{\Sigma}\right)a^{T}\left(\frac{1-b}{1-a}\right)+\frac{2\trace\left(\boldsymbol{\Sigma}\right)}{T}\left(2b^{T}-1\right).&
\label{eq:15}
\end{align}

\pagebreak

Using~\cref{eq:15} taking the derivative according to $\tilde \lambda$
\begin{align}
&\left|\frac{\dd}{\dd\tilde{\lambda}}\left[\lim_{T'\to\infty}\mathbb{E}\left[G_{T'}\right]-\mathbb{E}\left[G_{T}\right]\right]\right|
\\
&=\biggr|Ta^{T-1}\frac{\dd a}{\dd\tilde{\lambda}}+f\frac{-\left(Ta^{T-1}-Ta^{T}+a^{T}\right)}{\left(1-a\right)^{2}}\frac{\dd a}{\dd\tilde{\lambda}}
+\frac{\dd f}{\dd\tilde{\lambda}}\frac{-a^{T}}{1-a}
\notag\\&
\hspace{3em} -2\trace\left(\boldsymbol{\Sigma}\right)\left(Ta^{T-1}\frac{\dd a}{\dd\tilde{\lambda}}\left(\frac{1-b}{1-a}\right)+a^{T}\frac{\dd}{\dd\tilde{\lambda}}\left(\frac{1-b}{1-a}\right)\right)
+\frac{4\trace\left(\boldsymbol{\Sigma}\right)}{T}Tb^{T-1}\frac{\dd b}{\dd\tilde{\lambda}}\biggl|.&
\label{eq:16}
\end{align}
If $\ratio<1$,
we first show that $\frac{1-b}{1-a}$ and $\frac{\dd}{\dd\tilde{\lambda}}\frac{1-b}{1-a}$ are bounded. Consider the Laurent expansion of $\frac{1-b}{1-a}=\frac{1}{2}+O\left(\tilde{\lambda}^{-1}\right)$ by~\cref{lem:laur1}, which is analytic at $\infty$  (being a concatenation, sum, and division of analytic functions). By Lemma \ref{lem:diffrentiable}, we then obtain $\frac{\dd}{\dd\tilde{\lambda}}\frac{1-b}{1-a}=O\left(\tilde{\lambda}^{-2}\right)$. Since both terms are continuous on the region $\tilde{\lambda}\geq0$,~\cref{lem:bounded} implies that both terms are bounded.

Consider the Laurent expansion of $f=O\left(\tilde{\lambda}^{-2}\right)$ from~\cref{lem:laur2}. As above, it is analytic at~$\infty$. By~\cref{lem:diffrentiable}, $\frac{\dd f}{\dd\tilde{\lambda}}=O\left(\tilde{\lambda}^{-3}\right)$. Since both $f,\frac{\dd f}{\dd\tilde{\lambda}}$ are continuous at the region where $\tilde{\lambda}\geq0$ by~\cref{lem:bounded} both terms are bounded.

Otherwise, if $\ratio=1$, it is not clear that $f,\frac{1-b}{1-a},a,b$ and their corresponding derivatives are continuous at the region $\tilde{\lambda}\geq0$, as they may diverge at $\tilde \lambda = 0$.
From~\cref{lem:c=1dc,lem:7,lem:1:a} we know that $a^4\frac{\dd f}{\dd\tilde{\lambda}},a^2f,a^2\frac{\dd a}{\dd\tilde\lambda},b^3\frac{\dd b}{\dd\tilde\lambda},\frac{\dd}{\dd\tilde\lambda}\frac{1-b}{1-a},\frac{1-b}{1-a}$ are bounded at the region $\tilde{\lambda}>0$ in that case.


Combining with~\cref{lem:6} if $\ratio<1$, otherwise if $\ratio=1$ we obtain that each prefactor multiplying the terms involving $a^{T-5}$ or $b^{T-5}$ in~\cref{eq:16} is bounded, so there exist constants $C_{1},C_{2}>0$ such that,
\begin{align*} 
\left|\frac{\dd}{\dd\tilde{\lambda}}\left[\lim_{T'\to\infty}\mathbb{E}\left[G_{T'}\right]-\mathbb{E}\left[G_{T}\right]\right]\right|\leq\frac{C_{1}Ta^{T-5}}{\left(1-a\right)^{2}}+C_{2}Tb^{T-5}\leq\left(C_{1}+C_{2}\right)\frac{T\zeta^{T-5}}{\left(1-a\right)^{2}},
\end{align*}
where we used~\cref{lem:0<ab<1}.

Furthermore,
\begin{align}
\sup_{\lambda\in [0,\Lambda(T)]}\left|\frac{\dd}{\dd\tilde{\lambda}}\left[\lim_{T'\to\infty}\mathbb{E}\left[G_{T'}\right]-\mathbb{E}\left[G_{T}\right]\right]\right|\leq
\left(C_{1}+C_{2}\right)T\zeta^{T-5}O\left(\Lambda(T)^{2}\right),
\label{eq:finder}
\end{align}
where the last inequality follows from~\cref{lem:3}.

\paragraph{Proving $\lambda^{\star}=\omega\left(T^{\rho}\right), \forall \rho\in(0,1)$.}
First, we show that the optimal regularizer $\lambda^{\star}$ is not in $o\left(T^{\rho}\right)$ for all $0<\rho<1$ as stated in the theorem. Next, we ensure that $\frac{\dd}{\dd\lambda}\mathbb{E}\left[G_{T}\right]$ remains negative when $\Lambda\left(T\right)\in o\left(T^{\rho}\right)$ for arbitrary $0<\rho<1$. It suffices to show that the gap between the derivatives~(\cref{eq:finder}) converges to zero uniformly in $\tilde{\lambda}$ faster than $\frac{\dd}{\dd\tilde{\lambda}}\lim_{T'\to\infty}\mathbb{E}\left[G_{T'}\right]$ does.

Pick $\Lambda\left(T\right)\in o\left(T^{\rho}\right)$ for some $\ratio<1$. From \cref{thm:infinite_budget}, we have
\begin{align*}  
&\lim_{T'\to\infty}\mathbb{E}\left[G_{T'}\right]=\underbrace{2\trace\left(\boldsymbol{\Sigma}\right)\left(\frac{1-b}{1-a}\right)}_{\text{term 1}}+\underbrace{\frac{\labelnoise}{2\featnoise}\left(\frac{\featnoise\left(1+\ratio\right)+\lambda}{\sqrt{\lambda^{2}+2\lambda\featnoise\left(1+\ratio\right)+\featnoise^{2}\left(1-\ratio\right)^{2}}}-1\right)\frac{1}{1-a}}_{\text{term 2}}.
\end{align*}
Recall~\cref{thm:asymptotic_mse} where we proved that the derivative of both terms are negative for all $\lambda>0$.

We aim to establish a lower bound for $\inf_{\tilde{\lambda}\in [0,\Lambda(T)]}\left(-\frac{\dd}{\dd\tilde{\lambda}}\lim_{T'\to\infty}\mathbb{E}\left[G_{T'}\right]\right)$, providing a uniform rate of decrease for $\lim_{T'\to\infty}\mathbb{E}\left[G_{T'}\right]$. 
If either $\mathrm{Tr}(\boldsymbol{\Sigma})>0$ or $\labelnoise>0$ then, by \cref{lem:laurEgt}, 
since $\tilde{\lambda} = \frac{\lambda}{\featnoise}
\in [0,\frac{1}{\featnoise}\Lambda(T)]$, we have,
$$\frac{\dd}{\dd \tilde \lambda}\lim_{T'\to\infty}\mathbb E\left[G_{T'}\right]\in \Omega\left(\prn{\Lambda\left(T\right)}^{-2}\right)=\omega(T^{-2\rho})\,.$$
Let
$\eta=\inf_{\tilde{\lambda}\in [0,\Lambda(T)]}\left(-\frac{\dd}{\dd\tilde{\lambda}}\lim\limits_{T'\to\infty}\mathbb{E}\left[G_{T'}\right]\right)
=\omega\left(T^{-2\rho}\right)
$
and $\bar{\lambda}=\inf_{\tilde{\lambda}\in [0,\Lambda(T)]}\left(\frac{\dd}{\dd\tilde{\lambda}}\mathbb{E}\left[G_{T}\right]\right)$. Then, $$\frac{\dd}{\dd\bar{\lambda}}\mathbb{E}\left[G_{T}\right]\leq\frac{\dd}{\dd\bar{\lambda}}\lim_{T'\to\infty}\mathbb{E}\left[G_{T'}\right]+\left|\frac{\dd}{\dd\bar{\lambda}}\left[\lim_{T'\to\infty}\mathbb{E}\left[G_{T'}\right]-\mathbb{E}\left[G_{T}\right]\right]\right|\leq-\eta+\frac{1}{2}\eta\leq-\frac{1}{2}\eta<0,$$
where the second inequality follows from the uniform bound, 
$$\sup_{\tilde \lambda\in [0,\Lambda(T)]}\left|\frac{\dd}{\dd\lambda}\left[\lim_{T'\to\infty}\mathbb{E}\left[G_{T'}\right]-\mathbb{E}\left[G_{T}\right]\right]\right|\leq\frac{1}{2}\eta\,,$$ since the exponential factor 
$$\zeta^{T-5}=\left(1-\frac{\ratio}{\Lambda\left(T\right)+1+\ratio}\right)^{T-5}\leq\exp\left(-\frac{\ratio\left(T-5\right)}{\Lambda\left(T\right)+1+\ratio}\right)$$ decays faster than any polynomial in $T$. In particular, the factor $\frac{\ratio\left(T-5\right)}{\Lambda\left(T\right)+1+\ratio}\in\omega\left(T^{1-\rho}\right)$.

\paragraph{In superlinear regimes, the derivative turns positive.}
Next, we analyze the case of $\tilde{\lambda}\in\omega\left(T\right)$, and later, we show that $\frac{\dd}{\dd\tilde{\lambda}}\mathbb{E}\left[G_{T}\right]>0$ which completes our proof.

Recall~\cref{eq:15},   \begin{align*}
&\mathbb{E}\left[G_{T}\right]-\lim\limits_{T'\to\infty}\mathbb{E}\left[G_{T'}\right]
\\
&=\underbrace{a^{T}\left(\left\Vert \teacher\right\Vert ^{2}+\trace\left(\boldsymbol{\Sigma}\right)\right)+\frac{2\trace\left(\boldsymbol{\Sigma}\right)}{T}\left(2b^{T}-1\right)}_{\text{term 1}}
-\underbrace{2\trace\left(\boldsymbol{\Sigma}\right)a^{T}\left(\frac{1-b}{1-a}\right)}_{\text{term 2}}+\underbrace{f\frac{-a^{T}}{1-a}}_{\text{term 3}}.\end{align*}
Using~\cref{lem:Laurent-expansion-around} we adopt a Laurent expansion for the expression $\mathbb{E}\left[G_{T}\right]-\lim\limits_{T'\to\infty}\mathbb{E}\left[G_{T'}\right]$, term by term, with this expression we intend to show that $\frac{\dd}{\dd\tilde \lambda}\mathbb{E}\left[G_{T}\right]>0$ when $\tilde \lambda = \omega(T)$. Here, we analyze the asymptotic case where $T\to \infty$ and $\tilde \lambda = \omega(T)$, and thus, the $O(\cdot)$ notations below hide all bounded variables, but not $T, \tilde \lambda$.
Term 1 becomes,
\begin{align*}
&a^{T}\left(\left\Vert\teacher \right\Vert ^{2}+\trace\left(\boldsymbol{\Sigma}\right)\right)+\frac{2\trace\left(\boldsymbol{\Sigma}\right)}{T}\left(2b^{T}-1\right)
\\
&=\left(1-\frac{2{\ratio}T}{\tilde{\lambda}}+O\left(T^{2}\tilde{\lambda}^{-2}\right)\right)\left(\left\Vert\teacher \right\Vert ^{2}+\trace\left(\boldsymbol{\Sigma}\right)\right)+\frac{2\trace\left(\boldsymbol{\Sigma}\right)}{T}\left(1-\frac{2{\ratio}T}{\tilde{\lambda}}+O\left(T^{2}\tilde{\lambda}^{-2}\right)\right)
\\
&=\left\Vert\teacher \right\Vert ^{2}+\trace\left(\boldsymbol{\Sigma}\right)-\frac{2{\ratio}T}{\tilde{\lambda}}\left(\left\Vert\teacher \right\Vert ^{2}+\trace\left(\boldsymbol{\Sigma}\right)\right)+\frac{2\trace\left(\boldsymbol{\Sigma}\right)}{T}+O\left(T^{2}\tilde{\lambda}^{-2}\right)+O\left(\tilde{\lambda}^{-1}\right)\,.
\end{align*}

\bigskip

Term 2 is,
\begin{align*}2\trace\left(\boldsymbol{\Sigma}\right)a^{T}\left(\frac{1-b}{1-a}\right)
&=
2\trace\left(\boldsymbol{\Sigma}\right)\left(1-\frac{2{\ratio}T}{\tilde{\lambda}}+O\left(T^{2}\tilde{\lambda}^{-2}\right)\right)\frac{\frac{\ratio}{\tilde{\lambda}}+O\left(\tilde{\lambda}^{-2}\right)}{\frac{2{\ratio}}{\tilde{\lambda}}+O\left(\tilde{\lambda}^{-2}\right)}
\\
\explain{\text{\cref{lem:laur1}}}
&
=2\trace\left(\boldsymbol{\Sigma}\right)\left(1-\frac{2{\ratio}T}{\tilde{\lambda}}+O\left(T^{2}\tilde{\lambda}^{-2}\right)\right)\left(\frac{1}{2}+O\left(\tilde{\lambda}^{-1}\right)\right)
\\&
=2\trace\left(\boldsymbol{\Sigma}\right)\left(\frac{1}{2}-\frac{{\ratio}T}{\tilde{\lambda}}+O\left(T^{2}\tilde{\lambda}^{-2}\right)+O\left(\tilde{\lambda}^{-1}\right)\right)
\\&
=\trace\left(\boldsymbol{\Sigma}\right)-2\trace\left(\boldsymbol{\Sigma}\right)\frac{{\ratio}T}{\tilde{\lambda}}+O\left(T^{2}\tilde{\lambda}^{-2}\right)+O\left(\tilde{\lambda}^{-1}\right).
\end{align*}

Finally, Term 3 is,
\begin{align*}
&\frac{\labelnoise}{2\featnoise}\left(\frac{\tilde{\lambda}+\ratio+1}{\sqrt{\left(\tilde{\lambda}+\ratio+1\right)^{2}-4\ratio}}-1\right)\frac{-a^{T}}{1-a}\\&=\frac{\labelnoise}{2\featnoise}\left(\frac{1}{\sqrt{1-\frac{4\ratio}{\left(\tilde{\lambda}+\ratio+1\right)^{2}}}}-1\right)\frac{-a^{T}}{1-a}
\\
\explain{\left(1-u\right)^{-\frac{1}{2}}=1+\frac{1}{2}u+O\left(u^{2}\right)}
&=\labelnoise\frac{1}{2\featnoise}\left(1+\frac{2{\ratio}}{\left(\tilde{\lambda}+\ratio+1\right)^{2}}+O\left(\tilde{\lambda}^{-4}\right)-1\right)\frac{-a^{T}}{1-a}
\\
&=O\left(\tilde{\lambda}^{-2}\right)\left(\frac{-1+\frac{2{\ratio}T}{\tilde{\lambda}}+O\left(T^{2}\tilde{\lambda}^{-2}\right)}{\frac{2{\ratio}}{\tilde{\lambda}}+O\left(\tilde{\lambda}^{-2}\right)}\right)
\\&=
O\left(\tilde{\lambda}^{-2}\right)\left(-\frac{\tilde{\lambda}}{2{\ratio}}+T+O\left(T^{2}\tilde{\lambda}^{-1}\right)\right)
=O\left(\tilde{\lambda}^{-1}\right). 
\end{align*}
Overall, when summing, $$\mathbb{E}\left[G_{T}\right]-\lim\limits_{T'\to\infty}\mathbb{E}\left[G_{T'}\right]=\left\Vert\teacher \right\Vert ^{2}-\frac{2{\ratio}T}{\tilde{\lambda}}\left\Vert\teacher \right\Vert ^{2}+\frac{2\trace\left(\boldsymbol{\Sigma}\right)}{T}+O\left(T^{2}\tilde{\lambda}^{-2}\right)+O\left(\tilde{\lambda}^{-1}\right).$$

\pagebreak

By \cref{lem:diffrentiable}, as of the terms above is analytic,\footnote{An analytic function is a function locally given by a convergent power series.} and analyticity is preserved under concatenation, summing and multiplication and division,
\begin{align*}
&\frac{\dd}{\dd\tilde{\lambda}}\left[\mathbb{E}\left[G_{T}\right]-\lim\limits_{T'\to\infty}\mathbb{E}\left[G_{T'}\right]\right]	
=\frac{2{\ratio}T}{\tilde{\lambda}^{2}}\left\Vert\teacher \right\Vert ^{2}+O\left(T^{2}\tilde{\lambda}^{-3}\right)+O\left(\tilde{\lambda}^{-2}\right)\stackrel{[1]}{>}0
\\
&
\frac{\dd}{\dd\tilde{\lambda}}\mathbb{E}\left[G_{T}\right]
=\frac{2{\ratio}T}{\tilde{\lambda}^{2}}\left\Vert\teacher \right\Vert ^{2}+O\left(T^{2}\tilde{\lambda}^{-3}\right)+O\left(\tilde{\lambda}^{-2}\right)+\frac{\dd}{\dd\tilde{\lambda}}\lim\limits_{T\to\infty}\mathbb{E}\left[G_{T}\right]\stackrel{[1]+[2]}{>}0
\,,
\end{align*}
where [1] follows from $\lim\limits_{T\to\infty}\frac{2{\ratio}T\tilde{\lambda}^{-2}}{T^{2}\tilde{\lambda}^{-3}}=\lim\limits_{T\to\infty}\frac{2{\ratio}\tilde{\lambda}}{T}
\stackrel{\tilde{\lambda}=\omega\left(T\right)}{=}\infty$
and $\lim\limits_{T\to\infty}\frac{2{\ratio}T\tilde{\lambda}^{-2}}{\tilde{\lambda}^{-2}}=\infty$  implying the sign is determined by $\frac{2{\ratio}T}{\tilde{\lambda}^{2}}>0$ as $\teacher\neq \0$;
and [2] follows since $\frac{\dd}{\dd\tilde{\lambda}}\lim\limits_{T'\to\infty}\mathbb{E}\left[G_{T'}\right]=O\left(\tilde{\lambda}^{-2}\right)$ 
by \cref{lem:laurEgt}.

\paragraph{Conclusion.} Since for all $\tilde{\lambda}\in o(T^{\rho})$ 
(equivalently, $\lambda\in o(T^{\rho})$) with $0<\rho<1$, 
we have 
$
\frac{\dd}{\dd\tilde{\lambda}}\mathbb{E}[G_{T}] < 0,
$
and for all $\tilde{\lambda}\in \omega(T)$ 
(equivalently, $\lambda\in \omega(T)$), 
$
\frac{\dd}{\dd\tilde{\lambda}}\mathbb{E}[G_{T}] > 0,
$
it follows that any optimal regularizer $\lambda^{\star}$ must satisfy 
$
\lambda^{\star}\in O(T) 
$
and
$
\lambda^{\star}\in \Omega(T^{\rho})
$ for all $0<\rho<1$.
\end{proof}
\newpage
\subsection{Proof of~\cref{thm:optimal}: Key Result on Optimal Regularization Scaling}
\begin{recall}[\cref{thm:optimal}]
Under i.i.d.\ teachers (\cref{assump:iid_teachers}) with non-zero mean teacher $\teacher$, the optimal fixed regularization strength that minimizes the expected generalization 
loss after $T$ iterations satisfies 
$\lambda^\star
=
\lambda^\star(T)\asymp\frac{T}{\ln T}$.
More precisely, for any $\epsilon>0$, there exists
$T_0 = O\!\left(\hfrac{1}{\epsilon^2}\right)$ such that for all
$T \ge T_0$,
\[(1-\epsilon)
\frac{2\featnoise \alpha T}{
\ln
\Bigprn{
    \frac{4\alpha T\left\Vert \vw^{\star}\right\Vert ^{2}\featnoise
    }
    {{\labelnoise+
    \featnoise\mathrm{Tr}(\boldsymbol{\Sigma})
    \left(1+\ratio\right)}
    }
}
}
<
\lambda^\star
<
(1+\epsilon)
\frac{2\featnoise \alpha T}{
\ln
\Bigprn{
    \frac{4\alpha T\left\Vert \vw^{\star}\right\Vert ^{2}\featnoise
    }
    {\labelnoise+\featnoise\mathrm{Tr}(\boldsymbol{\Sigma})
    \left(1+\ratio\right)}
    }
}
\,.
\]
Furthermore, in the degenerate noiseless case of $\mathrm{Tr}(\boldsymbol\Sigma)=\labelnoise=0$, it holds that $\lambda^\star\to 0$.
\end{recall}
\begin{proof}
If both $\mathrm{Tr}(\boldsymbol\Sigma)=0$ and $\labelnoise=0$, then by~\cref{eq:singleteacher},
$$
\mathbb{E}[G_T]
= \mathbb{E}\bigl\|\vw_T - \vw^\star\bigr\|^{2}
= a^{T}\,\bigl\|\vw_{0}-\vw^{\star}\bigr\|^{2}.
$$
As a function of $\lambda$, the derivative of $a$ is positive for all $\lambda>0$ (\cref{lem:da:dl}). 
Therefore, the optimal regularizer is $\lambda^\star \to 0$.

Otherwise, let $\lambda = \psi(T)\,T$, where $\psi(T) \in \omega\left(T^{-1/4}\right)$ and
$\psi(T) \in O(1)$. By \cref{thm:opt_finite}, this entire region contains the
optimal regularizer $\lambda^\star$ for all sufficiently large $T$.
Recall \cref{eq:15},  
$$\mathbb{E}\left[G_{T}\right]-\!\lim\limits_{T\to\infty}\!\mathbb{E}\left[G_{T}\right]=\underbrace{\frac{2\mathrm{Tr}\left({\mSigma}\right)}{T}\left(2b^{T}-1\right)}_{\text{term 1}}+\underbrace{\mathrm{Tr}\left({\mSigma}\right)a^T\left(1-2\frac{1-b}{1-a}\right)}_{\text{term 2}}+\underbrace{a^{T}\left\Vert \vw^{\star}\right\Vert ^{2}+f\frac{-a^{T}}{1-a}}_{\text{term 3}},$$
Using \cref{lem:Laurent-expansion-around}, we begin with developing an equivalent Laurent expansion of each one of the terms. Term 1 becomes,
\begin{align}
\frac{2\,\mathrm{Tr}(\boldsymbol{\mSigma})}{T}\left(2b^{T}-1\right)
&=
\frac{2\,\mathrm{Tr}(\boldsymbol{\mSigma})}{T}\left(
2\exp\left(-\frac{\alpha T}{\tilde{\lambda}}\right)
\left(1+O(T\tilde{\lambda}^{-2})\right)
-1\right)\notag\\
&=
\frac{2\,\mathrm{Tr}(\boldsymbol{\mSigma})}{T}\left(
2\exp\left(-\frac{\alpha v_x}{\psi(T)}\right)
\left(1+O\left(T^{-1}\psi(T)^{-2}\right)\right)
-1
\right).
\label{eq:/t}
\end{align}
Term 2 is,
\begin{align}
&\mathrm{Tr}\left(\boldsymbol{\mSigma}\right)a^{T}\left[1-2\left(\frac{1-b}{1-a}\right)\right]\notag \\&=\mathrm{Tr}\left(\boldsymbol{\mSigma}\right)\left(\exp\left(-\frac{2\alpha v_x}{\psi(T)}\right)\left(1+O\left(\psi(T)^{-2}T^{-1}\right)\right)\right)O\left(\psi(T)^{-1}T^{-1}\right)
\label{eq:at}.
\end{align}
Term 3 is,
\begin{align}
&a^{T}\|\vw^{\star}\|^{2}+f\frac{-a^{T}}{1-a}
\notag
\\
&
=
\exp\left(-\frac{2\alpha v_{x}}{\psi(T)}+O\left(\psi(T)^{-2}T^{-1}\right)\right)\|\vw^{\star}\|^{2}\notag
\\
&\hspace{1em}
-\tfrac{\psi(T)Tv_{z}}{4\alpha v_{x}^{2}}\!\left(\frac{2\alpha}{\left(\frac{\psi(T)T}{v_{x}}+\alpha+1\right)^{2}}+O\left((\psi(T)T)^{-4}\right)\right)\!\exp\left(-\tfrac{2\alpha v_{x}}{\psi(T)}+O\left(\psi(T)^{-2}T^{-1}\right)\right)\notag
\\
&=\exp\left(-\tfrac{2\alpha v_{x}}{\psi(T)}\right)\left(1+O\left(\psi(T)^{-2}T^{-1}\right)\right)\|\vw^{\star}\|^{2}
\notag
\\
&\hspace{1em}
-\tfrac{v_{z}}{2}\left(\frac{1}{\psi(T)T}+O\left((\psi(T)T)^{-2}\right)\right)\left(\exp\left(-\tfrac{2\alpha v_{x}}{\psi(T)}\right)\left(1+O\left(\psi(T)^{-2}T^{-1}\right)\right)\right).\label{eq:st}
\end{align}

Using \cref{lem:diffrentiable}, we differentiate $\frac{\dd}{\dd \psi(T)}\left[\mathbb E[G_T]-\lim\limits_{T\to\infty} \mathbb E [G_T]\right]$, term by term.

We begin with \cref{eq:/t},
\begin{align}
&\frac{\dd}{\dd \psi\left(T\right)}\frac{2\,\mathrm{Tr}(\boldsymbol{\mSigma})}{T}\left(2b^{T}-1\right)=\frac{2\,\mathrm{Tr}(\boldsymbol{\mSigma})}{T}\left(\frac{2\alpha v_x}{\psi(T)^{2}}\exp\left(-\frac{\alpha v_x}{\psi(T)}\right)\left(1+O\left(\psi(T)^{-2}T^{-1}\right)\right)\right)+\notag\\&\frac{2\,\mathrm{Tr}(\boldsymbol{\mSigma})}{T}\left(2\exp\left(-\frac{\alpha v_x}{\psi(T)}\right)O\left(\psi(T)^{-3}T^{-1}\right)\right).\label{eq:dtdz}
\end{align}
Differentiating the expression in \cref{eq:at}, we get,
\begin{align}
&\frac{\dd}{\dd \psi(T)}\mathrm{Tr}\left(\boldsymbol{\mSigma}\right)a^{T}\left[1-2\left(\frac{1-b}{1-a}\right)\right]\notag\\&=\mathrm{Tr}\left(\boldsymbol{\mSigma}\right)\biggl(\frac{2\alpha v_x}{\psi(T)^{2}}\exp\left(-\frac{2\alpha v_x}{\psi(T)}\right)\left(1+O\left(\psi(T)^{-2}T^{-1}\right)\right)\notag\\&\hspace{5em}+\exp\left(-\frac{2\alpha v_x}{\psi(T)}\right)O\left(\psi(T)^{-3}T^{-1}\right)\biggr)O\left(\psi(T)^{-1}T^{-1}\right)\notag\\&
\quad+\mathrm{Tr}\left(\boldsymbol{\mSigma}\right)\left(\exp\left(-\frac{2\alpha v_x}{\psi(T)}\right)\left(1+O\left(\psi(T)^{-2}T^{-1}\right)\right)\right)O\left(\psi(T)^{-2}T^{-1}\right)\label{eq:dzdtr}
\end{align}
And \cref{eq:st} gives
\begin{align}
&\frac{\dd}{\dd \psi\left(T\right)}\left[a^{T}\|\vw^{\star}\|^{2}+f\frac{-a^{T}}{1-a}\right]\notag\\&=\tfrac{v_{z}}{2}\left(\frac{1}{\psi(T)^{2}T}+O\left(\psi(T)^{-3}T^{-2}\right)\right)\left(\exp\left(-\tfrac{2\alpha v_{x}}{\psi(T)}\right)\left(1+O\left(\psi(T)^{-2}T^{-1}\right)\right)\right)\notag\\&\hspace{4em}-\tfrac{v_{z}}{2}\left(\frac{1}{\psi(T)T}+O\left((\psi(T)T)^{-2}\right)\right)\Bigg(\tfrac{2\alpha v_{x}}{\psi(T)^{2}}\exp\left(-\tfrac{2\alpha v_{x}}{\psi(T)}\right)\notag\\&\hspace{8em}\left(1+O\left(\psi(T)^{-2}T^{-1}\right)\right)+\exp\left(-\tfrac{2\alpha v_{x}}{\psi(T)}\right)O\left(\psi(T)^{-3}T^{-1}\right)\Bigg)\notag\\&
\quad+\frac{2\alpha v_{x}}{\psi(T)^{2}}\exp\left(-\frac{2\alpha v_{x}}{\psi(T)}\right)\left(1+O\left(\psi(T)^{-2}T^{-1}\right)\right)\|\vw^{\star}\|^{2}\notag\\&
\quad+\exp\left(-\frac{2\alpha v_{x}}{\psi(T)}\right)O\left(\psi(T)^{-3}T^{-1}\right)\|\vw^{\star}\|^{2}.\label{eq:dzds}
\end{align}

Retaining only the leading–order terms, we have
\begin{align*}
&\frac{\dd}{\dd \psi(T)}[\mathbb{E}\left[G_{T}\right]-\lim\limits_{T\to\infty}\mathbb{E}\left[G_{T}\right]]\\&=\frac{2\alpha v_{x}}{\psi(T)^{2}}\exp\left(-q\right)\|\vw^{\star}\|^{2}
+\frac{4\,\mathrm{Tr}(\boldsymbol{\mSigma})}{T}\frac{\alpha v_x}{\psi(T)^{2}}
\exp\left(-\frac{q}{2}\right)\,,
\end{align*}
where $q=\frac{2\alpha v_{x}}{\psi(T)}$.

Let 
\begin{align}   
g\left(\psi(T)\right)
&
=
\frac{v_{z}v_{x}}{2T\psi(T)^{2}}\left[\underbrace{\exp\left(-q\right)\frac{4\alpha T\left\Vert \vw^{\star}\right\Vert ^{2}}{v_{z}}+\frac{8\alpha\mathrm{Tr}(\boldsymbol{\mSigma})}{v_z}\sqrt{\exp\left(-q\right)}}_{H\left(q\right)}\right],\label{eq:g}
\end{align}
Recall \cref{lem:laurEgt} where we have shown that
$$\frac{\dd}{\dd \lambda}\lim\limits_{T\to\infty}\mathbb E\left[G_T\right]=-\frac{1}{2v_x}\left(\frac{v_z}{v_x}+\mathrm{Tr}(\boldsymbol \Sigma)\left(\ratio+1\right)\right) \tilde\lambda^{-2}+O\left(\tilde\lambda^{-3}\right)\,.$$
We would like to show that as a consequence $g\left(\psi(T)\right)-\frac{\dd}{\dd \psi(T)}\lim\limits_{T\to\infty}\mathbb E\left[G_T\right]$ has a single root. The term $\frac{\dd}{\dd \psi(T)}\lim\limits_{T\to\infty}\mathbb E[G_T]$ has a contribution of $g=-\left(\frac{1}{v_{x}}+\mathrm{Tr}(\boldsymbol{\Sigma})\frac{\ratio+1}{v_{z}}\right)$ (retaining the leading term) to $H(q)$, therefore we define $\tilde H(q)=H(q)+g$.

The derivative of $\tilde H$ is given by $\frac{\dd}{\dd q}\tilde H =-\exp(-q)\frac{4\alpha T\|\vw^\star\|^{2}}{v_{z}}-\frac{4\alpha\mathrm{Tr}(\boldsymbol{\mSigma})}{v_z}\,\exp(-\frac{q}{2}),$ which is strictly negative. Consequently, $\tilde H(q)$ is strictly decreasing in this regime. Moreover, if $q \in o\left(\ln \left(T\right)\right)$ then $\tilde H(q)=\exp(-q)\frac{4\ratio T\left\Vert \teacher \right\Vert^2}{\labelnoise}+g+o(1)$ which is positive for all sufficiently large $T$.
Furthermore, if $q\in \omega(\ln\left(T\right))$ then $\tilde H(q)=g+o(1)$ hence for all sufficiently large $T$, we have that $\tilde H(q)$ is negative.
Therefore, it suffices to identify a value $q^\star$ such that $\tilde H\left(q^\star\right)\approx 0$. To this end, we solve $H(q^\star)\approx - g$ by retaining the leading-order term for $q^\star \in \Theta\left(\ln\left(T\right)\right)$ we evaluate $\frac{4\alpha T\left\Vert \vw^{\star}\right\Vert^{2}\exp\left(-q^\star\right)}{v_{z}} \approx -g$ which leaves us with $q^\star\approx\ln\frac{4\alpha T\left\Vert \vw^{\star}\right\Vert ^{2}}{-v_{z}g}$. We proceed to prove the Theorem using the approximation discussed above.

Let $\epsilon>0$,
\begin{align}   
\tilde H\left(\frac{1}{1+\epsilon}q^\star\right)
=-g\cdot\left(\frac{4\alpha T\|\vw^\star\|^2}{-v_zg}\right)^{\frac{\epsilon}{1+\epsilon}}
+g + o(1)
\xrightarrow[T\to\infty]{} +\infty,\label{eqq:1}
\end{align}
and hence \(\tilde H\left(\frac{1}{1+\epsilon}q^\star\right)>0\) for all sufficiently large \(T\).
Similarly,
\begin{align}
\tilde H\left(\frac{1}{1-\epsilon}q^\star\right)=
-g\cdot\left(\frac{4\alpha T\|\vw^\star\|^2}{-v_zg}\right)^{-\frac{\epsilon}{1-\epsilon}}
+g + o(1)
\xrightarrow[T\to\infty]{} g < 0,\label{eqq:2}
\end{align}
so \(\tilde H\left(\frac{1}{1-\epsilon}q^\star\right)<0\) for all sufficiently large \(T\).

By continuity (and the strict monotonicity of \(\tilde H\)), there exists a unique solution satisfying
\(\tilde H(q^\star)=0\). 
Converting back via \(q = \frac{2\alpha v_x}{\psi(T)}\) yields
\[
(1-\epsilon)\frac{2\alpha v_x T}{\ln\frac{4\alpha T\left\Vert \vw^{\star}\right\Vert ^{2}}{\frac{v_{z}}{v_{x}}+\mathrm{Tr}(\boldsymbol{\Sigma})\left(\ratio+1\right)}}
<
\lambda^\star
<
(1+\epsilon)\frac{2\alpha v_x T}{\ln\frac{4\alpha T\left\Vert \vw^{\star}\right\Vert ^{2}}{\frac{v_{z}}{v_{x}}+\mathrm{Tr}(\boldsymbol{\Sigma})\left(\ratio+1\right)}}.
\]
Recall \cref{eq:dtdz,eq:dzds,eq:dzdtr}, where the last term stems from the non-leading factor we discarded in the term $\frac{\dd}{\dd \psi(T)}\lim\limits_{T\to\infty}\mathbb E\left[G_T\right]$. Since $\psi(T)=\Theta\left(\frac{1}{\ln T}\right)$ the non-leading terms contribute $o(1)$ perturbation to $H(q)$, which is absorbed into the $o(1)$ in \cref{eqq:1,eqq:2}. We now investigate the asymptotic behavior under which the theorem holds for all $T>T_0$.

To satisfy $\tilde H\left(\frac{1}{1+\epsilon}q^\star\right)>0$, we begin with \cref{eqq:1}. In this case, the non-leading terms contribute $O\left(T^{-\frac{1}{2(1+\epsilon)}}\right)$ (as $\exp(-q/2)$, the dominating term, in \cref{eq:g} becomes $O\left(T^{-\frac{1}{2(1+\epsilon)}}\right)$).
\begin{align*}
&-g\cdot\left(\left(\frac{4\alpha T\left\Vert \vw^{\star}\right\Vert ^{2}}{-v_{z}g}\right)^{\frac{\epsilon}{1+\epsilon}}-1\right)+O\left(T^{-\frac{1}{2(1+\epsilon)}}\right)\\&=-g\cdot\left(\exp\left(\frac{\epsilon}{1+\epsilon}\ln\left(\frac{4\alpha T\left\Vert \vw^{\star}\right\Vert ^{2}}{-v_{z}g}\right)\right)-1\right)+O\left(T^{-\frac{1}{2(1+\epsilon)}}\right)\\
\explain{e^{x}-1\geq x}
&
\geq
-g\cdot\frac{\epsilon}{1+\epsilon}\ln\left(\frac{4\alpha T\left\Vert \vw^{\star}\right\Vert ^{2}}{-v_{z}g}\right)+O\left(T^{-\frac{1}{2(1+\epsilon)}}\right)\\&\gtrsim-g\cdot\frac{\epsilon}{1+\epsilon}\ln\left(T\right)+O\left(T^{-\frac{1}{2(1+\epsilon)}}\right).
\end{align*}
There exists a constant $K_1$ such that for all sufficiently large $T$ we have
$
O\left(T^{-\frac{1}{2(1+\epsilon)}}\right)\leq K_1\,T^{-\frac{1}{2(1+\epsilon)}},
$
and hence we require,
$-g\frac{\epsilon}{1+\epsilon}\ln T\gtrsim K_1\,T^{-\frac{1}{2(1+\epsilon)}}
$, which is satisfied if $T>T_0 = O\left(\frac{1}{\epsilon^2}\right)$ (as $\lim\limits_{\epsilon\to 0^+} \frac{\epsilon\ln\frac{1}{\epsilon}}{\epsilon^{\frac{1}{1+\epsilon}}}=\infty$).
Note that $\frac{4\alpha T\left\Vert \vw^{\star}\right\Vert ^{2}}{-v_{z}g}>0$.
We proceed with \cref{eqq:2}, in this case, the non-leading terms contribute $O\left(T^{-\frac{1}{2(1-\epsilon)}}\right)$ (as $\exp(-q/2)$, the dominating term, in \cref{eq:g} becomes $O\left(T^{-\frac{1}{2(1-\epsilon)}}\right)$).
Splitting it to cases.

\paragraph{Case (1)} When $\frac{\epsilon}{1-\epsilon}\ln\left(\frac{4\alpha T\left\Vert \vw^{\star}\right\Vert ^{2}}{-v_{z}g}\right)\in (0,1)$,
\begin{align*}
&-g\cdot\left(\exp\left(-\frac{\epsilon}{1-\epsilon}\ln\left(\frac{4\alpha T\left\Vert \vw^{\star}\right\Vert ^{2}}{-v_{z}g}\right)\right)-1\right)+O\left(T^{-\frac{1}{2(1-\epsilon)}}\right)\\
\explain{e^{-x}-1\leq-\frac{1}{2}x,~\forall x\in\left(0,1\right)}
&
\leq g\cdot\frac{\epsilon}{2\left(1-\epsilon\right)}\ln\left(\frac{4\alpha T\left\Vert \vw^{\star}\right\Vert ^{2}}{-v_{z}g}\right)+O\left(T^{-\frac{1}{2(1-\epsilon)}}\right)\\&\lesssim g\cdot\frac{\epsilon}{1-\epsilon}\ln\left(T\right)+O\left(T^{-\frac{1}{2(1-\epsilon)}}\right).
\end{align*}

As before, there exists a constant $K_2$ such that for all sufficiently large $T$ we have
$
O\left(T^{-\frac{1}{2(1-\epsilon)}}\right)\leq K_2\,T^{-\frac{1}{2(1-\epsilon)}},
$
and hence we require
$
-g\frac{\epsilon}{1-\epsilon}\ln T\gtrsim K_2\,T^{-\frac{1}{2(1-\epsilon)}}.
$
Although this equation does not admit a closed-form solution for $T$, choosing $T_0$ such that $T>T_0 = O\left(\frac{1}{\epsilon^2}\right)$ (as $\lim\limits_{\epsilon\to 0^+} \frac{\epsilon\ln\frac{1}{\epsilon}}{\epsilon^{\frac{1}{1-\epsilon}}}=\infty$), is sufficient.

\paragraph{Case (2)} 
When
$\frac{\epsilon}{1-\epsilon}\ln\left(\frac{4\alpha T\left\Vert \vw^{\star}\right\Vert ^{2}}{-v_{z}g}\right)\in [1,\infty)$, 
\begin{align*}
&-g\left(\exp\left(-\frac{\epsilon}{1-\epsilon}\ln\left(\frac{4\alpha T\left\Vert \vw^{\star}\right\Vert ^{2}}{-v_{z}g}\right)\right)-1\right)+O\left(T^{-\frac{1}{2(1-\epsilon)}}\right)\\&\leq -g\left(e^{-1}-1\right)+O\left(T^{-\frac{1}{2(1-\epsilon)}}\right)\,.  
\end{align*}
Therefore, we require, 
\begin{align*}
    -g\left(e^{-1}-1\right)\gtrsim K_{2}T^{-\frac{1}{2(1-\epsilon)}}
    \,,
\end{align*}
which holds for all sufficiently large $T$ independent of $\epsilon$.
\end{proof}

\end{document}